\documentclass[final]{colt2013}

\coltauthor{\Name{Matus Telgarsky} \Email{mtelgars@cs.ucsd.edu}\\
 \addr
Computer Science and Engineering,
UC San Diego\\
9500 Gilman Drive MC 0404, La Jolla, CA 92093
 }

\title{Boosting with the Logistic Loss is Consistent}

\usepackage{cleveref}

\usepackage{dsfont}

\def\1{\mathds 1}
\def\R{\mathbb R}
\def\Z{\mathbb Z}
\def\bbE{\mathbb E}

\def\bbL{\mathbb L}

\def\bfe{\mathbf e}

\def\cC{\mathcal C}
\def\cD{\mathcal D}
\def\cF{\mathcal F}

\def\cH{\mathcal H}

\def\cL{\mathcal L}

\def\cO{\mathcal O}

\def\cR{\mathcal R}
\def\cS{\mathcal S}

\def\cV{\mathcal V}
\def\cX{\mathcal X}

\def\cZ{\mathcal Z}

\def\SPAN{\textup{span}}

\def\nf{\nabla f}

\def\supp{\textup{supp}}

\def\dom{\textup{dom}}

\newcommand{\ip}[2]{\left\langle #1, #2 \right \rangle}
\newcommand{\argmin}{\operatornamewithlimits{arg\,min}}

\newcommand\Fvc{\cF_{\textup{vc}}}

\newcommand\Fds{\cF_{\textup{ds}}}
\newcommand\Fb{\cF_{\textup{b}}}

\newcommand\Lbds{\bbL_{\textup{lg}}}
\newcommand\Ltd{\bbL_{\textup{2d}}}
\newcommand\llog{\ell_{\textup{log}}}
\newcommand\lruss{\ell_{\textup{russ}}}

\newcommand\ttil{{\tilde \tau}}

\newcommand\hcL{\widehat \cL}
\newcommand\nhcL{\nabla \widehat\cL}
\newcommand\hmu{{\widehat \mu}}
\newcommand\hcR{\widehat \cR}

\newcommand{\proofref}[1]{$^{\hyperref[#1]{[\textup{pg.}\pageref{#1}]}}$}

\newenvironment{proofof}[1]{\begin{proof}\textbf{(of {#1})}}{\end{proof}}

\begin{document}

\maketitle

\begin{abstract}
    This manuscript provides optimization guarantees,
    generalization bounds,
    and statistical consistency results
    for AdaBoost variants which replace the exponential
    loss with the logistic and similar losses
    (specifically, twice differentiable
    convex losses which are Lipschitz and
    tend to zero on one side).

    The heart of the analysis is to show that, in lieu of explicit regularization and
    constraints, the structure of the problem is fairly rigidly controlled by the
    source distribution itself. 
    The first control of this type 
    is in the separable case, where a
    distribution-dependent relaxed weak learning
    rate induces speedy convergence with high probability over any sample.
    Otherwise, in the nonseparable case,
    the convex surrogate risk itself exhibits distribution-dependent levels of curvature,
    and consequently the algorithm's output has small norm with high
    probability.
\end{abstract}






\section{Introduction}
Boosting algorithms form accurate predictors by combining many simple ones.
These methods are practically effective \citep{caruana_empirical},
theoretically alluring \citep{schapire_wl},
and continue to be the topic of extensive research
\citep{schapire_freund_book_final}.

The most popular scheme, AdaBoost \citep{freund_schapire_adaboost},
was eventually revealed to be a coordinate descent method applied to
a convex empirical risk minimization problem \citep{breiman}.
Due to the lack of regularization and constraints, this optimization problem eschews the typical
structure which leads to a fast-converging, well-conditioned optimization problem: it
typically fails to have
minimizers (let alone possessing compact level sets or strong convexity),
and the simple predictors (weak learners)
can be linearly dependent, meaning the Hessian is singular.
Consequently, fairly customized convergence analyses must be
developed~\citep{freund_schapire_adaboost,collins_schapire_singer_adaboost_bregman,
mukherjee_rudin_schapire_adaboost_convergence_rate,
primal_dual_boosting_arxiv}, with data-dependent quantities dictating behavior.
This, however, can be a great boon: one such quantity, the
\emph{weak learning rate}---a measure of the compatibility of the weak learners to
the target function---allows for linear convergence in settings far removed from the
strong convexity typical of fast convergence in convex optimization.

Difficulties also arise on the statistical side: each round typically selects a
new hypothesis
from some VC class, and the method is consequently building hypotheses
in the linear span, which generally has infinite VC dimension,
and is thus statistically unstable \citep[Theorem 14.3]{DGL}.
It was therefore the topic of great research to establish consistency of AdaBoost
\citep{zhang_yu_boosting,jiang_boosting_consistency}, a question finally closed by
\citet{bartlett_traskin_adaboost}.

AdaBoost originally used the exponential loss, however much practical and theoretical research
has been devoted to the logistic loss
\citep{friedman_hastie_tibshirani_statboost,lafferty_logistic,
collins_schapire_singer_adaboost_bregman}, both due to intuitive appeal (e.g., less attention
to outliers), and statistical connections (e.g., consistency of maximum likelihood
\citep[Section 17]{ferguson_large_sample_theory}).
Even so, this choice has not been subjected
to the same intensive consistency study as the exponential loss, and as discussed by
\citet[Section 4]{bartlett_traskin_adaboost}, current analyses for the exponential loss
do not carry over.

\subsection{Outline}

The primary goal of this manuscript is to close the gap with the exponential loss;
namely, boosting with losses similar to the logistic loss is consistent under the
same assumptions as those is assumed for the exponential loss
\citep[Corollary 9]{bartlett_traskin_adaboost},
moreover with comparable rates
\citep[Theorem 12.2]{schapire_freund_book_final}.



The algorithm and related notation are detailed in \Cref{sec:notation}.
To fit practical regimes, both the selection of simple predictors
(also termed \emph{weak learners} and coordinates)
and step size may be approximate;
crucially, however, the analysis covers the case of unconstrained step sizes.
The usual early stopping threshold is employed: $m^a$ iterations are performed,
where $m$ is the sample size and $a\in(0,1)$ is a scalar parameter to the algorithm.
Lastly, rather than simply outputting the final predictor, the method returns the iterate
which achieved the smallest classification error.
While perhaps unnecessary, this choice leads to a pleasantly simple
convergence analysis in the separable case.


The general consistency result is presented in
\Cref{sec:general}, along with a sketch of the analysis.  As usual, the Borel-Cantelli Lemma
is used to convert finite sample guarantees into a consistency result; the finite
sample guarantees themselves are split into two cases:
a separable case in \Cref{sec:separable},
and a nonseparable case in \Cref{sec:nonseparable}.
In either case, 
when using the logistic loss,
the classification risk will decay roughly as
$m^{-c}$ for some $c<1$.


Proofs are only outlined in the body, with details deferred to the
appendices.

\subsection{Related Work}
On the general topic of AdaBoost, both the original papers
\citep{schapire_wl,yoav_boost_by_majority,freund_schapire_adaboost}
as well as the textbook by the original authors \citep{schapire_freund_book_final}
are indispensable.

Additive logistic regression was introduced by
\citet{friedman_hastie_tibshirani_statboost},
with extensive additional discussion appearing shortly
thereafter
\citep{friedman_gradient_boosting,lafferty_logistic,
mason_baxter_bartlett_frean_functionalgrad}.
The particular method studied in this manuscript, which is essentially
AdaBoost but with the exponential loss replaced by losses similar to the
logistic loss,
was shown to produce a sequence of empirical risks converging
to the infimum
by \citet{collins_schapire_singer_adaboost_bregman},
with (optimization) rates in the general case coming later \citep{primal_dual_boosting_arxiv},
and (optimization) rates in the margin case coming earlier \citep{duffy_helmbold_boost}.

The consistency of AdaBoost was first analyzed under various regularization strategies.
Most notably, the work of
\citet{blv_regularized_boosting} and \citet{bv_regularized_boosting_again}
studied the solutions of penalized estimators; the former work in particular achieving
excellent finite sample guarantees, with convex risk decaying roughly as $\cO(m^{-1/2})$
(where $m$ is the sample size), with improvements under various noise conditions.
This work, however, did not demonstrate tractable algorithms to produce these estimators,
which was a goal of the work by \citet{zhang_yu_boosting}; namely, there it is shown
that merely constraining the step size taken by an AdaBoost-style scheme (with
a variety of losses) suffices to achieve a convex risk
rate of roughly $\cO(m^{-1/4})$ (in the case of the logistic loss),
which includes the effect of approximate solutions
produced by the algorithm.  As will be discussed later, the present work, in the
nonseparable case, fits well with the development by \citet{zhang_yu_boosting}.

Two works give a consistency analysis of AdaBoost without any algorithmic modifications,
under the condition that the algorithm is stopped after $m^a$ iterations
(with arbitrary $a\in(0,1)$).
The first such analysis,
due to \citet{bartlett_traskin_adaboost}, was focused on establishing consistency,
and established the convex risk decays roughly as $\cO(1 /\sqrt{\ln(m)})$;
the analysis depends on a curvature lower bound,
which follows from a lower bound on the convex risk since the exponential
loss is equal to its derivative.
This derivative structure is of course not present with the logistic loss, and
the present analysis must find another way.
A streamlined consistency analysis of AdaBoost appears in the textbook of
\citet[Theorem 12.2]{schapire_freund_book_final}, with a rate of roughly
$\cO(m^{-1/9})$ (by choosing $a:=5/9$); the analysis is short and clean, but it is not
clear how to decouple the exponential loss.

In the separable case, the analysis here relies upon ideas from weak learnability,
just as with the original analysis of AdaBoost (under margin assumptions)
\citep{freund_schapire_adaboost}.  The relaxed notion of margin here is very close
to the quantity $\textup{AvgMin}_k$ as developed by
\citet[Section 4.1]{shai_singer_weaklearn_linsep};
the main contrasting point is that the present manuscript is concerned with statistical
properties, and in particular how these relaxed margin properties behave under
sampling.
The optimization analysis in the separable case here shares ideas both with
the original AdaBoost analysis
\citep{freund_schapire_adaboost},
but also with the literature on hard cores
\citep{russell_hardcore,satyen_hardcore}; one distinction is that the latter methods
take the target weak learning rate as input, whereas here (and in general with
\emph{ada}ptive \emph{boost}ing), it must be found by the algorithm.  Interestingly,
the loss function implicit in the boosting algorithm due to
\citet[Proof of Lemma 1]{russell_hardcore} achieves
superior constants to the logistic loss 
in \Cref{fact:sep:basic}; nearly the
same loss was presented and praised by
\citet[see the definition at the end of Section 4.6]{zhang_convex_consistency}.


As stated previously, the nonseparable case fits well with the scheme laid
down by \citet{zhang_yu_boosting}, where the algorithm is modified to constrain step
sizes.  Indeed, the analysis here first establishes that the iterates are well-behaved
with exactly the sorts of norm bounds needed by the analysis
of \citet{zhang_yu_boosting} (compare for instance the summability conditions
\citep[Equation (4)]{zhang_yu_boosting} with \Cref{fact:nonsep:controls}).
In order to produce these results, the present work uses a dual optimum as
a witness to the difficulty of the convex risk problem over the source distribution; this technique follows structural
properties of boosting laid in the finite-dimensional case by \citet{primal_dual_boosting_arxiv}.
That convergence analysis appears statistically unstable, and the subsequent analysis here follows
a similar path to the one by \citet{zhang_yu_boosting}, with additional
help from \citet{bartlett_traskin_adaboost}.  One interesting distinction between
the present work and those by \citet{bartlett_traskin_adaboost} and
\citet{zhang_yu_boosting} is that the latter two require a more strenuous
algorithm: the weak learner and step size selection must be performed simultaneously.
Decoupling these does not appear to impact the rates, however, this distinction
prevents those results from being directly invoked here, meaning they must instead be
reworked.

Lastly, note that the translation between convex and classification risks
follows standard results on classification calibration as first
developed by
\citet{zhang_convex_consistency},
and later extended by \citet{bartlett_jordan_mcauliffe}.



\section{Notation and Algorithm}
\label{sec:notation}

Let $\cH$ be the collection of weak learners, where each $h\in \cH$
is a function of the form $h : \cX \to [-1,+1]$,
with $\cX$ being an abstract instance space,
and the crucial property of the output space $[-1,+1]$ is that it is bounded.
Given any weighting $\lambda$ of $\cH$ satisfying $\sum_{h\in \cH} |\lambda(h)| < \infty$,
define the function
\[
    (H\lambda)_x := (H\lambda)(x) := \sum_{h \in \cH} h(x) \lambda(h).
\]
Since $\sum_h \lambda(h)$ is absolutely convergent and $\sup_{x,h} |h(x)| \leq 1$,
then
$H\lambda$ is well-defined.

Let $\Lambda$ denote the space of all absolutely
convergent weightings over $\cH$; formally, $\Lambda$ is the Banach space $L^1(\rho)$, where $\rho$ is
the counting measure over $\cH$.  In this way, $H$ can be viewed as a function from $\Lambda$
to the vector space of bounded functions over $\cX$.  The algorithm itself only considers
finite sets of hypotheses over a finite sample,
and thus $H$ can be viewed as a matrix, but the Banach space generalization will
be useful when considering the abstract problem over the distribution.

For additional
convenience, define a second function
\[
    (A\lambda)_{x,y} := (A\lambda)(x,y) := -y(H\lambda)_x := -y \sum_{h\in \cH} h(x) \lambda(h),
\]
which is again well-defined.
Let $\bfe_h\in \Lambda$ denote the weighting
placing unit weight on a fixed $h\in\cH$,
and zero weight elsewhere.
  For more properties of these Banach spaces, as well
as the linear operators $H$ and $A$,
please see \Cref{sec:linops}.

The basic measure of the complexity of $\cH$ is its VC dimension.
\begin{definition}
    Let $\Fvc$ contain all classes $\cH$ of finite VC dimension,
    denoted $\cV(\cH) < \infty$.
\end{definition}


The source distribution over $\cX\times \{-1,+1\}$ will always be denoted by $\mu$,
with a factorization (disintegration) into a marginal $\mu^\cX$ over $\cX$
and conditional $\Pr(Y=1|X = x)$, the latter considered as a function over $\cX$.
When a sample $\{(x_i,y_i)\}_{i=1}^m$ is available, $\hmu$ will denote
the corresponding empirical measure.
Many results hold for arbitrary probability measures
over $\cX\times \{-1,+1\}$, in which case the variable $\nu$ will be adopted; the
$\sigma$-algebra over $\cX$ is always the Borel $\sigma$-algebra (and it is
tacitly supposed $\cX$ is a topological space).
With the measures defined,
a second notion of class complexity
is as follows.
\begin{definition}
    Let $\Fds(\nu)$ contain every class $\cH$ whose linear span $\SPAN(\cH)$
    is dense (in the $L^1(\nu)$ topology) in the collection of all
    bounded measurable functions over $\cX$.
\end{definition}
Conditions similar to those defining $\Fds(\nu)$ are usually called
\emph{dense class assumptions}
\citep[Condition 1, Denseness]{bartlett_traskin_adaboost},
or \emph{completeness assumptions} \citep[Definition 1]{breiman_infinity_theory};
for a more extensive discussion of these conditions, please see \Cref{sec:Fds}; for the
time being, the important point is that reasonable elements of $\Fvc \cap \Fds(\nu)$
exist; in particular, the following result provides that if $\cX = \R^d$, then
decision lists and decision trees with $d$ axis-aligned splits suffice.
\begin{proposition}
    \label[proposition]{fact:Fds:cubicles}
    Suppose $\cX=\R^d$,
    and let $\nu$ be a Borel probability measure over $\cX$.
    If $\SPAN(\cH)$ contains all indicators of
    products of half-open intervals of the form $\times_{i=1}^d [a_i,b_i)$,
    where $a_i < b_i$, then $\cH\in \Fds(\nu)$.
\end{proposition}

Given a loss function $\ell:\R \to \R_+$ (where $\R_+$ denotes nonnegative reals,
and later $\R_{++}$ will denote positive reals), a first version of the relevant
optimization problem over the source distribution is
\[
    \bar\cL
    := \inf \left\{ \int \ell\left(-y(H\lambda)_x\right) d\mu(x,y) :
    \lambda \in \Lambda\right\}
    = \inf \left\{ \int \ell(A\lambda) d\mu : \lambda \in \Lambda\right\},
\]
where $\bar\cL$ denotes the optimal value, and the final expression both exhibits the practice
of dropping integration variables, and the convenience of $A$.
For further simplification, define the simplified surrogate risk functions
\[
    \cL(A\lambda) := \int \ell(A\lambda)d\mu
    \qquad
    \textup{and}
    \qquad
    \hcL(A\lambda) := \int \ell(A\lambda)d\hmu
    =
    \frac 1 m \sum_{i=1}^m \ell((A\lambda)_{x_i,y_i}),
\]
meaning $\hcL$ denotes the usual empirical risk.
The classes of loss functions considered here are as follows.

\begin{definition}
    Let $\Ltd$ denote twice continuously differentiable convex losses.
    Additionally, let $\Lbds$ contain all differentiable convex
    Lipschitz losses $\ell:\R\to\R_+$
    with tightest Lipschitz constant $\beta_2 :=
    \sup_{x\neq y} |\ell(x) -\ell(y)|/ |x-y|$
    as follows.
    First, every $\ell$ has
    $\ell' \in [0,\beta_2]$ everywhere
    and
    $\ell'\in [\beta_1,\beta_2]$ over $\R_+$ for some $0 < \beta_1 \leq \beta_2$.
    Second, every $\ell$ has Lipschitz gradients with (tightest) parameter $B_2$, meaning
    $B_2 := \sup_{x,y\in \R} (\ell'(y) - \ell'(x))/(y-x) < \infty$.
%
%
\end{definition}
Although the most general guarantees require $\ell \in \Lbds\cap \Ltd$,
the separable case needs only $\ell\in \Lbds$,
which allows consideration of
an interesting piecewise quadratic loss
$\lruss(x) := 0.5(x+1)^2\1[-1 < x < 0] + (x+0.5) \1[0 \leq x]$,
which was used by
\citet[Proof of Lemma 1]{russell_hardcore} in the study of hard cores.
Both $\lruss$ and the logistic loss $\llog(x) := \ln(1+\exp(x))$ are
in $\Lbds$, whereas $\llog$ and the exponential loss are within $\Ltd$ (cf. \Cref{fact:L:examples}).


Let $\cR$ denote the classification risk, meaning
\[
    \cR(H\lambda) := \int \left(\1[y = +1 \land (H\lambda)_x < 0]
     + \1[y = -1 \land (H\lambda)_x \geq 0]\right)d\mu(x,y).
\]
Analogously to $\cL$, let $\hcR$ denote empirical classification risk,
and $\bar \cR$ denote optimal classification risk over $\SPAN(\cH)$.
Notice that these definitions embed the fact that boosting procedures provide
a real-valued function $H\lambda$, which is then thresholded to produce a binary
classifier.

Borrowing from the probability literature, brackets denote a shorthand for sets; for
instance $[y(H\lambda)_x > 0] = \{(x,y) \in \cX\times \{-1,+1\} : y(H\lambda)_x > 0\}$
is the subset of $\cX\times\{-1,+1\}$
where $\lambda$ achieves strictly positive margins.  When the variables are clear, they
will be suppressed; e.g.,
$[p > 0] = [p(x,y) > 0] = \{(x,y)\in\cX\times \{-1,+1\} : p(x,y) > 0\}$.

\subsection{Algorithm}

The algorithm itself is spelled out in \Cref{alg:alg:alg}.  As the method is coordinate
descent applied to $\hcL\circ A$, the relevant gradient term is
$A^\top \nhcL(A\lambda) = \nabla (\hcL\circ A)(\lambda)$ (which can be computed from
the sample; see \Cref{subsec:linesearch} for details).
The scalar $\rho \in(0,1)$ allows for approximate weak learner selection,
and furthermore the step size also has some flexibility, though as stated previously,
the unconstrained case is the tricky one.  Lastly, note
that the iterate achieving the best classification error is
returned.

\begin{algorithm}[t!]
\floatconts
  {alg:alg:alg}%
   {\caption{$\textsc{boost}.$\label{alg:alg:alg}}}
   {
       \KwIn{loss $\ell$ and empirical measure $\hmu$ (granting $m$, $\hcL$, $\nhcL$),
           hypothesis class $\cH$ (granting $A$, $A^\top$),
        stopping and coordinate search parameters $a\in(0,1)$ and $\rho\in (0,1)$.
        }
    \KwOut{Coefficient vector $\hat\lambda$.}
    \BlankLine
    Initialize $\lambda_0 := 0$.\;

    \For{$t = 1,2,\ldots,\lceil m^{\alpha}\rceil:$}{
        Choose approximate best coordinate (weak learner) $h_t\in \cH$ satisfying \;
        \[
        |\nhcL(A\lambda_{t-1})^\top A \bfe_{h_t}| \geq
        \rho \sup_{h\in \cH}
        |\nhcL(A\lambda_{t-1})^\top A \bfe_{h}|
        =
        \rho\|\nhcL(A\lambda_{t-1})^\top A\|_\infty.
        \]
        Set descent direction $v_t \in \{\pm \bfe_{h_t}\}$, whereby\;
        \[
        -\rho\|A^\top \nhcL(A\lambda_{t-1})\|_\infty
        \geq
        \nhcL(A\lambda_{t-1})^\top A v_t
        \geq -\|A^\top \nhcL(A\lambda_{t-1})\|_\infty.
        \]

        \SetKwBlock{WAT}{}{}
        Set $\bar\alpha_t := \argmin_{\alpha>0}
        \hcL(A(\lambda_{t-1} + \alpha v_t))$,
        and choose a step $\alpha_t$ as follows:\WAT{
        \textbf{option 1:} If $\bar\alpha_t<\infty$, set $\alpha_t=\bar\alpha_t$
        (i.e., make an optimal unconstrained step).

        \textbf{option 2:} If $\ell \in \Lbds$, choose any
            $\alpha_t \in \left[
        -\nhcL(A\lambda_{t-1})^\top A v_t/B_2, \bar \alpha_t\right)$.

        \textbf{option 3:} Choose any $\alpha_t$ satisfying the Wolfe conditions
        (please see \Cref{rem:wolfe}.)
        }
        Update $\lambda_t := \lambda_{t-1} + \alpha_t v_t$.\;
    }
    \KwRet $\hat\lambda$ achieving the best classification error
(i.e., $\hcR(H\hat\lambda) = \min_{i \in [\lceil m^a \rceil]} \hcR(H\lambda_i)$).
    }
\end{algorithm}

\section{Consistency Statement and Analysis Sketch}
\label{sec:general}

The analysis considers two cases: either $\bar\cL = 0$ (\emph{separable})
or $\bar\cL > 0$ (\emph{nonseparable}).
(By \Cref{fact:why_say_separable}, $\bar \cL=0$ implies finite samples
have a separating choice $\lambda\in\Lambda$ almost surely.)

When the instance is separable, the improvement in objective value $\cL(A\lambda)$
in early iterations may be lower bounded by a margin-based quantity related to
the classical weak learning rate; while this quantity is a random variable,
with high probability it can be lower bounded by the analogous quantity over the distribution
(which will be shown positive iff the instance is separable).
The bulk of the analysis is in constructing and controlling this quantity; the optimization
and generalization analysis thereafter is straightforward, yielding a rate
of roughly $\cO(1/m^{1/3})$ when $a=2/3$.

When the instance is not separable, every weak learner makes a fair number of mistakes,
and thus the algorithm makes more hesitant progress.
Concretely, with high probability,
the norms of the iterates are bounded, and moreover the quantity
$m^{-1}\sum_i \ell''((A\lambda)_{x_i,y_i})$, which is roughly the Hessian
in axis-aligned directions (and relevant to coordinate descent), is also lower bounded.
This in turn allows adaptation of the optimization analysis due to
\citet{zhang_yu_boosting}.
While the rate in this case is still roughly $\cO(1/m^{1/c})$, unfortunately the exponent
$c>1$ depends both on $\mu$ and on $\cH$ (but is of course finite).

As a final point of interest,
each case, in order to respectively
establish either fast decrease or the norm constraints,
considers the behavior of the reweighted
average margins
\begin{equation}
    \int (A\lambda) p d\mu = \int (A\lambda)_{x,y} p(x,y) d\mu(x,y),
    \label{eq:pdm}
\end{equation}
where $\lambda\in \Lambda$ and $p\in L^\infty(\nu)$.
In the separable case, this quantity is studied for a single good $\lambda$
as $p$ varies, whereas the nonseparable
case studies a single bad $p$ as $\lambda$ varies.


Combining these finite sample results with the Borel-Cantelli Lemma
gives the following.

\begin{theorem}
    \label{fact:consistency}
    Let loss $\ell\in\Lbds\cap \Ltd$,
    probability measure $\mu$ over $\cX\times \{-1,+1\}$,
    binary class $\cH\in\Fvc\cap\Fds(\mu^{\cX})$,
    and any stopping parameter $a\in(0,1)$ be given.
    Let $\hat\lambda_m$ denote the output of \Cref{alg:alg:alg}
    when run on $m$ examples,
    and let $\cR_\star$ to denote the Bayes error rate.
    Then
    $\cR(A\hat\lambda_m) \to \cR_\star$ almost surely as $m\to\infty$.
\end{theorem}

\section{The Separable Case ($\bar\cL = 0$)}
\label{sec:separable}

The rates in the case $\bar \cL = 0$ will depend on the
following quantity $\gamma_\epsilon(\nu)$, which directly embeds the reweighted
margin expression in \cref{eq:pdm}.


\begin{definition}
    Let $\nu$ be any probability measure over $\cX\times \{-1,+1\}$ (relevant choices
    are $\mu$ and $\hmu$), and let $\epsilon \in [0,1]$ be given.
    Define a permissible set of densities (with respect to $\nu$)
    \[
        \cD_\epsilon(\nu) :=
        \{
            p \in L^1(\nu)
            :
            p\geq 0\ \nu\textup{-a.e.},
            \|p\|_1 = 1,
            \|p\|_\infty \leq 1 / \epsilon
        \},
    \]
    with the convention $1/0 = \infty$ in the case $\epsilon=0$.
    Additionally define
    \[
        \gamma_\epsilon(\nu) :=
        \inf_{p\in \cD_\epsilon(\nu)}
        \sup \left\{
            - \int (A\lambda)p d\nu
            :
            \lambda \in \Lambda, \|\lambda\|_1\leq 1
        \right\}.
    \]
    (When $\nu$ is a discrete measure, $\gamma_{\epsilon}(\nu)$ is almost equivalent to
    $\textup{AvgMin}_k$ as developed by \citet[Section 4.1]{shai_singer_weaklearn_linsep}.)
\end{definition}



This quantity will play a role analogous to the \emph{weak learning rate} in
AdaBoost, which guarantees the algorithm makes speedy progress in certain
separable cases. The correspondence between these two quantities will occupy much
of this section; but first, note primary guarantee in the separable case.

\begin{theorem}\label[theorem]{fact:sep:basic}
    Let $\ell\in\Lbds$ (with parameters $\beta_1,\beta_2,B_2$)
    and
    any $\cH\in\Fvc$ be given,
    and suppose $\bar \cL = 0$.
    Let any error tolerance $\epsilon \in (0,1]$
    and any confidence parameter $\delta \in (0,1]$ be given,
    and for convenience set $\epsilon' := \epsilon\beta_1/\beta_2$;
    by these choices, $\gamma_{\epsilon'}(\mu) > 0$.
    Suppose \Cref{alg:alg:alg} is run with stopping parameter $a\in(0,1)$,
    and the sample size $m$ satisfies
    \[
        m \geq \max\left\{
            \frac {2}{(\epsilon \gamma_{\epsilon'}(\mu))^2}
            \ln \left(\frac 4 \delta\right)
            \ ,\
            \left(
                \frac {24B_2\ell(0)}{
                    (\rho\beta_1\epsilon\gamma_{\epsilon'}(\mu))^2
                }
            \right)^{1/a}
        \right\}.
    \]
    Then, with probability at least $1-\delta$, the algorithm's
    output $\hat\lambda$ satisfies
    \[
        \cR(H\hat\lambda)
        \leq \epsilon
        + 2 \sqrt{
            \epsilon \frac {(\cV(\cH)+1)\ln(2em) + \ln(8/\delta)}{m^{1-a}}
        }
        + 4 \frac {(\cV(\cH)+1)\ln(2em)  + \ln(8/\delta)}{m^{1-a}}.
    \]
\end{theorem}

To simplify this bound, first note that, for the logistic loss, $B_2=1/4$,
$\beta_1 = 1/2$, and $\beta_2 = 1$ (cf. \Cref{fact:L:examples}).
Ignoring these terms, as well as $\ell(0)=1$, $\rho$ (which can be set to 1/2),
and $\gamma_{\epsilon'}(\mu)$; the choices $a=1/2$ and $\epsilon :=\cO(m^{-a/2}) =
\cO(m^{-1/4})$
grant that $\cO(1/\epsilon^2)$ iterations suffice to achieve classification
risk $\cO(1/\sqrt{m})$,
whereas the choices $a=2/3$ and $\epsilon := \cO(m^{-a/2}) = \cO(m^{-1/3})$
provide that $\cO(1/\epsilon^3)$ suffice to achieve error $\cO(m^{-1/3})$.


\subsection{The Quantity $\gamma_\epsilon(\nu)$}

To develop the meaning and necessity of $\gamma_\epsilon(\nu)$, first recall
the classical definitions associated with \emph{weak learnability}
(adjusted here so that ``binary'' means $\{-1,+1\}$ and not $\{0,1\}$).
\begin{definition}(\citet[Chapter 2]{schapire_freund_book_final}.)
    A class $\cH$ is \emph{weakly PAC-learnable with rate $\gamma$} if for any measure $\nu$ over
    $\cX\times \{-1,+1\}$, there exists $h\in \cH$ with $\int yh(x) d\nu(x,y) \geq \gamma$.
    Additionally, class $\cH$ and empirical measure $\nu$ are
    \emph{empirically weakly learnable with rate $\gamma(\nu)$} if there
    exists $h\in \cH$ so that $\int yh(x) p(x,y) d\nu(x,y) \geq \gamma(\nu)$
    for every reweighting $p$ of measure $\nu$.
\end{definition}

The definitions of $\gamma$ and $\gamma(\nu)$ are close to the definition of
$\gamma_\epsilon(\nu)$: the latter replaces the quantifiers and inequalities
with explicit infima and suprema, which grants the following correspondence.
\begin{proposition}
    \label[proposition]{fact:gamma:basic:equiv}
    Let class $\cH$,
    probability measure $\mu$ (over $\cX\times\{-1,+1\}$),
    and empirical counterpart $\hmu$ be given.
    Then the weak PAC-learning rate $\gamma$ satisfies $\gamma \leq \gamma_0(\mu)$,
    and the empirical weak learning rate $\gamma(\hmu)$ satisfies
    $\gamma(\hmu) \leq \gamma_0(\hmu)$.
\end{proposition}

The following example highlights why $\gamma_0(\nu)$ can be problematic,
even when $\bar \cL=0$.



\begin{example}[Nightmare scenario \#1]
    \label[example]{ex:sep:nightmare}
    Suppose $\cX = (0,1]$, and
    \[
        \Pr[ y = +1 | x] = \begin{cases}
        1 &\textup{when $x \in(1/(i+1),1/i]$ for some odd integer $i$},
        \\
        0 &\textup{when $x \in(1/(i+1),1/i]$ for some even (positive) integer $i$}.
        \end{cases}
    \]
    Let $\cH$ consist of threshold functions (decision stumps).
    Given any integer $k$, a combination of thresholds may be constructed
    which is correct on $k$ intervals, and thus $\bar \cL =0$ by considering
    $k\to\infty$.
    Unfortunately, the norm of these solutions also grows unboundedly,
    suggesting
    $\gamma$ and $\gamma(\mu)$ are tiny.
    Indeed, consider a distribution over $\cX$ which is uniform on $k$ of the intervals,
    and zero elsewhere.  Any threshold is incorrect on nearly half of these intervals,
    and by considering $k\to\infty$, it follows that $\gamma \leq \gamma_0(\mu) = 0$.
\end{example}

In precise terms, this nightmare, and suggested sequence of distributions, provide
the following property.
\begin{proposition}
    \label[proposition]{fact:gamma:basic:bad}
    There exist choices for $\cH$ and $\mu$ so that $\bar \cL =0$,
    but $\gamma \leq \gamma_0(\mu) = 0$ and, with any probability $1-\delta$
    and sample size $m$ large enough that
    $k:= \lfloor m^{1/4} / (3\sqrt{2\ln(4/\delta)})\rfloor$ satisfies $k \geq 2$,
    then
    $\gamma(\hmu) \leq \gamma_0(\hmu) \leq \gamma_\epsilon(\hmu)
    \leq \cO((\ln(4/\delta) + \ln(m))^{1/2}  / m^{1/4})$,
    where $\epsilon = \cO(1/k) = \cO(m^{-1/4}/ \ln(4/\delta))$ and the $\cO(\cdot)$
    only suppresses terms independent of $m$ and $\delta$.
\end{proposition}

But something is wrong here --- \Cref{ex:sep:nightmare} seems quite easy!
The reason $\gamma$ indicates otherwise is
that it simply tries too hard: \Cref{ex:sep:nightmare}
is easy if giving up on an $\epsilon$-fraction of the data is acceptable.  This reasoning
leads to the relaxation $\gamma_\epsilon(\nu)$, which, in contrast
to \Cref{fact:gamma:basic:bad}, carries the following guarantee.

\begin{proposition}
    \label[proposition]{fact:gamma_epsilon:basic}
    Let probability $\mu$ over $\cX\times \{-1,+1\}$ and class $\cH$ be given.
    \begin{enumerate}
        \item Let loss $\ell\in \Lbds$ be given.
        Then $\bar \cL = 0$ iff $\gamma_\epsilon(\mu) > 0$ for all $\epsilon > 0$.
        \item Let any $\epsilon>0$, confidence parameter $\delta \in (0,1]$,
            and empirical measure $\hmu$ be given.
            Then with probability at least $1-\delta$,
            \[
                \gamma_\epsilon(\hmu) \geq
                \gamma_{\epsilon}(\mu)
                - \frac 1 {\epsilon} \sqrt{
                    \frac{1}{2m} \ln \left (\frac 2 \delta\right)
                }.
            \]
    \end{enumerate}
\end{proposition}

In order to prove this result, and also a few other components in the proof
of \Cref{fact:sep:basic}, the following dual representation of $\gamma_\epsilon(\nu)$ is
used.
A similar result was proved by
\citet[see the quantity $\textup{AvgMin}_k$]{shai_singer_weaklearn_linsep}
in the case of measures with finite support
and finite cardinality hypothesis classes; the proof here invokes Sion's Minimax
Theorem \citep{komiya_sion}, which operates in fairly general topological vector
spaces.

\begin{lemma}
    \label[lemma]{fact:gamma_epsilon:duality:body}
    Let probability measure $\nu$ over $\cX\times \{-1,+1\}$,
    any $\cH$,
    and any $\epsilon \in (0,1]$ be given.
    Then
    \begin{align*}
        \gamma_\epsilon(\nu)
        &=
        \min_{p\in\cD_\epsilon(\nu)}
        \sup\left\{
            \int (A\lambda) p d\nu
            :
            \lambda \in \Lambda,
            \|\lambda\|_1\leq 1
        \right\}
        \\
        &=
        \sup\left\{
            \min_{p\in\cD_\epsilon(\nu)}
            \int  (A\lambda) p d\nu
            :
            \lambda \in \Lambda,
            \|\lambda\|_1\leq 1
        \right\}
        \\
        &=
        \min_{p\in\cD_\epsilon(\nu)}
        \|A^\top p\|_\infty,
    \end{align*}
    where $A^\top$ is the unique adjoint operator to $A$ (cf. \Cref{fact:AT}),
    and
    \[
        \|A^\top p\|_\infty = \sup\left\{
            \left|\int y h(x) p(x,y) d\nu(x,y)\right|
            :
            h\in \cH
        \right\}.
    \]
\end{lemma}

In order to use this to prove the first part of \Cref{fact:sep:basic},
first note that whenever $\gamma_\epsilon(\nu)=0$,
there exists a dual element $p\in \cD_\epsilon(\nu)$ certifying this property, which in
turn can be related to the duality structure of $\cL$ (presented
later in \Cref{fact:duality:body}), and gives the result.
For the second part of \Cref{fact:sep:basic},
similarly the infimum in the definition
of $\gamma_\epsilon(\nu)$ can be removed
by considering a single bad certificate $p\in\cD_{\epsilon}(\mu)$,
and the supremum can be removed by
considering a single good $\lambda \in \Lambda$.  The certificate $p$ can be
shown to have a simple structure (it emphasizes margin violations for the fixed good
$\lambda$), and in turn the deviations are easy to control.

\subsection{Proof Sketch of \Cref{fact:sep:basic}}

The pieces are in place to establish the finite sample guarantees in \Cref{fact:sep:basic}.
First, note the following empirical risk guarantee.

\begin{lemma}
    \label[lemma]{fact:sep:iterhelper}
    Let any $\ell \in \Lbds$, empirical measure $\hmu$, and $\cH$ be given.
    Suppose \Cref{alg:alg:alg} is run with any of the three step size choices
    for $T$ iterations, let $\epsilon_t = \hcR(H\lambda_t)$
    denote the classification error of $H\lambda_t$,
    and set $\epsilon'_t := \epsilon_t\beta_1/\beta_2$ for convenience.
    Then
    \[
        \hcL(A\lambda_T) \leq
        \ell(0)
        -
        \sum_{t=1}^{T}
        \frac{(\rho\beta_1\epsilon_{t-1}\gamma_{\epsilon_{t-1}'}(\hmu))^2}{6B_2}.
    \]
\end{lemma}
Notice that this result indicates that the convex risk decreases quickly in the presence
of classification errors.
The proof, sketched as follows, is fairly straightforward.  First, standard properties
of the line search choices show that $\hcL$ drops in round $t$ proportionally
to $\|A^\top\nhcL(A\lambda_{t-1})\|_\infty$.
Considering $\nhcL(A\lambda_{t-1})$ as a reweighting of $\hmu$,
this expression appears in the dual form of dual form of $\gamma_{\epsilon}(\hmu)$ as
presented in \Cref{fact:gamma_epsilon:duality:body}.
In order to make the correspondence precise, $\nhcL(A\lambda_{t-1})$ must be rescaled
to unit norm; but, by the Lipschitz property, the rescaling is by at most $\beta_1\epsilon_{t-1}$!  After some algebra, and summing across all iterations, the result follows.

From here, there is little to do.  By \Cref{fact:sep:iterhelper}, until some iteration
has low error, progress is quick.  The selection rule
(returning $\hat\lambda$ with minimal
classification risk) ensures there are no problems if the classification risk happens
to go back up, and
\Cref{fact:gamma_epsilon:basic} allows $\gamma_\epsilon(\mu)$ to replace
$\gamma_\epsilon(\hmu)$.
As this reasoning provides a direct guarantee on the empirical classification
risk, standard uniform convergence techniques give the result.

\section{The Nonseparable Case ($\bar\cL > 0$)}
\label{sec:nonseparable}

When $\bar\cL >0$, the essential object will be an optimum to the convex dual of
the central optimization problem $\inf_\lambda \cL(A\lambda)$, specified as follows.

\begin{proposition}
    \label[proposition]{fact:duality:body}
    Let loss $\ell\in \Lbds$ (with tightest Lipschitz parameter $\beta_2$),
    class $\cH$,
    and probability measure $\nu$ over $\cX\times \{-1,+1\}$ be given.
    Then
    \[
        \inf \left\{ \int \ell(A\lambda) d\nu : \lambda \in \Lambda\right\}
        =
        \max\left\{ - \int \ell^*(p)
            : p \in L^\infty(\nu), p\in [0,\beta_2]\ \nu\textup{-a.e.},
            \|A^\top p\|_\infty = 0
        \right\},
    \]
    where $\ell^*$ is the Fenchel conjugate to $\ell$, and the adjoint
    $A^\top$ is as in \Cref{fact:gamma_epsilon:duality:body} and \Cref{fact:AT}.
    Additionally, the dual optimum $\bar p$ satisfies
    $\mu([\bar p \geq \tau]) \geq \tau$,
    where $\tau > 0$ whenever the optimal value $\bar\cL_\nu$ is positive,
    and moreover $\tau$ has the explicit form
    $\tau := \ell^{-\star}(-\bar\cL_\nu)/2$,
    where $\ell^{-\star}$ is the (well-defined) inverse of $\ell^*$ along
    $[0, \ell'(0)]$.
\end{proposition}

The strategy in the nonseparable case is to exhibit curvature in the objective
function (i.e., a lower bound on the second-order expression
$m^{-1} \sum_i \ell((A\lambda)_{x_i,y_i})$),
and the dual optimum $\bar p$ will be provide the mechanism.
Making these statement precise is the topic of this section, however, for the time
being, note that the dual problem resembles a maximum entropy problem, where
the constraint $\|A^\top p\|_\infty=0$
requires reweightings (including $\bar p$) to decorrelate all predictors from the target,
and the objective function $-\int\ell^*$ prefers weightings which are large and
close to uniform (cf. \Cref{fact:loss:basic}; in the case of the logistic loss $\llog$,
these statements are fairly concrete:
$\llog^*(\phi) = \phi\ln(\phi) + (1-\phi)\ln(1-\phi)$,
the Fermi-Dirac Entropy).

\begin{theorem}
    \label[theorem]{fact:nonsep:basic}
    Let loss $\ell \in \Lbds\cap \Ltd$,
    binary class $\cH \in \Fvc$,
    probability measure $\mu$ over $\cX\times \{-1,+1\}$ with empirical
    counterpart $\hmu$ corresponding to a sample of size $m$,
    time horizon $t\leq m^a$ with $a\in (0,1)$,
    and any confidence $\delta \in (0,1]$ be given.
    Suppose $\bar\cL>0$,
    and let $\bar p\in L^\infty(\mu)$ denote the dual optimum as
    in \Cref{fact:duality:body}, with corresponding real number $\tau$
    so that $\mu([p\geq \tau])\geq \tau$.
    Define the quantities
    \begin{align*}
        c
        &:= \frac {16 \ell(0)}{\tau \ell'(0)}
        \max\left\{1,  \frac 1 \tau \right\}
        \max\left\{1,  \|p\|_\infty \right\},
        &B_1
        &:= \frac \tau 8 \inf_{z\in [-c,+c]} \ell''(z),
        \\
        R_i
        &:= \sqrt{i} \sqrt{\frac{\ell(0)\max\{5,2B_1/B_2\}}{\rho^2 B_1}},
    \end{align*}
    and suppose the sample size
    is large enough to satisfy
    $m \geq (2 /\tau^2)\ln(4/\delta)$ and
    \[
        \frac{2(R_t + 2c)\|\bar p\|_\infty}{m^{1/2}}\left(
            2\sqrt{2\cV(\cH)\ln(m+1)} +\sqrt{2\ln(4/\delta)}
        \right) \leq \frac {c\ttil^2}{8}
    \]
    (which happens for all large $m$ since $\lim_{m\to\infty} R_t\sqrt{\ln(m) / m}
    = \lim_{m\to\infty} \sqrt{\ln(m) / m^{1-a}} = 0$).
    Then it holds that the above values $\tau$, $c$, $B_1$, and
    $R_i$ (for $0< i \leq t$) are all positive, and moreover the following
    statements hold simultaneously with probability at least $1- \delta$.
    \begin{enumerate}
        \item The final coefficient vector $\lambda_t\in\Lambda$ satisfies
            \begin{align*}
                \cL(A\lambda_{t})
                &\leq \inf_{\|\lambda\|_1 \leq R_{t-1}} \cL(A\bar\lambda)
                + m^{-a/4}
                + R_{t-1} \sqrt{\frac 2 m \ln\left( \frac 6 \delta\right)}
                \\
                &\qquad
                + \frac {2\beta_2 R_t}{m^{1/2}}
                \left(2\sqrt{2\cV(\cH)\ln(m+1)} + \ell(R_t)\sqrt{2\ln(6/\delta)}\right)
                \\
                &\qquad
                + \ell(0)
                \left(
                    \frac
                    {\|\bar \lambda\|_1}
                    {\|\bar \lambda\|_1 + \rho m^{a/4} / (4B_2R_1)}
                \right)^{9B_2/ (B_1\rho^3)}.
            \end{align*}
        \item If $\cH\in \Fds(\mu^\cX)$ (where $\mu^\cX$ is the marginal of $\mu$ over
            $\cX$), and letting $\cR_\star$ denote the Bayes error rate,
            there exists $\psi:\R\to\R$ satisfying
            $
                \psi\left(
                    \cR(A\lambda_t) - \cR_\star
                \right)
                \leq
                 \cL(A\lambda_t) - \bar \cL 
            $
            and $\psi(z) \to 0$ as $z\to 0$. (For instance, when $\ell=\llog$,
            then $\psi(z) = z^2/2$.)
        \item The returned coefficients $\hat\lambda$ satisfy
            \begin{align*}
                \cR(H\hat\lambda)
                &\leq \cR(H\lambda_t)
                + 4 \sqrt{
                    \hcR(H\lambda_t)\frac {(\cV(\cH)+1)\ln(2em) + \ln(24/\delta)}{m^{1-a}}
                }
                \\
                &\qquad
                + 8 \frac {(\cV(\cH)+1)\ln(2em)  + \ln(24/\delta)}{m^{1-a}}.
            \end{align*}
    \end{enumerate}
\end{theorem}

This bound is inferior to the guarantee in the separable case; while it is still of
the form $1/m^{1/c}$, the exponent $1/c$ is distribution-dependent.
The source of weakness is the optimization guarantee (cf.
\Cref{fact:nonsep:opt:zhangyu_style}), which is brute-forced and should be improvable.

\subsection{Curvature}

Recall that the dual optimum $\bar p$ satisfies
$\|A^\top \bar p\|_\infty$, which implies $\int (A\lambda)\bar p d\mu = 0$ for
every $\lambda\in\Lambda$ (cf. \Cref{fact:AT}).  To see how this helps locate
bad examples and produce curvature, note the rearrangement
\[
    - \int_{[A\lambda < 0]} (A\lambda) \bar p d\nu
    =
    \int_{[A\lambda \geq 0]} (A\lambda) \bar p d\nu,
\]
meaning $\nu$ has been reweighted by $\bar p$ so that negative and positive
margins are equal (in a sense, $\bar p$ renders every $\lambda \in\Lambda$ equivalent
to random guessing).  Since $\bar p$ is fairly well-behaved (it is within
$[0,\beta_2]$ $\mu$-a.e. (where $\beta_2$ is the Lipschitz constant for $\ell$),
and is fairly flat since $\mu([\bar p \geq \tau]) \geq \tau$), then some algebra
allows the removal of $\bar p$ from the above display, which yields the statement:
if $A\lambda$ has many good margins, it also has many bad margins.  This constrains
the norms of solutions found by the algorithm, and generates curvature in the sense
that progress in any direction quickly leads to $\cL$ increasing.

Of course, $\bar p$ could have instead been directly constructed from the presence of noise,
but then the results would not be applicable to cases where $\nu$ itself is noiseless,
but $\cH$ is simply very weak.  The following example emphasizes this role of noise,
but also shows that the above development overlooked the effect of sampling.

\begin{example}[Nightmare scenario \#2]
    \label[example]{ex:nonsep:nightmare}
    Pick any $\cX$, (marginal) distribution $\mu^\cX$ over $\cX$,
    hypothesis class $\cH\in \Fds(\nu^\cX)$,
    and any $\bar\lambda \in \Lambda$.  Define the conditional
    density $\Pr[y =+1| x]$ to be 0.9 when $(H\bar\lambda)_x \geq 0$,
    and 0.1 otherwise when $(H\bar\lambda)_x < 0$.  By this construction, $H\bar\lambda$
    attains
    the Bayes error rate (which is 0.1), and every other $H\lambda$ does at best this well.
    Any weighting $\lambda$ with favorable convex risk $\cL(A\lambda)$ will necessarily
    have a small norm in consequence of the guaranteed 10\% classification error.

    Unfortunately, finite samples look slightly different.
    Suppose $\cX=\R^d$ and $\mu^{\cX}$ is absolutely continuous with respect to Lebesgue
    measure.  With probability 1, a random sample of any size
    will contain no noise, and $\SPAN(\cH)$ has
    a perfect predictor $H\tilde\lambda$ (over the sample);
    in particular, nothing inhibits the norms of solutions over $\hcL$.
\end{example}

In this example, the good predictor $H\tilde\lambda$ is potentially very complex,
as it is fitting noise.  The solution here will be to only control those predictors
with small norms; note that this deviation inequality embeds the reweighted average
margin expression from \cref{eq:pdm}.

\begin{lemma}
    \label[lemma]{fact:pdc:body}
    Let probability measure $\mu$ over $\cX\times \{-1,+1\}$ with empirical counterpart
    $\hmu$,
    any hypothesis class $\cH\in \Fvc$,
    reweighting $p \in L^\infty(\mu)$ with $\|A^\top p \|_\infty = 0$,
    and norm bound $C$ be given.
    Then, with probability at least $1-\delta$,
    $p\in [0,\|p\|_\infty]$ $\hmu$-a.e., and
    \[
        \sup_{\substack{ \lambda \in \Lambda \\ \|\lambda\|_1\leq C}}
        \left|
        \int (A \lambda) p d\hmu
        \right|
        \leq \frac {2C\|p\|_\infty}{m^{1/2}}
        \left(2\sqrt{2\cV(\cH)\ln(m+1)} + \sqrt{2\ln(1/\delta)}\right).
    \]
\end{lemma}

Armed with these tools, the structure of the nonseparable problem is as follows.
Note that the term $B_1$ is the aforementioned curvature lower bound,
and furthermore the facts
$\sum_i \alpha_i = \infty$ and $\sum_i \alpha_i^2 < \infty$ mean that the step sizes
exactly fit the constrained step size regime studied by
\citet[Equation (4)]{zhang_yu_boosting}.

\begin{lemma}
    \label[lemma]{fact:nonsep:controls}
    Suppose the setting and quantities in the preamble of \Cref{fact:nonsep:basic};
    the following statements hold simultaneously with probability at least $1- \delta/2$.
    \begin{enumerate}
        \item
            Every $\lambda\in \Lambda$ with
            $\|\lambda\|_1 \leq R_t + 4c$
            and
            $\hcL(A\lambda) < 2\ell(0)$
            has
            $m^{-1} \sum_{i=1}^m \ell''((A\lambda)_{x_i,y_i}) \geq B_1$.
        \item
            For every choice of step size,
            $\|\lambda_i\|_1 \leq R_i$
            and
            \[
                \alpha_i^2
                \leq
                \min\left\{
                \frac {9\|A^\top \nhcL(A\lambda_{i-1})\|_\infty^2}{4\rho^2 B_1^2}
                    \ ,  \ 
                \frac {\max\{5,2B_2/B_1\}(\hcL(A\lambda_{i-1}) - \hcL(A\lambda_i))}
                {\rho^2 B_1}
            \right\}
            .
            \]
        \item
            Let $\bar\lambda \in \Lambda$ with
            $\|\bar\lambda\|_1 \leq R_1 \sqrt{t-1} = R_{t-1}$
            and
            $\epsilon := \min_{i\in [t-1]} \hcL(A\lambda_i) - \hcL(A\bar\lambda) \geq 0$
            be arbitrary.
            For every choice of step size,
            \[
                \alpha_i
                \geq \frac {\rho (\hcL(A\lambda_{i-1}) - \hcL(A\bar\lambda))}
                {2B_2 (\|\bar\lambda\|_1 + R_1\sqrt{i-1})}
                \qquad
                \textup{and}
                \qquad
                \sum_{i=1}^t \alpha_i \geq \frac {\rho \epsilon \sqrt{t-1}} {4B_2R_1}.
            \]
    \end{enumerate}
\end{lemma}

\subsection{Proof of \Cref{fact:nonsep:basic}}
The convergence analysis due to \citet{zhang_yu_boosting} can be adjusted to the
present setting (where step and coordinate selection are decoupled), yielding the
following guarantee.  Note that \Cref{fact:nonsep:controls} also allows the application
of the analysis due to \citet{bartlett_traskin_adaboost} (again with decoupling
modifications), however this leads to a rate of roughly $\cO(1 / \sqrt{\ln(m)})$.

\begin{lemma}
    \label[lemma]{fact:nonsep:opt:zhangyu_style}
    Let $\ell \in \Lbds \cap \Ltd$ with Lipschitz gradient
    parameter $B_2$,
    binary class $\cH$,
    time horizon $t$,
    and empirical probability measure $\hmu$ 
    be given.
    Let $\bar\lambda\in\Lambda$ be arbitrary,
    and
    suppose there exists $c_3>0$
    with
    $c_3 \alpha_i \leq \|A^\top \nhcL(A\lambda_{i-1})\|_\infty$ for
    for
    all $0 \leq i \leq t$.
    Then
    \[
        \hcL(A\lambda_{t}) - \hcL(A\bar\lambda)
        \leq
        \left(\hcL(A\lambda_{0}) - \hcL(A\bar\lambda)\right)
        \left(
            \frac
            {\|\bar \lambda\|_1}
            {\|\bar \lambda\|_1 + \sum_{i\leq t} \alpha_i}
        \right)^{6B_2/ (c_3\rho^2)}.
    \]
\end{lemma}

From here, there is little to do: the conditions for this rate are met with high
probability thanks to \Cref{fact:nonsep:controls}, and the rest is standard uniform
convergence.

\acks{This manuscript exists thanks to  valuable comments and support from
    Akshay Balsubramani, Sanjoy Dasgupta, Daniel Hsu, Alexander Rakhlin,
    Robert Schapire,
Karthik Sridharan, and the COLT 2013 reviewers.}

\addcontentsline{toc}{section}{References}
\bibliography{ab}

\appendix

\section{Spaces and Linear Operators}
\label{sec:linops}

As stated in \Cref{sec:notation}, $H$ and $A$ are mappings which produce bounded functions;
the bulk of the analysis, however,
considers them as producing functions over $L^1(\nu^\cX)$ and
$L^1(\nu)$ as follows (where $\nu$ is a probability distribution over $\cX\times \{-1,+1\}$).



\begin{lemma}
    \label[lemma]{fact:AH_basic}
    Let $\nu$ be a probability measure over $\cX\times \{-1,+1\}$,
    and let $\nu^\cX$ denote the marginal distribution over $\cX$.
    \begin{enumerate}
        \item The definition of $H$ and $A$ is valid for arbitrary
            weightings $\lambda \in \Lambda$; in particular,
            $\supp(\lambda)$ is countable,
            and
            \begin{align*}
                (H\lambda)_x
                &= \int h(x) \lambda(h) d\rho(h)
                = \sum_{h\in\supp(\lambda)} (H\lambda)_x \lambda(h),
                \\
                (A\lambda)_{x,y}
                &= -y\int (H\lambda)_x \lambda(h) d\rho(h)
                = -y\sum_{h\in\supp(\lambda)} (H\lambda)_x \lambda(h),
            \end{align*}
        \item $H$ and $A$ are linear operators.
        \item $H : \Lambda \to L^1(\nu^\cX)$
            and $A : \Lambda \to L^1(\nu)$
            are continuous linear operators (with unit norm).
    \end{enumerate}
\end{lemma}
\begin{proof}
    If $\lambda \in \Lambda$, then $\int |\lambda(h)|d\rho(h)| < \infty$, and since
    $\rho$ is a counting measure, it follows that $\supp(\lambda)$ is countable.
    Furthermore, for any $x,y,h$,  $|h(x)| \leq |-yh(x)| \leq 1$, and
    thus the rescalings $h(x)\lambda$ and $-yh(x)\lambda)$ are both in $\Lambda$,
    and in particular
    \[
        (H\lambda)_x
        = \int h(x) \lambda(h) d\rho(h)
        = \sum_{h\in\supp(\lambda)} (H\lambda)_x \lambda(h),
    \]
    and similarly for $A$.

    It follows by definition (and another check for integrability)
    that $(H(a\lambda_1 + b\lambda_2))_x = a(H\lambda_1)_x + b(H\lambda_2)_x$,
    and thus $H$ is a linear operator; the proof for $A$ is the same.

    Lastly, $H$ is continuous with unit norm, since boundedness of each $h$ combined
    with $\nu^\cX$ being a probability measure gives
    \begin{align*}
        \sup \left\{ \|H\lambda\| : \lambda\in \Lambda, \|\lambda\|_1 \leq 1\right\}
        &=
        \sup \left\{ \int (H\lambda)d\nu^\cX : \lambda\in \Lambda, \|\lambda\|_1 \leq 1\right\}
        \\
        &\leq
        \sup \left\{ \int \sum_{h\in \supp(\lambda)}  h(x) \lambda(h)d\nu^\cX(x) :
        \lambda\in \Lambda, \|\lambda\|_1 \leq 1\right\}
        \\
        &\leq
        \sup \left\{ \sup_{x,h} |h(x)| \left|\sum_{h\in \supp(\lambda)}\lambda(h)\right| :
        \lambda\in \Lambda, \|\lambda\|_1 \leq 1\right\}
        \\
        &=
        1.
    \end{align*}
    (The proof for $A$ is the same, since $y\in \{-1,+1\}$ implies $|yh(x)| = |h(x)|$.)
\end{proof}

Note, $H$ and $A$ may also be defined as Bochner (or similar) integrals.

Next, to develop the adjoint of $A^\top$, relevant dual spaces need to be
established (the adjoint of $H^\top$ does not appear, but is similar).

\begin{lemma}
    \label[lemma]{fact:banach_duality}
    If $\nu$ is a probability measure over $\cX\times \{-1,+1\}$,
    then $L^1(\nu)^*$ is isometrically isomorphic to $L^\infty(\nu)$, and in particular
    for every $Q\in L^1(\nu)^*$ there exists $q\in L^\infty(\nu)$ so that
    $Q(f) = \int qfd\nu$ for every $f\in L^1(\nu)$.  Similarly,
    recalling $\Lambda = L^1(\rho)$ where $\rho$ is counting measure over some class
    $\cH$, the dual $L^1(\rho)^*$ is isometrically isomorphic to $L^\infty(\rho)$,
    and once again elements of $L^1(\rho)^*$ can be written as integrals over $\rho$
    with an element of $L^\infty(\rho)$.
\end{lemma}
\begin{proof}
    The first relationship follows since $\nu$ is a probability measure and thus
    $\sigma$-finite
    \citep[Theorem 6.15]{folland},
    and the second is a general property of counting measures (even though the cardinality
    of $\cH$ may preclude $\rho$ from being $\sigma$-finite)
    \citep[Exercises 3.15 and 6.25]{folland}.
\end{proof}

\begin{remark}
    This manuscript \emph{always} identifies the above dual spaces by the provided
    isometric isomorphism, a fact which will be crucial in the
    convex duality theory of $\cL$
    (cf. \Cref{fact:duality}).
\end{remark}

Lastly, the adjoint $A^\top$ has the following structure.

\begin{lemma}
    \label[lemma]{fact:AT}
    Let probability measure $\nu$ over $\cX\times \{-1,+1\}$
    and any $\cH$ be given.
\begin{enumerate}
\item
    Considering $A$ as a linear operator from $\Lambda$ to $L^1(\nu)$,
    its adjoint $A^\top: \Lambda^* \to L^1(\nu)^*$
    is the unique continuous linear operator satisfying $(A^\top p)(\lambda)
    = \int (A\lambda) p$, where $p\in L^\infty(\nu)$ and $\lambda \in \Lambda$
    (and dual spaces have been identified via isomorphism
    as in \Cref{fact:banach_duality}).
%
\item
    Again identifying $A^\top p$ for $p\in L^\infty(\nu)$ with an element
    of $L^\infty(\rho)$,
    \begin{align*}
        \|A^\top p\|_\infty
        &= \sup\left\{
            \left|\int y h(x) p(x,y) d\nu(x,y)\right|
            :
            h\in \cH
        \right\}
        \\
        &= \sup\left\{
            \int (A\lambda)_{x,y} p(x,y) d\nu(x,y)
            :
            \lambda \in \Lambda, \|\lambda\|_1 \leq 1
        \right\}.
    \end{align*}
\item
    The map $p \mapsto \|A^\top p\|_\infty$
    is a convex function over $L^\infty(\nu)$, and is lower semi-continuous in the
    weak* topology (i.e., the weak topology induced on $L^\infty(\nu)$ by $L^1(\nu)$).
    \end{enumerate}
\end{lemma}
\begin{proof}
    \begin{enumerate}
        \item
    Recall by \Cref{fact:AH_basic} that $A$ is a continuous linear operator;
    the basic properties of $A^\top$ follow
    by properties of adjoints of continuous linear operators
    \citep[Theorem 4.10]{rudin_functional} combined with the isometric isomorphism of
    the relevant dual spaces as provided by \Cref{fact:banach_duality}.


\item
    Let $p\in L^\infty(\nu)$ be given.
    Since the isometric isomorphism provided by
    \Cref{fact:banach_duality} allows $A^\top p$ to be identified with an element of
    $L^\infty(\rho)$, the operator norm of $A^\top p$ is simply the $L^\infty(\rho)$
    norm of the element it has been identified with by the isomorphism.
    Since $\rho$ is a counting measure,
    letting $\bfe_h\in \Lambda$ be an indicator function for a single $h\in\cH$
    (it is 1 on $h$ and 0 elsewhere),
    using the above adjoint relation $(A^\top p)(\bfe_h) = \int (A\bfe_h) p$,
    and using the definition of norms on $L^\infty(\rho)$,
    \begin{align*}
        \|A^\top p\|_\infty
        &= \inf
        \left\{a \geq 0 : \rho( \{ h\in \cH : |(A^\top p)(\bfe_h)| > a\}) = 0\right\}
        \\
        &= \inf
        \left\{a \geq 0 : \rho( \{ h\in \cH : \left|\int -y h(x)p(x,y)d\nu(x,y)\right| > a\}) = 0\right\}
        \\
        &= \inf
        \left\{a \geq 0 : \forall h\in \cH, \left|\int -y h(x)p(x,y)d\nu(x,y)\right| \leq a\right\}
        \\
        &= \sup_{h\in\cH}\left|\int -y h(x)p(x,y)d\nu(x,y)\right|,
    \end{align*}
    where the last equality can established by
    noting the domain of the infimum includes all $a\geq 0$ satisfying
    $a\geq \sup_{h\in\cH}\left|\int -y h(x)p(x,y)d\nu(x,y)\right|$,
    but no values satisfying $a< \sup_{h\in\cH}\left|\int -y h(x)p(x,y)d\nu(x,y)\right|$.

    Next, to show
    \begin{align*}
        &\sup\left\{
            \left|\int y h(x) p(x,y) d\nu(x,y)\right|
            :
            h\in \cH
        \right\}
        \\
        &\qquad\qquad= \sup\left\{
            \int (A\lambda)_{x,y} p(x,y) d\nu(x,y)
            :
            \lambda \in \Lambda, \|\lambda\|_1 \leq 1
        \right\},
    \end{align*}
    one direction is immediate, since
    positive and negative copies $\pm \bfe_h$ of the indicator elements satisfy
    $\pm \bfe_h \in \Lambda$ and $\|\pm \bfe_h\|_1 = 1$.
    For the other direction,
    let $\tau > 0$ be arbitrary,
    and choose any $\lambda_\tau\in \Lambda$ which is within $\tau$ of the
    supremum on the right side of the display.
    Then,
    since $|\supp(\lambda)|$ is countable (via \Cref{fact:AH_basic}),
    and since $\|\lambda_\tau\|_1 \leq 1$ implies
    $\|A\lambda_\tau\|_1 \leq 1$, the dominated convergence
    theorem \citep[Theorem 2.25 (summation form)]{folland} may be applied (with dominating
    function 1), and
    \begin{align*}
        &\sup\left\{
            \int (A\lambda)_{x,y} p(x,y) d\nu(x,y)
            :
            \lambda \in \Lambda, \|\lambda\|_1 \leq 1
        \right\}
        \\
        &\qquad\leq
        \tau +
        \int (A\lambda_\tau)_{x,y} p(x,y) d\nu(x,y)
        \\
        &\qquad=
        \tau +
        \int \sum_{h\in \supp(\lambda_\tau)}-y h(x) \lambda_\tau(h) p(x,y) d\nu(x,y)
        \\
        &\qquad=
        \tau + \sum_{h\in \supp(\lambda_\tau)}
        \lambda_\tau(h)
        \int -y h(x) p(x,y) d\nu(x,y)
        \\
        &\qquad\leq \tau + \|\lambda_\tau\|_1\sup\left\{
            \left|\int y h(x) p(x,y) d\nu(x,y)\right|
            :
            h\in \cH
        \right\};
    \end{align*}
    since $\|\lambda_\tau\|_1 \leq 1$,
    and since $\tau>0$ was arbitrary, the result follows.
%

\item
    For the last part, define a convex indicator over $L^1(\nu)$ as
    \[
        \iota(f) = \begin{cases}
            0 & \textup{when }
            f \in \{ A\lambda : \lambda \in \Lambda, \|\lambda\|_1\leq 1\}, \\
            \infty & \textup{otherwise}.
        \end{cases}
    \]
    (Note that $\iota$ is not necessarily lower semi-continuous over $L^1(\nu)$,
    since as discussed shortly in \Cref{fact:H:notclosed}, the subspace $A\Lambda$ might not
    be closed.)
    The conjugate of $\iota$ is, for any $p\in L^\infty(\nu)$,
    \begin{align*}
        \iota^*(p)
        &=
        \sup\left\{ \int fp : \exists \lambda \in \Lambda, \|\lambda\|_1 \leq 1\centerdot
        f = A\lambda\right\}
        \\
        &=
        \sup\left\{ \int (A\lambda)p : \lambda \in \Lambda, \|\lambda\|_1 \leq 1\right\}
        \\
        &=
        \|A^\top p\|_\infty,
    \end{align*}
    where the last step used the earlier equalities for $A^\top$.
    Since $p\mapsto \|A^\top p\|_\infty$ is the conjugate of a convex function,
    it is lower semi-continuous in the weak* topology \citep[Theorem 2.3.1(i)]{zalinescu}.
    \end{enumerate}
\end{proof}

Lastly, note the following properties of the sets $H\Lambda$ and $A\Lambda$.

\begin{lemma}
    \label[lemma]{fact:H:notclosed}
    Let any $\cH$ and any probability measure $\nu$ over $\cX\times \{-1,+1\}$ with
    marginal $\nu^\cX$ over $\cX$ be given.  Then $H\Lambda$ and $A\Lambda$ are subspaces,
    but it is possible that neither is closed in its respective
    $L^1(\nu^\cX)$ and $L^1(\nu)$
    topology
    (indeed, \Cref{ex:sep:nightmare} provides the counterexample).
\end{lemma}
\begin{proof}
    Since $\Lambda$ is a Banach space and $H$ and $A$ are linear operators, it follows
    that $H\Lambda$ and $A\Lambda$ are subspaces.

    For the lack of closure, consider the setting of
    \Cref{ex:sep:nightmare},
    and in particular building a sequence of functions $\{f_k\}_{k=1}^\infty$ which
    are a combination of $k$ thresholds, and predict correctly on the last $k$ intervals.
    This sequence has a limit point in $L^1(\mu^\cX)$ (in particular, it is a countable
    sum of indicators over intervals), but no such function is in $H\Lambda$, which
    is therefore not closed in $L^1(\nu^\cX)$.  To obtain a similar result for $A\lambda$,
    define $g_k(x,y) := f_k(x)$.
\end{proof}

\section{The Family of Dense Classes $\Fds(\nu)$}
\label{sec:Fds}

As the goal of a consistency analysis is to show that the Bayes predictor is
approximated arbitrarily finely, necessarily the function class considered by
a purportedly consistent algorithm must be very large.

As discussed in \Cref{sec:notation}, one choice is the class $\Fds(\nu)$ of functions
dense according to $L^1(\nu)$ in the family of bounded measurable functions.  A partial
survey of density assumptions in other work is as follows.
\begin{itemize}
    \item
        \citet[Definition 1]{breiman_infinity_theory}
        works with a similar definition: the relevant metric is $L^2(\nu)$,
        and the closure must contain $L^2(\nu)$, where $\nu$ is constrained to
        be continuous with respect to Lebesgue measure.
        By contrast, the metric for $\Fds(\nu)$
        is $L^1(\nu)$, where $\nu$ is an arbitrary measure over the Borel $\sigma$-algebra,
        and the closure of the class must contain bounded measurable functions,
        which are a subspace of $L^\infty(\nu)$, which is contained within $L^2(\nu)$.

        \Cref{fact:Fds:cubicles}, which will be proved shortly, states that it suffices
        for $\SPAN(\cH)$ to contain boxes formed by half-open intervals.
        This result was stated by \citet[Proposition 1]{breiman} with an
        abbreviated proof for his setting of Lebesgue-continuous measures,
        thus the present result can be taken as merely
        proving that result with slightly more generality and verbosity.
    \item
        The closest assumption and family of results to those here were provided
        by \citet[Section 4]{zhang_convex_consistency}; while an analog to
        \Cref{fact:Fds:cubicles} is not shown there, the proofs rely on a form of
        Lusin's Theorem, which is used in \Cref{fact:Fds:cubicles} as well;
        indeed, the proofs here owe their existence to those earlier ones by
        \citet[Section 4]{zhang_convex_consistency}.
    \item
        Another approach, suggested by
        \citet[Theorem 1 and subsequent remarks]{bv_regularized_boosting_again},
        and later used by
        \citet[Condition 1]{bartlett_traskin_adaboost}
        and
        \citet[eq. (12.11)]{schapire_freund_book_final},
        is to require the weaker condition that
        \[
            \inf \{ \cL(A\lambda) : \lambda \in \Lambda \}
            =
            \inf \left\{ \int \ell (-yf(x)) d\nu(x,y) :
            f\ \textup{measurable from $\cX$ to $\R$}\right\};
        \]
        for a verification that this property is indeed weaker, see
        \Cref{fact:Fds_weakening}.
        \citet[Lemma 1]{bv_regularized_boosting_again} show that this assumption
        is satisfied by classes whose convex hull contains indicators of all Borel sets,
        and thus \Cref{fact:Fds_weakening} can be considered a simplification which
        suffices to grant consistency with more computationally tractable classes (like
        decision lists and trees).
\end{itemize}

As discussed above, the essential property of $\Fds(\nu)$ is that it implies the weaker
condition used by
\citet[Theorem 1 and subsequent remarks]{bv_regularized_boosting_again},
which in turn is directly needed for the
classification calibration methods in the consistency proof (cf.
\Cref{fact:consistency}).
The Lipschitz condition here is not crucial,
and for instance can be removed by adjusting $\Fds(\nu)$ to
require approximants to a function to carry nearly the same uniform bound.

\begin{lemma}
    \label[lemma]{fact:Fds_weakening}
    Let distribution $\nu$ over $\cX\times \{-1,+1\}$,
    class $\cH \in \Fds(\nu^\cX)$,
    and nonnegative Lipschitz convex loss $\ell$ be given (with Lipschitz constant $\beta_2$).
    Then
    \[
        \inf \{ \cL(A\lambda) : \lambda\in\Lambda\}
        =
        \inf \left\{ \int \ell (-yf(x)) d\nu(x,y)
        : f\ \textup{measurable from $\cX$ to $\R$}\right\}.
    \]
\end{lemma}
\begin{proof}
    One direction is immediate, since $H\lambda$ defines a family of measurable functions.

    Going the other direction,
    first define, for any measurable $f$, a clamping
    \[
        [f]_r(z) := \begin{cases}
            f(z) & \textup{when $f(z) \leq r$,}
            \\
            r & \textup{otherwise}
        \end{cases}.
    \]
    For any $\epsilon >0$, based on four cases for the structure of $\ell$, a clamping
    value $r_\epsilon$ is defined as follows in order to satisfy, for any $f$ and $z$,
    $\ell([f]_{r_\epsilon}(z)) \leq \ell(f(z)) + \epsilon$.
    \begin{itemize}
        \item If $\lim_{z\to-\infty}\ell(z) = \lim_{z\to+\infty}\ell(z) < \infty$,
            then $\ell$ is a constant function, and $r_\epsilon = 0$ suffices.
        \item If $\lim_{z\to-\infty}\ell(z) = \lim_{z\to+\infty}\ell(z) = \infty$,
            then $\ell$ has compact level sets, and in particular an $r_\epsilon$
            exists so that
            \[
                \{z : \ell(z) \leq \inf_q \ell(q) + \epsilon\}
                \subseteq \{z : |z| \leq r_\epsilon\}.
            \]
            It follows that
            $\ell([f]_{r_\epsilon}(z)) \leq \ell(f(z))$.
        \item If $\lim_{z\to-\infty}\ell(z) < \infty$ and
            $\lim_{z\to+\infty}\ell(z) = \infty$,
            then set
            \[
                r_\epsilon := \inf \{ |z| : \ell(z) \leq \inf_q \ell(q) + \epsilon\}.
            \]
            Unlike the preceding two cases, clamping here can increase the value,
            but not by more than $\epsilon$.
        \item If $\lim_{z\to-\infty}\ell(z) = \infty$ and
            $\lim_{z\to+\infty}\ell(z) < \infty$, then this case is handled by the
            preceding one by considering the reflection $z \mapsto \ell(-z)$.
    \end{itemize}
    Consequently, let $\{f_i\}_{i=1}^\infty$ be a minimizing sequence for the target
    infimum above so that
    \[
        \int \ell (-yf_i(x)) d\nu(x,y)
        \leq
        2^{-i} +
        \inf \{ \int \ell (-yf(x)) d\nu(x,y) : f\ \textup{measurable from $\cX$ to $\R$}\}.
    \]
    Each $f_i$ might not be bounded,
    so define $g_i := [f_i]_{r_{\epsilon_i}}$ where $\epsilon_i := 2^{-i}$; by this choice,
    \begin{align*}
        \int \ell (-yg_i(x)) d\nu(x,y)
        &=
        \int \ell (-y[f_i]_{r_{\epsilon_i}}(x)) d\nu(x,y)
        \\
        &\leq
        2^{-i} +
        \int \ell (-yf_i(x)) d\nu(x,y)
        \\
        &\leq
        2^{-i+1} +
        \inf\left \{ \int \ell (-yf(x)) d\nu(x,y) : f\ \textup{measurable from $\cX$ to $\R$}\right\}.
    \end{align*}
    Lastly, since $\SPAN(\cH)$ is dense in the $L^1(\nu^\cX)$ metric, let $h_i\in\SPAN(\cH)$
    satisfy $\|h_i-g_i\|_1\leq 2^{-i}$; since $\ell$ is Lipschitz with constant $\beta_2$,
    then
    \begin{align*}
        \int \ell (-yh_i(x)) d\nu(x,y)
        &\leq
        \int \ell (-yh_i(x)) d\nu(x,y)
        +
        \int (\ell (-yg_i(x)) - \ell(-yh_i(x))) d\nu(x,y)
        \\
        &\leq
        \int \ell (-yh_i(x)) d\nu(x,y)
        +
        \int \beta_2 |-yg_i(x) + yh_i(x)| d\nu(x,y)
        \\
        &\leq
        \int \ell (-yh_i(x)) d\nu(x,y)
        +\beta_2 \|g_i - h_i\|_1
        \\
        &\leq
        (2+\beta_2)2^{-i} +
        \inf \left\{ \int \ell (-yf(x)) d\nu(x,y) :
        f\ \textup{measurable from $\cX$ to $\R$}\right\},
    \end{align*}
    and the result follows.
\end{proof}

To close, the proof of \Cref{fact:Fds:cubicles},
which avoids strong structural assumptions on the measure (for instance, a relationship
to Lebesgue measure) via an invocation of Lusin's Theorem.

\begin{proofof}{\Cref{fact:Fds:cubicles}} \label{proof:Fds:cubicles}
    Let $\epsilon > 0$ and
    bounded measurable $g$ with $\|g\|_u := \sup_x |g(x)| \in (0, \infty)$
    (when $\|g\|_u$, then $g\in \SPAN(\cH)$ and the proof is complete).
    By Lusin's Theorem, there exists compactly-support
    continuous $h\in L^1(\nu)$
    which satisfies  $\nu([h \neq g]) \leq \epsilon/ (4\|g\|_u)$,
    and $\|h\|_u \leq \|g\|_u$
    \citep[Theorem 7.10]{folland}.  Let $C$ denote the compact support of $h$;
    continuity over a compact subset of $\R^d$
    means uniform continuity, and therefore let $\tau >0$
    be sufficiently small that the bounding box of $C$ may be partitioned into finitely
    many cubes of side length $\tau$ (products of half-open intervals of length $\tau$)
    so that, for any $x_1$ and $x_2$ within a single cube,
    $|h(x_1) - h(x_2)| \leq \epsilon/2$.
    Now let $f$ be a sum of indicators of these cubes, where each indicator is weighted by
    $h(x)$ with $x$ being an arbitrary point in the corresponding cube.
    By construction and since
    $\SPAN(\cH)$ contains such cubes,
    $f\in \SPAN(\cH)$, and moreover $\|f-h\|_1 \leq \epsilon/2$ since $\nu$ is a probability
    measure, which provides
    \[
        \|f - g\|_1 \leq \|f-h\|_1 + \|h-g\|_1
        \leq \epsilon/2 + \int_{[h \neq g]} |h-g|d\nu
        \leq \epsilon/2 + 2\|g\|_u \nu([h\neq g])
        \leq \epsilon.
    \]
\end{proofof}

\section{Loss Function Classes $\Lbds$ and $\Ltd$}
First, note that $\Lbds$ and $\Ltd$ contain a few useful things.

\begin{lemma}
    \label[lemma]{fact:L:examples}
    $\llog \in \Lbds$ with parameters $B_2=1/4$, $\beta_1 = 1/2$, $\beta_2 = 1$.
    $\lruss \in \Lbds$ with parameters $B_2 = \beta_1 = \beta_2 = 1$.
    Lastly, $\exp \in \Ltd$ and $\llog \in \Ltd$.
\end{lemma}
\begin{proof}
    For the logistic loss $\llog$, note $0 \leq \sup_x \llog''(x) \leq 1/4$,
    thus the mean value theorem grants Lipschitz gradients
    with parameter $B_2\leq 1/4$.
    $\llog$'s Lipschitz parameters are $\beta_1 = 1/2$ and $\beta_2 = 1$.

    Since $\lruss$ is not twice differentiable, gradient slopes must be checked
    manually.  To start, note
    \[
        \lruss'(x) = \begin{cases}
            0 &\textup{when $x\leq -1$},
            \\
            x+1 &\textup{when $x\in(-1,0)$},
            \\
            1 &\textup{when $x\geq 0$},
        \end{cases}
    \]
    whereby $\beta_1=\beta_2 = 1$.
    Within each line segment, the gradient slopes are 0, 1, and 0.  By manually checking
    pairs $x<y$ in the first and second, first and third, and second and third intervals,
    the tightest Lipschitz constant on the gradients is 1.

    The containments within $\Ltd$ are direct.
\end{proof}

The next two results establish the value of Lipschitz gradients: the standard
Taylor expansion inequality used in conjunction with twice differentiability is
still valid.

\begin{lemma}
    \label[lemma]{fact:Lbds:taylor_single}
    Let $\ell \in \Lbds$ with Lipschitz gradient parameter $B_2$ be given.
    Then, for any $x,y\in \R$,
    \[
        \ell(y) \leq \ell(x) + \ell'(x)(y-x) + \frac {B_2}{2}(x-y)^2.
    \]
\end{lemma}
\begin{proof}
    Suppose $x\leq y$;
    by the mean value theorem and the definition of $B_2$,
    \begin{align*}
        \ell(y)
        &= \ell(x) + \int_{x}^y \ell'(t)dt
        \\
        &= \ell(x) + \int_{x}^y \left(\ell'(x) + \frac {\ell'(t) - \ell'(x)}{t-x}(t-x)\right)dt
        \\
        &\leq \ell(x) + \ell'(x)(y-x) + \left.B_2\left(\frac {t^2}{2} - xt\right)\right|_x^y
        \\
        &\leq \ell(x) + \ell'(x)(y-x) + \frac {B_2}{2}(x-y)^2.
    \end{align*}
    Almost identically, when $x > y$,
    \begin{align*}
        \ell(y)
        &= \ell(x) + \int_{x}^y \ell'(t)dt
        \\
        &= \ell(x) + \int_{x}^y \ell'(x)dt + \int_y^x \frac {\ell'(t) - \ell'(x)}{t-x}(x-t)dt
        \\
        &\leq \ell(x) + \ell'(x)(y-x) + \frac {B_2}{2}(x-y)^2.
    \end{align*}
\end{proof}

\begin{corollary}
    \label[corollary]{fact:Lbds:taylor}
    Let $\ell \in \Lbds$ with Lipschitz gradient parameter $B_2$ be given.
    Then, for any $x,y\in \R^m$,
    \[
        \frac 1 m\sum_i \ell(y_i)
        \leq \frac 1 m\sum_i\ell(x_i)
        + \frac 1 m \sum_i \ell'(x_i)(y_i-x_i) + \frac {B_2}{2m} \sum_i(x_i-y_i)^2.
    \]
\end{corollary}
\begin{proof}
    It suffices to apply
    \Cref{fact:Lbds:taylor_single} $m$ times.
\end{proof}

Lastly, the following convexity properties of losses will be useful.
Note that the nonnegativity of $\ell^*$ is the reason losses were chosen to be increasing
functions (much of the literature uses decreasing functions); this makes the dual space
more readily interpretable as a space of reweightings.

\begin{lemma}
    \label[lemma]{fact:loss:basic}
    Suppose $\ell: \R \to \R_+$ is convex with $\lim_{z\to-\infty} \ell(z) = 0$.
    \begin{enumerate}
        \item $\ell$ is lower semi-continuous, whereby $\ell^*$ is convex
            lower semi-continuous, and $\ell = \ell^{**}$.
        \item $\ell^*(\phi) = \infty$ for $\phi < 0$, and $\ell^*(0) = 0$.
        \item Let $\beta := \sup_{x\neq y} |\ell(x) - \ell(y)|/ |x-y|$ denote the tightest
            Lipschitz constant for $\ell$.  If $\beta < \infty$,
            then
            $\ell^*(\phi) =\infty$ when $\phi > \beta$,
            and $\ell^*(\phi) < \infty$ when $\phi \in [0,\beta]$.
        \item If $g\in \partial\ell(0)$ is any subgradient of $\ell$ at the origin
            and $\ell(0) > 0$,
            then $\ell^*(\phi) < 0$ for $\phi \in (0,g)$,
            and $\ell^*$ attains its minimum value at $g$.
    \end{enumerate}
\end{lemma}
\begin{proof}
%
    Since $\ell$ is finite everywhere, it is continuous (thus lower semi-continuous),
    and thus $\ell = \ell^{**}$ and $\ell^*$ is convex lower semi-continuous
    \citep[Theorem 12.2]{ROC}.

    For any $x\in \R$ and subgradient $g_x\in \partial \ell(x)$,
    $\ell(0) \geq \ell(x) + g(0-x)$.
    Since $\lim_{z\to-\infty}\ell(z) = 0$ and $\ell$ is convex,
    it follows that $\ell$ is nondecreasing, meaning $g \geq 0$,
    and thus, for any $\phi < 0$,
    \begin{align*}
        \ell^*(\phi)
        &
        = \sup_{x\in\R} \phi x - \ell(x)
        \geq \sup_{x\in\R} \phi x - \ell(0) - g_xx
        \geq \sup_{x < 0} (\phi - g_x) x - \ell(0)
        = \infty.
    \end{align*}

    Additionally, since $\inf_x \ell(x) = 0$,
    \[
        \phi^*(0) = \sup_x 0\cdot x - \ell(x) = - \inf_x\ell(x) = 0.
    \]

    Next, suppose $\ell$ has tightest Lipschitz parameter $\beta$,
    whereby the any subgradient $g_x$ at a point $x$ satisfies $|g_x| \leq \beta$.
    Consequently, proceeding just as in the study of the case $\phi < 0$,
    for any $\phi > \beta$,
    \begin{align*}
        \ell^*(\phi)
        \geq \sup_{x\in\R} \phi x - \ell(0) - g_xx
        \geq \sup_{x > 0} (\phi -\beta) x - \ell(0)
        =\infty.
    \end{align*}
    On the other hand, let $\beta' \in (0,\beta)$ be arbitrary,
    whereby there must exist $x > y$ with
    \[
        r <  \frac {\ell(x) - \ell(y)}{x-y}
    \]
    (where the absolute values were dropped since $x> y$ and $\ell$ is nondecreasing).
    Taking any $h\in \partial\ell(x)$, note
    \[
        r <  \frac {\ell(x) - \ell(y)}{x-y} \leq \frac {\ell(x) - (\ell(x) + h(y-x))}{x-y}
        = h.
    \]
    Consequently, by the Fenchel-young inequality,
    \[
        \ell^*(h) = hx - \ell(x) < \infty.
    \]
    Since $\ell^*$ is convex, it is finite over a convex set.
    Since $r$ was arbitrary, it follows that $\ell^*$ is finite over $[0,\beta)$.
    Since $\ell^*$ is lower semi-continuous, it must also hold that $\ell^*(\beta) < \infty$.

    For the final property, let $g\in \ell(0)$ be given;
    by the Fenchel-Young inequality and $\ell(0) > 0$,
    \[
        \ell^*(g) = 0\cdot g - \ell(0) < 0.
    \]
    Since $\ell^*(0) = 0$ and $\ell^*$ is closed and convex, the first part follows.
    For the second part, since $\ell$ is closed and convex, $g\in \partial \ell(0)$
    implies $0 \in \partial \ell^*(g)$ \citep[Theorem 23.5]{ROC}, which is precisely
    the first order optimality condition \citep[Proposition 3.1.5]{borwein_lewis}.
\end{proof}

As a final basic result about $\ell$, note that the terminology ``separable'' is
at least somewhat justified.
\begin{proposition}
    \label[proposition]{fact:why_say_separable}
    Suppose $\ell: \R \to \R_{++}$ is convex with $\lim_{z\to-\infty} \ell(z) = 0$,
    and let any $\cH$ and any probability measure $\nu$ over $\cX\times \{-1,+1\}$
    be given.
    Suppose $\inf_\lambda \int \ell(A\lambda)d\nu = 0$.
    \begin{enumerate}
        \item For any $\epsilon > 0$, there exists $\lambda_\epsilon\in \Lambda$ so
            that $\nu([A\lambda_\epsilon \leq -1]) \geq 1-\epsilon$.
        \item With probability 1 over the draw of a sample $\{(x_i,y_i)\}_{i=1}^m$
            (for any $m<\infty$),
            there exists $\lambda\in\Lambda$ so that $(A\lambda)_{x_i,y_i} \leq -1$ for every $i$.
        \item In general, there does not exist $\lambda\in\Lambda$
            so that $\nu([A\lambda \leq 0]) = 1$ (indeed,
            \Cref{ex:sep:nightmare} provides
            a counterexample).
    \end{enumerate}
\end{proposition}
\begin{proof}
    Let $\epsilon>0$ be given,
    and choose $\{\lambda_i\}_{i=1}^\infty$ so that
    $\int \ell(A\lambda)d\nu \leq 1/i$.
    Since $\ell(A\lambda_i) \to 0$ $\nu$-a.e.,
    by Egoroff's theorem there exists $S$ with $\nu(S)\geq 1-\epsilon$
    so that $\ell(A\lambda_i) \to 0$ uniformly on $S$
    \citep[Theorem 2.33]{folland}.  But since $\ell>0$ everywhere
    and $\lim_{z\to-\infty}\ell(z) = 0$ and $\ell$ is convex,
    it must be the case that $(A\lambda_i) \to -\infty$ uniformly on $S$,
    and so there exists $i$ with $A\lambda_i \leq -1$ on $S$, which gives the first
    result.

    For the second result, take any $\epsilon>0$, and choose $\lambda_\epsilon$ as
    granted by the first part.
    Let $\widehat \nu$ denote the empirical measure over the provided sample;
    then
    \[
        \Pr\left[\forall i\centerdot (A\lambda_\epsilon)_{x_i,y_i} \leq -1 \right]
            =
            \nu\left([A\lambda_\epsilon \leq -1]\right)^m \geq (1-\epsilon)^m.
    \]
    Since $\epsilon >0$ was arbitrary, the second result follows.

    For the third result, recall that \Cref{ex:sep:nightmare} (whose properties
    are provided in \Cref{fact:gamma:basic:bad}) gave an instance
    where every element of $\SPAN(\cH)$ makes some mistakes.
\end{proof}








\section{Duality Properties of $\gamma_\epsilon$}

In order to develop $\gamma_\epsilon(\nu)$, the set $\cD_\epsilon(\nu)$ must first
be studied.

\begin{proposition}[Basic properties of $\cD_\epsilon(\nu)$]
    \label[proposition]{fact:cDeps_basic}
    Let $\nu$ be an arbitrary probability measure over $\cX\times \{-1,+1\}$,
    and let $\epsilon \in [0,1]$ be arbitrary.
    The set $D_\epsilon(\nu)$ has the following properties.
    \begin{enumerate}
        \item
            $D_\epsilon(\nu)$ is convex.
        \item
            $D_\epsilon(\nu)$ is
            closed in the $L^1(\nu)$ topology.
        \item If $\epsilon > 0$,
            then $D_\epsilon(\nu)$
            is closed in the $L^\infty(\nu)$ topology,
            and also closed in the weak* topology
            (i.e., the weak topology induced upon $L^\infty(\nu)$ by
            $L^1(\nu)$).
        \item $D_\epsilon(\nu)$ is compact in the weak* topology on $L^\infty(\nu)$
            (as discussed in the preceding point).
        \item
            $D_\epsilon(\nu)$ is not guaranteed to be compact in the $L^1(\nu)$ or
            $L^\infty(\nu)$ topologies; indeed, it is not compact
            when $\epsilon = 1/2$,
            $\cX = [0,1]$,
            the marginal distribution $\nu^\cX$ is uniform on $[0,1]$, and
            the conditional distribution $\Pr(Y=1 | X=x)$ is arbitrary.
    \end{enumerate}
\end{proposition}
\begin{proof}
    \begin{enumerate}
        \item
            For convexity,
            let any $\alpha \in (0,1)$ and $p_1,p_2\in\cD_\epsilon(\nu)$ be given,
            and define sets $N_j := p_j^{-1}((-\infty,0))$ for $j\in \{1,2\}$,
            where necessarily $\nu(N_j) = 0$.  The goal is to show
            $p := \alpha p_1 + (1-\alpha)p_2 \in \cD_\epsilon(\nu)$.

            Define $N := N_1 \cup N_2$ (where again $\nu(N) = 0)$.
            First, for any $(x,y) \in N^c$,
            \[
                p(x,y) = \alpha p_1(x,y) + (1-\alpha)p_2(x,y) \geq \alpha\cdot 0 + (1-\alpha)\cdot 0
                = 0,
            \]
            whereby it follows that $p \geq 0$ $\nu$-a.e..
            Second,
            \[
                \|p\|_\infty \leq \alpha \|p_1\|_\infty + (1-\alpha)\|p_2\|_\infty \leq 1/\epsilon,
            \]
            again using the convention $1/\infty = 0$, whereby $\|p\|_\infty \leq 1/\epsilon$ as desired.
            Lastly,
            \[
                \|p\|_1
                = \int |\alpha p_1 + (1-\alpha)p_2|
                = \int_{N^c} (\alpha p_1 + (1-\alpha)p_2)
                = \alpha \|p_1\|_1 + (1-\alpha)\|p_2\|_1
                = 1,
            \]
            meaning all conditions are met, and $p\in D_\epsilon(\nu)$.
            Since $\alpha, p_1$, and $p_2$ were arbitrary, it follows that $D_\epsilon(\nu)$
            is convex.
        \item
            For closure within $L^1(\nu)$,
            since $L^1(\nu)$ is a metric space, it is first countable, and thus
            it suffices to check that any sequence $\{p_j\}_{j=1}^\infty$
            with $p_j\in \cD_\epsilon(\nu)$ and $p_j\to p\in L^1(\nu)$
            satisfies $p\in \cD_\epsilon(\nu)$ \citep[Proposition 4.6]{folland}.
            Given any such sequence $\{p_j\}_{j=1}^\infty$,
            choose a subsequence $\{q_i\}_{i=1}^\infty$ so
            that $q_i\to p$ $\nu$-a.e. \citep[Corollary 2.32]{folland}.

            Let $N_p$ be the (null) set of points for which convergence fails,
            and additionally, for each $i$, define $N_i := q_i^{-1}((-\infty,0))$;
            lastly, set $N := N_p \cup (\cup_i N_i)$, where again $\nu(N) = 0$.
            Thus for any $(x,y) \in N^c$,
            \begin{align*}
                p(x,y)
                &= q_i(x,y) + (p(x,y) - q_i(x,y))
                \\
                &\geq q_i(x,y) - |p(x,y) - q_i(x,y)|
                \\
                &\geq \liminf_{i\to\infty} q_i(x,y) - |p(x,y) - q_i(x,y)|
                \\
                &\geq 0,
            \end{align*}
            thus $p\geq 0$ $\nu-a.e.$.
            Additionally,
            \[
                \|p\|_1
                = \int |p|
                = \int_N p
                = \int_N p_i + \int_N (p_i - p)
                = \|p_i\|_1 + \int (p_i - p),
            \]
            whereby
            \[
                \big|
                    \|p\|_1 - \|p_i\|_1
                \big|
                = \left| \int (p_i - p)\right|
                \leq \| p_i - p\|_1
                \to 0,
            \]
            and $\|p\|_1 = 1$ as desired.

            For the last property, if $\epsilon = 0$, there is nothing to show,
            thus suppose $\epsilon \in (0,1]$, set $P_i := q_i^{-1}((1/\epsilon,\infty])$,
            and $Z:=N_p \cup (\cup_i P_i)$, whereby it follows that
            \[
                \nu(Z) = 0,
                \qquad
                q_i \to p \textup{ over } Z^c,
                \qquad
                q_i \leq 1/\epsilon \textup{ over } Z^c.
            \]
            Then, for any $(x,y)\in Z^c$,
            \[
                p(x,y)
                = q_i(x,y) + (p(x,y) - q_i(x,y))
                \leq \limsup_{i\to\infty}  q_i(x,y) + |p(x,y) - q_i(x,y)|
                \leq 1/\epsilon,
            \]
            which establishes $\|p\|_\infty \leq 1/\epsilon$, and
            thus $p \in \cD_\epsilon(\nu)$.
        \item
            Note firstly that if $\epsilon = 0$, then $\cD_\epsilon(\nu)$ can contain
            members which are not elements of $L^\infty(\nu)$, and thus discussing this
            set in the $L^\infty(\nu)$ topology does not make sense.  For the remainder
            of this case, suppose $\epsilon>0$.

            Just as in the case of $L^1(\nu)$, for $L^\infty(\nu)$ it suffices to
            let a sequence $\{p_i\}_{i=1}^\infty\subseteq \cD_\epsilon(\nu)$
            be given with $p_i\to p$ in the $L^\infty(\nu)$ topology,
            and to show that $p\in \cD_\epsilon(\nu)$.
            Notice however, since $\nu$ is a probability measure, that
            \[
                \|p_i - p\|_1
                = \int |p_i - p_1|
                \leq \int \|p_i - p\|_\infty
                = \|p_i - p\|_\infty,
            \]
            meaning $p_i \to p$ in $L^1(\nu)$ as well, which by the preceding case
            provides that $p\in \cD_\epsilon(\nu)$ as desired.

            Lastly, since $\cD_\epsilon(\nu)$ is convex and
            additionally closed according to $L^\infty(\nu)$,
            then it is also weak* closed \citep[Theorem 3.12]{rudin_functional}.
        \item
            Again suppose $\epsilon > 0$, and define
            \[
                B_\epsilon := \{ p \in L^\infty(\nu) : \|p\|_\infty \leq 1/\epsilon\}.
            \]
            By Alaoglu's Theorem \citep[Theorem 5.18]{folland},
            $B_1$ is compact in the weak* topology,
            thus $B_\epsilon = \epsilon^{-1} B_1$ is weak*-compact as well.
            The result follows since $\cD_\epsilon(\nu)$ is a weak*-closed subset
            of $B_\epsilon$, and closed subsets of compact sets are compact
            \citep[Theorem 4.22]{folland}.
        \item
            Noncompactness can be understood from the fact that norm balls are in general
            not compact, but an explicit construction is provided for completeness.
            Since both $L^1(\nu)$ and $L^\infty(\nu)$
            are metric spaces, to prove non-compactness,
            it suffices to prove $\cD_{1/2}(\nu)$ is not totally bounded.
            In particular, a countably infinite subset of $\cC\subset \cD_{1/2}(\nu)$ will
            be constructed satisfying the property
            $(f,g) \in \cC\times \cC$ with $f\neq g$ implies
            $\|f-g\|_1 = 1/2$ and $\|f-g\|_\infty = 2$, which suffices to
            show that $\cC$ (and thus $\cD_{1/2}(\nu)$) is not totally bounded
            (in either metric) for the following reason.
            Let $S$ be any finite subset of $L^1(\nu)$ or $L^\infty(\nu)$.
            Since $\cC$ and $S$ have respectively infinite and finite cardinalities,
            there must exist $h\in S$ which is a closest element in $S$
            to two distinct functions
            $f\neq g$ in $\cC$.  Let $\|\cdot\|$ denote either norm under consideration,
            and note that
            \[
                1/2 \leq \|f-g\| \leq \|h-f\| + \|h-g\| \leq 2\max\{\|h-f\|, \|h-g\|\},
            \]
            which means that one of these two distances is at least $1/4$.  Since
            $S$ was an arbitrary finite set, it follows that there is no finite
            set of balls of radius $1/8$ which covers $\cC$, and thus
            $\cC$ and $\cD_{1/2}(\nu)$ are not totally bounded according to either norm.

            The construction is as follows.  For every positive integer $i\in \Z_{++}$,
            define the function 
            \begin{align*}
                f_i(x,y)
                &:= 2\sum_{j=0}^{2^i-1}
                \1\left[x \in [(2j)2^{-i-1}, (2j+1)2^{-i-1})\right].
            \end{align*}
            Define $\cC := \{f_i : i \in \Z_{++}\}$.
            By construction, $\cC\subset \cD_{1/2}(\nu)$
            (i.e., $\|f_i\|_1 = 1$ and $\|f_i\|_\infty = 2$),
            and moreover $i\neq j$ implies $f_i$ and $f_j$ disagree on exactly half
            of their support, which yields
            $\|f_i-f_j\|_1 = 1/2$.
            and
            $\|f_i-f_j\|_\infty = \|f_i\|_\infty = 2$.
    \end{enumerate}
\end{proof}

With the structure of $\cD_\epsilon(\nu)$ established, the basic duality structure of
$\gamma_\epsilon(\nu)$ follows.  Note that the value of establishing the
weak*-compactness of $\cD_\epsilon(\nu)$ is to grant an application of Sion's minimax
Theorem without making any topological assumptions on $\cH$ (or rather, on the subspace
$H\Lambda$).  Additionally,
\Cref{fact:gamma_epsilon:duality:body} in \Cref{sec:separable} is a combination of this
result and part of \Cref{fact:AT}.

\begin{lemma}
    \label[lemma]{fact:gamma_epsilon:duality}
    Let probability measure $\nu$ over $\cX\times \{-1,+1\}$,
    any $\cH$,
    and any $\epsilon \in [0,1]$ be given.
    Then
    \begin{align*}
        \gamma_\epsilon(\nu)
        &=
        \min_{p\in\cD_\epsilon(\nu)}
        \sup\left\{
            \int (A\lambda) p d\nu
            :
            \lambda \in \Lambda
        \right\}
        \\
        &=
        \sup\left\{
            \min_{p\in\cD_\epsilon(\nu)}
            \int  (A\lambda) p d\nu
            :
            \lambda \in \Lambda
        \right\}
        \\
        &=
        \min_{p\in\cD_\epsilon(\nu)}
        \|A^\top p\|_\infty,
    \end{align*}
    where $\|A^\top p\|_\infty$ is discussed in \Cref{fact:AT}.
\end{lemma}
\begin{proofof}{\Cref{fact:gamma_epsilon:duality}}
    Before applying the duality result, it must be established that the
    various infima are attained.
    To start, consider the final expression
    $\inf_{p\in \cD_{\epsilon}(\nu)} \|A^\top\|_\infty$,
    and
    let $\{p_i\}_{i=1}^\infty$ with $p_i \in \cD_\epsilon(\nu)$ be a minimizing sequence
    to the infimum.
    Since \Cref{fact:cDeps_basic} establishes
    that $\cD_\epsilon(\nu)$ is weak*-compact, there is a subsequence $\{q_i\}_{i=1}^\infty$
    which weak*-converges to some $q\in \cD_\epsilon$
    \citep[Theorem 4.29]{folland}.
    But \Cref{fact:AT} established that $p\mapsto \|A^\top p\|_\infty$ is weak* lower
    semi-continuous,
    and since it is finite over $L^\infty(\nu)$, it is therefore weak* continuous,
    and therefore the limit point $q$ attains the infimum.
    Furthermore, \Cref{fact:AT} provides that $\|A^\top p\|_\infty$ is the same
    as the first infimand, whereby both expressions attain their minimizers and
    are equal.

    The middle expression is the easiest; once again constructing a weak*-convergent
    sequence $q_i \to q$ with $q_i\in\cD_\epsilon(\nu)$,
    the definition of weak*-convergence explicitly grants $\int q_if \to \int qf$ for every
    $f\in L^1(\nu)$, and since $A\lambda\in L^1(\nu)$
    is held fixed within this inner expression,
    it follows that $q$ attains the infimum.

    What remains is to swap minimization and maximization.
    This in turn follows by Sion's minimax theorem
    \citep{komiya_sion};
    to verify this application, note that $(p,\lambda) \mapsto \int p\lambda$ is linear
    and continuous in both parameters (indeed, this is by construction,
    since $L^\infty(\nu)$ is isometrically isomorphic to the topological dual to
    $L^1(\nu)$, and the weak* topology over $L^\infty(\nu)$ ensures that this integral
    relation is continuous for every $\lambda \in \Lambda$),
    also that $\Lambda$ is a topological vector space,
    and lastly that $\cD_\epsilon(\nu)$ is a convex compact subset of a topological vector
    space (namely, the weak* topology, and not the $L^\infty(\nu)$ topology, where
    $\cD_\epsilon(\nu)$ is not necessarily compact as per \Cref{fact:cDeps_basic}).
\end{proofof}

\section{Duality Properties of $\cL$}

Throughout this section, the identification of $L^1(\nu)^*$ with $L^\infty(\nu)$
and $L^1(\rho)^*$ with $L^\infty(\rho)$ via isometric isomorphism as provided
by \Cref{fact:banach_duality} will be central to obtaining meaningful expressions for
the various conjugates.

To start, note the convexity structure of $\int \ell$.

\begin{lemma}
    \label[lemma]{fact:cL:conjugacy}
    Let $\ell:\R\to\R_+$ be convex with $\lim_{z\to-\infty} \ell(z) = 0$
    and finite tightest Lipschitz constant
    $\beta := \sup_{x\neq y} |\ell(x) - \ell(y)|/ |x-y| < \infty$,
    and let $\nu$ be a probability measure over $\cZ := \cX\times \{-1,+1\}$.
    \begin{enumerate}
        \item If $f\in L^1(\nu)$,
            then $\int \ell(f(z)) d\nu(z)$ is well-defined and finite.
        \item
            $\int \ell$ is convex lower semi-continuous over $L^1(\nu)$.
        \item
            Its conjugate $(\int \ell)^*$ is also convex lower semi-continuous
            as a function over $L^\infty(\nu)$.
        \item
            If $p\in L^\infty(\nu)$,
            then $(\int \ell)^*(p) = \int \ell^*(p)$,
            which is finite
            $(\int \ell)^*(p)$ iff $p \in [0,\beta]$ $\nu$-a.e..
    \end{enumerate}
\end{lemma}
\begin{proof}
    Let $f\in L^1(\nu)$ be arbitrary.
    Since $\ell$ is convex and finite, it is continuous, so $\ell\circ f$ is measurable,
    and moreover it is nonnegative thus $\int \ell(f)$ is well-defined.
    Additionally,
    \begin{align*}
        \int \ell(f(z)) d\nu(z)
        &= \int \ell(0) d\nu(z) + \int (\ell(f(z)) - \ell(0)) d\nu(z)
        \\
        &\leq \ell(0)\nu(\cZ)  + \int \beta |f(z) - 0| d\nu(z)
        = \ell(0)\nu(\cZ) + \beta\|f\|_1 < \infty.
    \end{align*}

    Next, for any $f_1, f_2\in L^1(\nu)$ and $\alpha \in [0,1]$,
    \begin{align*}
        \int \ell(\alpha f_1(z) + (1-\alpha)f_2(z)) d\nu(z)
        &\leq
        \int \left(\alpha \ell(f_1(z)) + (1-\alpha)\ell(f_2(z))\right) d\nu(z)
        \\
        &\leq
    \alpha \int \ell(f_1(z))d\nu(z)  + (1-\alpha)\int \ell(f_2(z)) d\nu(z),
    \end{align*}
    whereby $\int \ell$ is convex.  Since it is finite over $L^1(\nu)$ (as above),
    it is necessarily lower semi-continuous.

    Since $\int \ell$ is convex lower semi-continuous, so is its
    conjugate \citep[Theorem 2.3.3]{zalinescu}, where the dual space
    $L^1(\nu)^*$ is identified with $L^\infty(\nu)$ as per the isomorphism statements
    in \Cref{fact:banach_duality}.

    The remainder of this proof will reason about the conjugate to $\int \ell$.
    First let $p\in L^\infty(\nu)$ be given with $\nu(p^{-1}([0,\beta]^c)) > 0$;
    it will follow that $(\int \ell)^*(p) = \infty$.
    Define the sets
    \[
        S_- := p^{-1}((-\infty,0))
        \qquad
        S_0 := p^{-1}([0,\beta]),
        \qquad
        S_+ p^{-1}((\beta,\infty)),
    \]
    as well as, for every $c\in \R$, the reals
    \[
        g_- \in \partial \ell(-c),
        \qquad
        g_0 \in \partial \ell(0),
        \qquad
        g_+ \in \partial \ell(+c),
    \]
    and lastly the simple functions
    \begin{align*}
        f_c(z) &:= -c \1[z\in S_-] + 0\1[z\in S_0] + c\1[z\in S_+],
        \\
        g_c(z) &:= g_- \1[z\in S_-] + g_0\1[z\in S_0] + g_+\1[z\in S_+].
    \end{align*}
    By these choices, $f_c$ and $g_c$ are measurable and within $L^1(\nu)$,
    and moreover $g_c \in \partial \ell(f_c)$ everywhere.
    As such,
    \begin{align*}
        \left(\int \ell\right)^*(p)
        &
        = \sup\left\{ \int (fp - \ell(f))d\nu : f\in L^1(\nu) \right\}
        \\
        &
        \geq
        \sup\left\{ \int (f_{c}p - \ell(f_{c}))d\nu : c\in \R \right\}
        \\
        &
        \geq
        \sup\left\{ \int (f_{c}p - \ell(0)
            + g_{c}(0 - f_c))
        d\nu(z) : c\in \R \right\}
        \\
        &
        \geq
        \sup\left\{
            c\int_{S_-} (g_c - p)d\nu
            + c \int_{S_+}(p -g_c) d\nu
        : c\in \R \right\}
        -\ell(0)\nu(\cZ)
        \\
        &
        = \infty,
    \end{align*}
    the last step following since $g_c\in [0,\beta]$ everywhere and
    $\nu(S_-\cup S_+)> 0$.  As such, $(\int \ell)^*(p) = 0$,
    and since $\int \ell^*(p) = \infty$ by
    properties of $\ell^*$ (cf. \Cref{fact:loss:basic}),
    it follows that $\int\ell^*(p) = (\int \ell)^*(p) = \infty$.

    In the remainder of the proof, suppose $p \in [0,\beta]$ $\nu$-a.e..

    Now consider the case that $p$ is a simple function with $p \in (0,\beta)$ everywhere.
    Since $\ell^*$ is finite over $[0,\beta]$ (cf. \Cref{fact:loss:basic}),
    $p$ is within the relative interior of the domain of $\ell^*$ everywhere,
    and thus $\partial \ell(p(z))$ is a nonempty set for every $z\in \cZ$
    \citep[Theorem 23.4]{ROC}.
    Consequently, construct $q$ so that $q(z) \in \partial\ell^*(p(z))$ everywhere,
    and moreover $q$ is also a simple function (i.e., pick the same subgradient along
    each of the finitely many regions composing $p$); these choices will ensure that
    there are no measurability issues  with $q$ (otherwise, the arguments pass through for
    arbitrary $p\in (0,\beta)$); additionally, $q\in L^1(\nu)$ since $\nu$ is a finite
    measure.
    Since $\ell$ is lower semi-continuous, $q(z) \in \partial \ell^*(p(z))$
    implies $p(z) \in \partial \ell^{**}(q(z)) = \partial \ell(q(z))$,
    and the Fenchel-Young inequality implies
    \[
        \ell^*(p(z)) = p(z) q(z) - \ell^{**}(q(z)) = p(z) q(z) - \ell(q(z)).
    \]
    As such,
    \begin{align*}
        \left(\int \ell\right)^*(p)
        &
        = \sup\left\{ \int (fp - \ell(f))d\nu : f\in L^1(\nu) \right\}
        \\
        &
        \geq
        \int (qp - \ell(q))d\nu
        \\
        &
        =
        \int \ell^*(p)d\nu
    \end{align*}
    Now using the fact that $p(z) \in \partial \ell(q(z))$,
    \begin{align*}
        \left(\int \ell\right)^*(p)
        &
        = \sup\left\{ \int (fp - \ell(f))d\nu : f\in L^1(\nu) \right\}
        \\
        &
        \leq \sup\left\{ \int (fp - \ell(q) - p(f-q))d\nu : f\in L^1(\nu) \right\}
        \\
        &
        = \sup\left\{ \int (pq - \ell(q))d\nu : f\in L^1(\nu) \right\}
        \\
        &
        = \int \ell^*(p);
    \end{align*}
    combining these two inequalities, $(\int \ell)^*(p) = \int \ell^*(p)$.

    Now consider the case that $p\in (0,\beta)$ is just measurable.
    Since the simple functions are dense in $L^\infty(\nu)$ \citep[Theorem 6.8]{folland},
    there exists a simple function $\phi_i\in L^\infty(\nu)$ with $\|p-\phi_i\|_\infty
    \leq 1/i$, and moreover $\phi_i$ may be clamped to the range $[1/i, \beta-1/i]$
    (with $i$ sufficiently large to make this interval nonempty),
    whereby this clamped simple function $\psi_i$ satisfies $\|p - \psi_i\|_\infty \leq
    2/i$.
    Since $\left(\int \ell\right)^*$ is lower semi-continuous,
    \[
        \left(\int \ell\right)^*(p)
        = \lim_i \left(\int \ell\right)^*(\psi_i)
        = \lim_i \int \ell^*(\psi_i)
        = \int \ell^*(p),
    \]
    where the last step used the dominated convergence theorem applied with
    dominating constant map $z \mapsto \sup_{q\in [0,\beta]} |\ell^*(q)|$, which is finite
    since $\ell^*$ is continuous over the compact set $[0,\beta]$
    (cf. \Cref{fact:loss:basic}).

    Next consider the case that measurable $p\in(0,\beta)$ $\nu$-a.e.;
    then
    $\tilde p(z) := p(z) \1[p(z) \in (0,\beta)] + (\beta/2) \1[p(z) \not \in (0,\beta)]$
    satisfies $(\int \ell)^*(p) = (\int \ell)^*(\tilde p)$ by definition of the
    conjugate (the integrals ignore measure zero sets),
    whereby $(\int \ell)^*(p) = (\int \ell)^*(\tilde p)
    = \int \ell^*(\tilde p) = \int \ell^*(p)$.

    Lastly, suppose measurable $p\in [0,\beta]$ $\nu$-a.e..
    For each $i$, define $p_i = (1-1/i) p + \beta / (2i)$.
    Then $p_i \in (0,\beta)$ $\nu$-a.e., and $\|p_i - p\|_\infty \to 0$,
    whereby the lower semi-continuity of $(\int \ell)^*$ and dominated convergence
    theorem cover this case in the same way as the move away from simple functions.

    Note lastly that these last choices provide a finite integral,
    since $\sup_{z\in [0,\beta]}|\ell^*(z)|<\infty$ as above, and $\nu$ is a
    finite measure.
\end{proof}

While the above proof (properties of $\int \ell$) may have seemed like a technical exercise,
note that these structural properties
can not be taken for granted; in particular, the following result establishes
that the $L^1(\nu)$ topology is not the correct way to study the exponential loss.
\begin{proposition}
    Let $\nu$ denote the standard Gaussian measure over $\R$,
    and define $f(x) := x^2$ and $f_i(x) := x^2 \1[|x| \leq i]$.
    Then $f_i\in L^1(\nu)$,
    $f \in L^1(\nu)$,
    and $\|f_i -f\|_1 \to 0$,
    but
    \[
        \int \exp(f_i(x)) d\nu(x) < \infty
        \qquad\textup{and}\qquad
        \int \exp(f(x)) d\nu(x) = \infty.
    \]
    In particular, $\int \exp$ is not lower semi-continuous over $L^1(\nu)$.
\end{proposition}
\begin{proof}
    To start, $\int f d\nu = 1$ (variance of a standard Gaussian),
    and thus $\int f_i \to f$ by the monotone convergence theorem
    (and so $\| f_i - f \|_1 \to 0$).
    But
    \begin{align*}
        \int \exp(f_i(x)) d\nu(x)
        &\leq e^{i^2} \int d\nu(x) < \infty,
        \\
        \int \exp(f(x)) d\nu(x)
        &= \frac {1}{\sqrt{2\pi}}\int e ^{x^2/2} dx
        = \infty.
    \end{align*}
    It follows that there are convergent sequences within $L^1(\nu)$ for which
    the values of $\int \exp$ do not converge, and consequently $\int \exp$ is not
    lower semi-continuous over $L^1(\nu)$.
\end{proof}

Returning to Lipschitz losses, the desired duality relation follows.

\begin{lemma}
    \label[lemma]{fact:duality}
    Let $\ell:\R\to\R_+$ be convex with $\lim_{z\to-\infty} \ell(z) = 0$
    and finite tightest Lipschitz constant
    $\beta := \sup_{x\neq y} |\ell(x) - \ell(y)|/ |x-y| < \infty$.
    Additionally, let $\nu$ be a probability measure over $\cX\times \{-1,+1\}$,
    and $\cH$ be arbitrary.
    Then
    \[
        \inf \left\{ \int \ell(A\lambda) d\nu : \lambda \in \Lambda\right\}
        =
        \max\left\{ - \int \ell^*(p)
            : p \in L^\infty(\nu), p\in [0,\beta]\ \nu\textup{-a.e.},
            \|A^\top p\|_\infty = 0
        \right\}.
    \]
\end{lemma}
\begin{proof}
    \newcommand{\Vpp}{V_{\textup{p}}}
    \newcommand{\Vdd}{V_{\textup{d}}}
    Consider the following two Fenchel problems:
    \begin{align*}
        \Vpp &:= \inf \left\{
            \int \ell(A\lambda)d\nu + \int 0\cdot\lambda d\rho : \lambda \in \Lambda
        \right\},
        \\
        \Vdd &:= \sup\left\{
            -\int \ell^*(p) - \iota_{\{0\}}(A^\top p)
            : p\in L^\infty(\nu)
        \right\},
    \end{align*}
    where $\iota_ {\{0\}}$ is the indicator for the set $\{0\}$,
    \[
        \iota_{\{ 0\} }(\lambda) =
        \begin{cases}
            0 &\textup{when } \lambda=0, \\
       \infty &\textup{otherwise,}
        \end{cases}
    \]
    and is the conjugate to $\int 0 \cdot \lambda$.  In order to show $\Vpp = \Vdd$
    and attainment occurs in the dual,
    an appropriate Fenchel duality rule will be applied
    \citep[Corollary 2.8.5 using condition (vii)]{zalinescu},
    which requires the verification of the following properties.
    \begin{itemize}
        \item First note that $\int\ell$ and $\int\ell^*$ are both convex lower
            semi-continuous, and moreover mutually conjugate
            (cf. \Cref{fact:cL:conjugacy}).
            The function $\lambda \mapsto 0 = \int 0\lambda$
            is immediately convex lower semi-continuous (over $\Lambda$),
            and thus its conjugate
            $\iota_{\{0\}}$ is similarly convex lower semi-continuous,
            and the two are mutually conjugate \citep[Theorem 2.3.3]{zalinescu}.
        \item
            Both $L^1(\nu)$ and $\Lambda = L^1(\rho)$ are Banach and therefore
            Fr\'echet spaces. (The present proof is one of the reasons
            $\Lambda$ was taken to be a Banach space
            and not merely, say, weightings with finite support as
            used by the algorithm).
        \item
            Let $\dom(f) = \{x : f(x) <\infty\}$ denote the effective domain of a convex
            function, meaning those values where it is finite.
            As provided by \Cref{fact:cL:conjugacy}, $\dom(\int \ell) = L^1(\nu)$,
            and thus, since $A: \Lambda \to L^1(\nu)$,
            \[
                A(\dom(\lambda \mapsto 0)) - \dom\left(\int \ell\right)
                = A\Lambda + L^1(\nu)
                = L^1(\nu),
            \]
            which settles the constraint qualification.  (Recall that $A\Lambda$
            is not necessarily a closed subspace (cf. \Cref{fact:H:notclosed});
            thus further problems would occur
            here if this proof were attempted for $\ell = \exp$, as $\dom(\int \ell)$
            would not swallow the closure issues of $A\Lambda$.)
    \end{itemize}
    This completes the conditions necessary for the Fenchel duality result.
    To adjust the proof into the desired form,
    \Cref{fact:cL:conjugacy} provided that $\int\ell^*$ is finite iff its input
    lies within $[0,\beta]$ $\nu$-a.e. (thus other values may safely be discarded from
    the optimization problem, which always has feasible point $0\in L^\infty(\nu)$),
    and secondly
    $\iota_{0}(A^\top p) < \infty$ iff $\|A^\top p\|_\infty = 0$ (recall the
    form of $\|A^\top p\|_\infty$ in \Cref{fact:AT}).
\end{proof}

\section{Line Search Guarantees}
\label{subsec:linesearch}

Before proceeding with the various properties of the line searches, it is a good time
to discuss expressions involving $\nhcL$, upon which these line searches depend.
In the context of the algorithm, the sample size is
finite and $|\supp(\lambda)|<\infty$, thus
\[
    \hcL(A\lambda)
    = \frac 1 m \sum_{i=1}^m
    \left(
        \sum_{h\in\supp(h)} -y_i h(x_i) \lambda(h)
    \right)
\]
always involves only finitely many computations.  In this way, $A$ may be
simply viewed as a matrix with $m$ rows and at most
$\sup_{t \leq \lceil m^a\rceil} |\supp(\lambda_t)| \leq \lceil m^a \rceil$ columns;
furthermore, if $\cH$ is binary, $2^m$ columns suffice and are known a priori (and the
Sauer-Shelah Lemma can further reduce the dimensions).  As such,
when working with gradient computations, this manuscripts adopts the familiar notation
of the form
\begin{align*}
    \nhcL(A\lambda)^\top A\lambda'
    &= \ip{A^\top \nf(A\lambda)}{\lambda'}
    \\
    &= \ip{\nf(A\lambda)}{A\lambda'}
    \\
    &= \frac 1 m \sum_{i=1}^m
    \ell'((A\lambda)_{x_i,y_i}) (A\lambda')_{x_i,y_i}, 
\end{align*}
and moreover the matrix rule $\nabla (\hcL \circ A) (\lambda) = A^\top \nhcL(A\lambda)$
makes sense.

This manuscript never considers gradients of $\cL$
(e.g., in the sense of G\^ateux or Fr\'echet).
However, to connect the above expressions to the development of the spaces
(e.g., $L^1(\mu)$ and $\Lambda$) and linear operators (e.g., $A$ and $A^\top$)
from \Cref{sec:linops}, note firstly that $\partial \cL$ is a subset of
$L^1(\mu)\to \R$ (identified with $L^\infty(\mu)$ via \Cref{fact:banach_duality}),
meaning never a singleton since it contains $\mu$-a.e. equivalent copies of functions.
Modulo these details, $(A^\top g)$, for some $g\in \partial\cL(A\lambda)$, can be
identified with an element of $L^\infty(\rho)$ as in \Cref{fact:AT},
and thus $(A^\top g)(\lambda')$ makes sense (and indeed, by properties of the adjoint
$A^\top$ and the dual space identification from \Cref{fact:banach_duality},
$(A^\top g)(\lambda') = \int (A\lambda') g d\mu$).
Of course, these expressions are nonsense from a computational standpoint.

The remainder of this section gives basic guarantees for various line searches.

\begin{remark}[Wolfe line search]
\label[remark]{rem:wolfe}
The Wolfe line search
chooses any $\alpha_t$ which satisfies the following conditions (where this manuscript
makes the simple choice $c_1 = 1/3$ and $c_2 = 1/2$):
\begin{align}
    \hcL(A(\lambda_{t-1} + \alpha v_t))
    &\quad\leq\quad \hcL(A\lambda_{t-1})
    +\alpha c_1 \ip{A^\top\nhcL(A\lambda_{t-1})}{v_t}
    \notag\\
    &\quad\leq\quad \hcL(A\lambda_{t-1})
    -\frac {\alpha\rho}{3} \|A^\top \nabla \hcL(A\lambda_{t-1})\|_\infty,
    \label{eq:wolfe:1}
    \\
    \nabla \hcL(A(\lambda_{t-1} + \alpha v_t))^\top A v_{t}
    &\quad\geq\quad
    c_2 \ip{A^\top\nhcL(A\lambda_{t-1})}{v_t}
    \notag\\
    &\quad\geq\quad
    -\frac 1 2 \|A^\top \nabla\hcL(A\lambda_{t-1}) \|_\infty.
    \label{eq:wolfe:2}
\end{align}
The method itself may be implemented (in the convex case) similarly to
binary search~\citep[Section D.1]{primal_dual_boosting_arxiv}.
\end{remark}

\begin{lemma}
    \label[lemma]{fact:singlestep:quadub}
    Let $\ell\in\Lbds$ with Lipschitz gradient parameter $B_2$,
    and iteration $t$ be given,
    and suppose $\alpha_t$ is chosen according to one of the first two step choices
    in \Cref{alg:alg:alg}, meaning either $\alpha_t=\bar\alpha_t$
    or $\alpha_t \in \left[
        -\nhcL(A\lambda_{t-1})^\top A v_t/B_2, \bar \alpha_t\right)$.
    Then
    \begin{align*}
        \alpha_t &\geq  \frac {\rho\|A^\top\nhcL(A\lambda_{t-1})\|_\infty}{B_2},
        \\
        \hcL(A\lambda_t) &\leq \hcL(A\lambda_{t-1})
        - \frac {\rho^2\|A^\top\nhcL(A\lambda_{t-1})\|_\infty^2}{2B_2}.
    \end{align*}
\end{lemma}
\begin{proof}
    By \Cref{fact:Lbds:taylor}, for every $\alpha > 0$,
    since $A$ has entries within $[-1,+1]$,
    \begin{align*}
        \hcL(A(\lambda_{t-1} + \alpha v_t))
        &\leq
        \hcL(A\lambda_{t-1}) + \alpha \ip{A^\top\nhcL(A\lambda_{t-1})}{v_t}
        + \frac {B_2}{2m}\sum_i(\alpha A v_t)^2
        \\
        &\leq
        \hcL(A\lambda_{t-1}) + \alpha \ip{A^\top\nhcL(A\lambda_{t-1})}{v_t}
        + \frac {B_2}{2m}\sum_i(\alpha A v_t)_i^2
        \\
        &\leq
        \hcL(A\lambda_{t-1}) + \alpha \ip{A^\top\nhcL(A\lambda_{t-1})}{v_t}
        + \frac {B_2\alpha^2}{2}.
    \end{align*}
    This final expression defines a univariate quadratic
    with minimum $\bar \alpha := - \ip{A^\top\nhcL(A\lambda_{t-1})}{v_t}/B_2$.
    This function has slopes everywhere exceeding $\hcL\circ A$ along
    $[\lambda_{t-1}, \lambda_t]$ (for either choice of step size),
    and so $\bar \alpha \leq \alpha_t \leq \bar\alpha_t$.
    (Indeed, these bounds give a derivation for the second step size choices.)
    To get the second guarantee, note that
    plugging
    $\bar\alpha$ into the above quadratic
    and simplifying via
    \[
        \ip{A^\top\nhcL(A\lambda_{t-1})}{v_t}^2
        \geq
        \rho^2\|A^\top\nhcL(A\lambda_{t-1})\|_\infty^2
    \]
    gives the desired minimum quadratic upper bound.
\end{proof}

\begin{lemma}
    \label[lemma]{fact:singlestep:wolfe}
    Let $\ell\in\Lbds$ with Lipschitz gradient parameter $B_2$, and iteration $t$ be given,
    and suppose $\alpha_t$ satisfies the Wolfe conditions for some $0 < c_1 < c_2 < 1$.  Then
    \begin{align*}
        \alpha_t &
        \geq \frac {\rho(1-c_2)\|A^\top\nhcL(A^\top\lambda_{t-1})\|_\infty}{B_2},
        \\
        \hcL(A\lambda_t) &\leq \hcL(A\lambda_{t-1})
        - \frac {\rho^2c_1(1-c_2)\|A^\top\nhcL(A\lambda_{t-1})\|_\infty^2}{B_2}.
    \end{align*}
\end{lemma}
\begin{proof}
    By the definition of $B_2$ and since $A$ has entries in $[-1,+1]$,
    \begin{align*}
        \ip{A^\top\nhcL(A^\top\lambda_t) - A^\top\nhcL(A^\top\lambda_{t-1})}{v_t}
        &= \frac 1 m\sum_i (\ell'((A^\top\lambda_t)_i) - \ell'((A^\top\lambda_{t-1})_i))(Av_t)_i
        \\
        &= \sum_i \frac{(\ell'((A^\top\lambda_t)_i) - \ell'((A^\top\lambda_{t-1})_i))}
            {m(\alpha_t Av_t)_i}\alpha_t(Av_t)_i^2
        \\
        &\leq \alpha_t B_2.
    \end{align*}
    The rest of the proof is just as for standard Wolfe search guarantees
    (cf. \citet[Theorem 3.2]{nocedal_wright} or \citet[Proposition D.6]{primal_dual_boosting_arxiv}),
    and direct from the Wolfe conditions.
    First, subtracting $\ip{A^\top \nhcL(A\lambda_{t-1})}{v_t}$ from both sides
    of \cref{eq:wolfe:2} gives
    \[
        \ip{A^\top\nhcL(A^\top\lambda_t) - A^\top\nhcL(A^\top\lambda_{t-1})}{v_t}
        \geq (c_2 - 1)\ip{A^\top\nhcL(A^\top\lambda_{t-1})}{v_t},
    \]
    which can be combined with the above derivation to yield
    \[
        \alpha_t
        \geq \frac {(c_2 - 1)\ip{A^\top\nhcL(A^\top\lambda_{t-1})}{v_t}}{B_2}
        \geq \frac {\rho(1-c_2)\|A^\top\nhcL(A^\top\lambda_{t-1})\|_\infty}{B_2}
    \]
    Plugging this into \cref{eq:wolfe:1} gives
    \[
        \hcL(A\lambda_t) \leq \hcL(A\lambda_{t-1})
        - \frac {\rho^2c_1(1-c_2)\|A^\top\nhcL(A\lambda_{t-1})\|_\infty^2}{B_2}.
    \]
\end{proof}

\section{Reweighted Margin Deviations (with $p$ Fixed)}



\begin{lemma}
    \label[lemma]{fact:pdc}
    Let probability measure $\mu$ over $\cX\times \{-1,+1\}$ with empirical counterpart
    $\hmu$,
    any hypothesis class $\cH\in \Fvc$,
    reweighting $p \in L^\infty(\mu)$,
    and norm bound $C$ be given.
    Then, with probability at least $1-\delta$,
    $p\in [0,\|p\|_\infty]$ $\hmu$-a.e., and
    \[
        \sup_{\substack{ \lambda \in \Lambda \\ \|\lambda\|_1\leq C}}
        \left|
        \int (A \lambda) p d\hmu -
        \int (A\lambda) p d\mu\right|
        \leq \frac {2C\|p\|_\infty}{m^{1/2}}
        \left(2\sqrt{2\cV(\cH)\ln(m+1)} + \sqrt{2\ln(1/\delta)}\right).
    \]
\end{lemma}

\begin{proofof}{\Cref{fact:pdc}}
    First, define a simplified reweighting $p'(x,y) := p(x,y) \1[|p(x,y)| \leq \|p\|_\infty]$;
    by the definition of $\|\cdot\|_\infty$, then $p' = p$ $\mu$-a.e., and thus, with
    probability 1, any finite sample of any size has $p'$ and $p$ agreeing.  The proof
    will work with $p'$, which satisfies $\sup_{x,y} |p'(x,y)| \leq \|p\|_\infty$,
    and then close by discarding a measure zero set and thus relating to $p$.

    The main part of the proof is an almost standard application
    of Rademacher complexity techniques
    for voted classifiers \citep[Theorem 4.1 and its proof, which controls for a
        surrogate
    loss and not just the classification loss]{bbl_esaim}; the only
    modification will be to work with a loss function which is sensitive
    to each example in the sample $S = \{(x_i,y_i)\}_{i=1}^m$,
    which will require a slightly refined Lipschitz contraction principle
    for Rademacher complexities \citep[Section 22.2, Lemma 15]{shai_course}.

    Specifically, define the loss
    \[
        \phi((A\lambda)_{x,y}) :=
        \begin{cases}
            -C p'(x,y) &\textup{when $(A\lambda)_{x,y} \leq -C$},
            \\
            (A\lambda)_{x,y} p'(x,y)
            &\textup{when $|(A\lambda)_{x,y}| \leq C$},
            \\
            +C p'(x,y) &\textup{when $(A\lambda)_{x,y} \geq C$}.
        \end{cases}
    \]
    Since $\sup_{h,x} |h(x)| \leq 1$ and $\|\lambda\|_1 \leq C$, it follows that
    $|(A\lambda)_{x,y}| \leq C$, and thus the extremal cases are never encountered,
    meaning
    \[
        \phi((A\lambda)_{x,y})
        = (A\lambda)_{x,y} p'(x,y),
    \]
    and by construction $\phi$ is Lipschitz with parameter $\|p\|_\infty$ (as a function of
    $(A\lambda)$) and $\phi \circ (A\lambda)$ has uniform bound $C\|p\|_\infty$.

    As such, letting $R$ denote Rademacher complexity,
    by the Lipschitz contraction principle for per-coordinate losses
    \citep[Section 22.2, Lemma 15]{shai_course},
    behavior of Rademacher complexity on convex hulls
    \citep[Theorem 3.3]{bbl_esaim},
    and relationship between Rademacher complexity and VC dimension
    \citep[See the display after eq. (7)]{bbl_esaim},
    \[
        R(\phi \circ (A\lambda))
        \leq \|p\|_\infty R((A\lambda))
        \leq \|p\|_\infty C R(\cH)
        \leq \|p\|_\infty C \sqrt{\frac {2\cV(\cH)\ln(m+1)}{m}}.
    \]
    This handling of a per-coordinate Lipschitz loss may be inserted into a
    standard deviation bound for uniformly bounded Lipschitz losses
    \citep[Theorem 4.1 and its proof]{bbl_esaim} --- albeit with an extra factor two to
    control deviations in both directions --- and it follows,
    with probability at least $1-\delta$, that
    \[
        \sup_{\substack{ \lambda \in \Lambda \\ \|\lambda\|_1\leq C}}
        \left|
        \int (A \lambda) p' d\hmu -
        \int (A\lambda) p' d\mu\right|
        \leq \frac {2C\|p\|_\infty}{m^{1/2}}
        \left(2\sqrt{2\cV(\cH)\ln(m+1)} + \sqrt{2\ln(1/\delta)}\right).
    \]
    To complete the proof, recall that $p' = p$ $\mu$-a.e., and a measure zero event
    was discarded, whereby $p' = p$ $\hmu$-a.e. as well.
\end{proofof}



\section{Deferred Material from \Cref{sec:separable}}

\subsection{Deviations of $\gamma_\epsilon(\nu)$}

This subsection establishes the following one-sided deviation bound on
$\gamma_\epsilon(\nu)$.

\begin{lemma}
    \label[lemma]{fact:gamma_epsilon:deviations}
    Let any $\cH$,
    any $\epsilon \in(0,1]$,
    any confidence parameter $\delta \in (0,1]$,
    and any probability measure $\mu$
    with empirical counterpart $\hmu$ be given.
    Then with probability at least $1-\delta$,
    \[
        \gamma_\epsilon(\hmu) \geq
        \gamma_{\epsilon}(\mu)
        - \frac 1 {\epsilon} \sqrt{
            \frac{1}{2m} \ln \left (\frac 2 \delta\right)
        }.
    \]
\end{lemma}

The difficulty in the analysis is that the definition of $\gamma_\epsilon(\nu)$
involves an
infimum over $p\in\cD_{\epsilon}(\nu)$
and a supremum over $\lambda\in \Lambda$ with $\|\lambda\|_1\leq 1$.
The proof strategy employed here is to consider a single good choice
for $\lambda$, and to consider the effect on deviations as $p$ varies.
These deviations do not appear to be amenable to the usual approach, as
$\cD_\epsilon(\nu)$ is massive: it is in general not
compact in the relevant metric topologies (cf. \Cref{fact:cDeps_basic}),
and does not obviously possess other structure granting a uniform convergence result.
The approach here is to instead identify that the dual optimum has very simple structure,
and moreover this structure is robust to sampling.

Considering again the definition of $\gamma_\epsilon(\nu)$, while it is true that
$p\in\cD_{\epsilon}(\nu)$ is defined over a potentially massive space, when placed in the
expression $\int (A\lambda)p d\nu$, all that matters is the behavior of $p$ for each
value of $A\lambda$, which ranges over $[-1, +1]$.  That is to
say, $p$ is really reweighting the univariate margin distribution of $A\lambda$,
and the best it can do is emphasize bad margins.  In particular, the following
\namecref{fact:sep:dualopt_univariate}
proves basic properties of an idealized univariate distillation of this scenario.

\begin{lemma}
    \label[lemma]{fact:sep:dualopt_univariate}
    Let a probability measure $\xi$ supported on $[-1,+1]$ and some $\epsilon\in(0,1]$ be given.
    Correspondingly define
    \begin{align*}
        S_\epsilon &:= \{ c\in [-1,+1] : \xi((-\infty, c)) \leq \epsilon\},
        \\
        c_\epsilon &:= \sup S_\epsilon,
        \\
        I_\epsilon &:= (-\infty,c_\epsilon),
        \\
        p_\epsilon(r) &:= \frac 1 \epsilon \left ( \1[r \in I_\epsilon]
        + \frac { (\epsilon - \xi(I_\epsilon) )}{\xi(\{c_\epsilon\})}
    \1[r = c_\epsilon] \right),
    \end{align*}
    with the convention $0/0 = 0$ in the definition of $p_\epsilon$.
    These objects have the following properties.
    \begin{enumerate}
        \item $S_\epsilon$ is the closed interval $[-1,c_\epsilon]$.
        \item $\xi(I_\epsilon) \leq \epsilon$, and $\xi(I_\epsilon \cup \{c_\epsilon\}) \geq \epsilon$.
        \item $\|p_\epsilon\|_1 = 1$ and $\|p_\epsilon\|_\infty \leq 1/\epsilon$.
        \item The optimization problem
            \[
                \inf \left\{ \int r p(r) d\xi(r) : p\in L^\infty(\xi), \|p\|_1 = 1, p \in [0,1/\epsilon]\ \xi\textup{-a.e.}\right\}
            \]
            is minimized at $p_\epsilon$.
    \end{enumerate}
\end{lemma}
\begin{proof}
    First note that $-1 \in S_\epsilon$, since $\xi$ is supported on $[-1,+1]$
    and thus $\xi((-\infty,-1)) = 0$.

    Next, $S_\epsilon$ is an interval, since if $-1 \leq c_1 \leq c_2$ and $c_2 \in S_\epsilon$,
    then $\xi((-\infty,c_1)) \leq \xi((-\infty,c_2))$ and thus $c_1 \in S_\epsilon$.

    To show that $S_\epsilon$ is indeed a closed interval, consider any increasing sequence
    $\{c_i\}_{i=1}^\infty$ with $c_i \in S_\epsilon$, thus $c_i \to c$ for some $c\in[-1,+1]$
    since $[-1,+1]$ is compact and the sequence is increasing.  Then
    \[
        (-\infty,c) = \cup_{i=1}^\infty (-\infty,c_i),
    \]
    and thus, by continuity of measures \citet[Theorem 1.8]{folland},
    \[
        \xi((-\infty,c))
        = \xi(\cup_{i=1}^{\infty} (-\infty,c_i))
        = \lim_{i\to\infty} \xi((-\infty,c_i))
        \leq \limsup_{i\to\infty} \xi((-\infty,c_i))
        \leq \epsilon,
    \]
    meaning $c\in S_\epsilon$ and $S_\epsilon$ is closed.

    Since $S_\epsilon$ is a closed interval, then $c_\epsilon = \sup S_\epsilon \in S_\epsilon$,
    and it follows by the preceding properties that $S_\epsilon = [-1,c_\epsilon]$.

    By definition, for every $c\in S_\epsilon$, it holds that $\xi((-\infty,c))\leq \epsilon$,
    thus $c_\epsilon\in S_\epsilon$ implies that $\xi(I_\epsilon) \leq \epsilon$.

    Next, for every positive integer $i\in \Z_{++}$, it holds by definition of $c_\epsilon$
    that $c_\epsilon + 1/i\not\in S_\epsilon$, and thus, again by continuity of measures
    \citet[Theorem 1.8]{folland},
    \[
        \xi([-1,c_\epsilon])
        = \xi\left(\cap_{i=1}^\infty [-1,c_\epsilon + 1/i]\right)
        \geq \liminf_{i\to\infty} \xi\left([-1,c_\epsilon + 1/i]\right)
        \geq \epsilon.
    \]

    For the norms of $p_\epsilon$
    (which is a simple function over the Borel $\sigma$-algebra),
    notice that
    \[
        \|p_\epsilon\|_1
        = \frac 1 \epsilon \left(
        \xi(I_\epsilon) + \frac { (\epsilon - \xi(I_\epsilon) )}{\xi(\{c_\epsilon\})} \xi(\{c_\epsilon\})\right)
        = 1.
    \]
    Moreover, $p_\epsilon = 1/\epsilon$ on $(-\infty,c_\epsilon)$, and $p_\epsilon = 0$ on $(c_\epsilon,\infty)$;
    to show $\|p\|_\infty \leq 1$, the behavior of $p_\epsilon$ on $c_\epsilon$ is all that needs to be
    checked.
    Since $\xi((-\infty,c_\epsilon]) \geq \epsilon$, then
    \[
        \epsilon - \xi(I_\epsilon)
        = \epsilon - \xi((-\infty,c_\epsilon]) + \xi(\{c_\epsilon\})
        \leq \xi(\{c_\epsilon\}),
    \]
    so $p_\epsilon(c_\epsilon) \leq 1/\epsilon$.
    Additionally $\xi(I_\epsilon) \leq \epsilon$ implies $p_\epsilon(c_\epsilon) \geq 0$,
    and thus $\|p\|_\infty \leq 1/\epsilon$ as desired.

    Lastly, for the minimization problem, consider any feasible $p$ (meaning
    $\|p\|_1 = 1$ and $\|p\|_\infty \leq 1/\epsilon$) with $\|p_\epsilon - p\|_1 > 0$.
    But since $p_\epsilon$ is as large as possible along $I_\epsilon$, it follows
    that $p < p_\epsilon$ for a positive measure subset of $I_\epsilon$, and $p > p_\epsilon$
    for a positive measure subset of $[c_\epsilon,\infty)$.  Consequently
    $\int r p_\epsilon(r) d\xi(r) < \int r p'(r) d\xi(r)$.  Since $p'$ was arbitrary,
    it follows that $p_\epsilon$ is a minimal choice.
\end{proof}

The task now is to map the optimization over $\cD_\epsilon(\nu)$ down to this idealized
univariate search problem.   Temporarily adopting notation from probability theory,
a first step in this direction would be to write
\[
    \int (A\lambda) pd\nu
    = \bbE_\nu((A\lambda)p)
    = \bbE_\nu( \bbE((A\lambda) p | (A\lambda) = r) ),
\]
where the latter notation signifies a conditional expectation with respect the
$\sigma$-algebra generated by events such that $(A\lambda)$ falls in some Borel subset
of $\R$ (recall that all $\sigma$-algebras here are Borel).
In some circumstances, the function $\bbE((A\lambda) p | (A\lambda) = r)$ can be
converted into integration over a function that takes $r$ as input, which would directly
allow conversion to the above univariate idealization; these techniques
generally require assumptions on $\cX\times \{-1,+1\}$ which would rather be
avoided here
\citep[Section 5.1.3, regular conditional probabilities]{durrett_prob}.
As such, the following result exhibits the desired correspondence manually,
albeit keeping the above
idea in mind.

\begin{lemma}
    \label[lemma]{fact:sep:dualopt_multivariate}
    Let any $\epsilon \in (0,1]$,
    any probability measure $\nu$ over $\cX\times \{-1,+1\}$,
    any $\cH$,
    and any $\lambda \in \Lambda$ with $\|\lambda\|_1 \leq 1$ be given.
    Define a probability measure $\xi$ over $\R$ as the pushforward of $\nu$ through
    $A\lambda$, meaning,
    for any Borel subset $S$ of $\R$,
    \[
        \xi(S) := \nu((A\lambda)^{-1}(S)).
    \]
    Then $\xi$ is supported on $[-1,+1]$, and moreover the function
    $p_{\epsilon}^\nu(x,y) := p_{\epsilon}((A\lambda)_{x,y})$,
    where $p_{\epsilon}$ is
    as defined in \Cref{fact:sep:dualopt_univariate}, is
    a (feasible) minimizer to the optimization problem
    \[
        \inf \left\{
            \int (A\lambda) p d\nu
            :
            p \in \cD_\epsilon(\nu)
        \right\}.
    \]
\end{lemma}

\begin{proof}
    Since $\|\lambda\|_1\leq 1$ and $A$ is a continuous linear operator with unit
    norm (cf. \Cref{fact:AH_basic}, or recall the definition of $A$ and the
    property $\sup_{x,h} |h(x)| \leq 1$), then $|(A\lambda)_{x,y}| \leq 1$, and
    thus $(x,y) \mapsto (A\lambda)_{x,y}$ maps $\cX \times \{-1,+1\}$ to $[-1,+1]$,
    and so the corresponding pushforward measure $\xi$ is supported on $[-1,+1]$.
    Therefore \Cref{fact:sep:dualopt_univariate} provides
    the structure of $p_{\epsilon}\in L^\infty(\xi)$
    attaining the minimum in
    \begin{align*}
        &
        \inf \left\{ \int r p(r) d\xi(r) : p\in L^\infty(\xi), \|p\|_1 = 1, p \in [0,1/\epsilon]\ \xi\textup{-a.e.}\right\}
        .
    \end{align*}
    Setting $p_{\epsilon}^\nu := p_{\epsilon}
    \circ (A\lambda)$ as in the statement,
    by the above optimality guarantee
    and by properties of pushforward measures
    \citep[Theorem 5.5.1]{resnick_prob},
    \begin{align}
        &
        \inf \left\{ \int r p(r) d\xi(r) : p\in L^\infty(\xi), \|p\|_1 = 1, p \in [0,1/\epsilon]\ \xi\textup{-a.e.}\right\}
        \label{eq:gamma_epsilon:dualopts:hellofriends}
        \\
        &\qquad
        = \int r p_{\epsilon}(r) d\xi(r)
        \notag\\
        &\qquad
        = \int (A\lambda) p_{\epsilon}^\nu d\mu
        \notag\\
        &\qquad
        \geq
        \inf \left\{
            \int (A\lambda) p d\nu
            :
            p \in \cD_\epsilon(\nu)
        \right\}.
        \notag
    \end{align}

    Now let $\sigma >0$ and $p\in \cD_\epsilon(\nu)$ be arbitrary.
    A corresponding element
    $q$ with $\|q\|_1 = \|p\|_1$ and $\|q\|_\infty \leq \|p\|_\infty$
    will be constructed as follows in order to satisfy
    \[
        \left| \int (A\lambda) p d\mu - \int r q(r) d\xi(r)\right| \leq \sigma.
    \]
    Cover $[-1,+1]$ with at most $1+\lceil 1/\sigma\rceil$ disjoint
    half-open intervals $\{I_i\}_{i=1}^k$
    of the form $[-1 + i\sigma', -1 + (i+1)\sigma')$ where
    $i$ is a nonnegative integer and
    $\sigma' := \sigma / (1+\lceil 1/\sigma \rceil)$.
    Define
    \[
        q(r)
        :=
        \sum_{i=1}^k \xi(I_i)^{-1} \1[r\in I_i] \int_{[A\lambda \in I_i]} p d\mu,
    \]
    with the convention $0/0 = 0$ (i.e., $q(I_i) = 0$ when $\xi(I_i) = 0$).
    By this choice, $\|q\|_\infty \leq \|p\|_\infty$, and
    \[
        \|q\|_1
        =
        \int q(r) d\xi(r)
        =
        \sum_{i=1}^k \int_{[A\lambda \in I_i]} p d\mu
        = \|p\|_1.
    \]
    More importantly, using Fubini's Theorem to interchange the
    integrals over $\xi$ and $\mu$,
    \begin{align*}
        \left|
            \int (A\lambda) p d\mu
            - \int r q(r) d\xi(r)
        \right|
        &\leq
        \sum_{i=1}^k
        \left|
            \int_{[A\lambda \in I_i]} (A\lambda) p d\mu
            - \int_{I_i} r \xi(I_i)^{-1}\left(\int_{[A\lambda \in I_i]} pd\mu\right) d\xi(r)
        \right|
        \\
        &\leq
        \sum_{i=1}^k
        \int_{[A\lambda \in I_i]}
        \left|
        (A\lambda)p
        -  p \int_{I_i} r \xi(I_i)^{-1}d\xi(r)
        \right| d\mu
        \\
        &\leq
        \sum_{i=1}^k
        \int_{[A\lambda \in I_i]} |p| |\sigma'|
        \leq \sigma.
    \end{align*}
    Since $\sigma$ and $p$ were arbitrary,
    \begin{align*}
        &\inf \left\{
            \int (A\lambda) p d\nu
            :
            p\in\cD_{\epsilon}(\nu)
        \right\}
        \\
        &\qquad
        \geq
        \inf \left\{ \int r p(r) d\xi(r) : p\in L^\infty(\xi), \|p\|_1 = 1, p \in [0,1/\epsilon]\ \xi\textup{-a.e.}\right\}
        ,
    \end{align*}
    which combined with the inequalities starting with
    \cref{eq:gamma_epsilon:dualopts:hellofriends}
    provides that $p_{\epsilon}^\nu$ is indeed
    a minimizer.
\end{proof}

With these tools in place, the proof of \Cref{fact:gamma_epsilon:deviations} follows.

\begin{proofof}{\Cref{fact:gamma_epsilon:deviations}}
    Consider the form of $\gamma_\epsilon(\nu)$ provided by
    \Cref{fact:gamma_epsilon:duality}, whereby the supremum over $\lambda\in\Lambda$ is
    on the outside.
    Let $\sigma > 0$ be arbitrary, choose $\lambda\in\Lambda$ which
    is within $\sigma>0$ of achieving the supremum,
    and let $p_{\epsilon}^\mu$ be an optimal dual element as provided by
    \Cref{fact:sep:dualopt_multivariate}, together meaning
    \begin{equation}
        \gamma_{\epsilon}(\mu)
        \leq
        \sigma + \inf_{p \in \cD_\epsilon(\mu)} \int (A\lambda) p d\mu
        =
        \sigma + \int (A\lambda)p_{\epsilon}^\mu d\mu.
        \label{eq:gamma_epsilon:dualopt:ugh1}
    \end{equation}

    Now consider the behavior of $p_{\epsilon}^\mu$ over $\hmu$.
    By construction, $\|p_{\epsilon}^\mu\|_{L^\infty(\hmu)} \leq 1/\epsilon$,
    however
    $\|p_{\epsilon}^\mu\|_{L^1(\hmu)}$ is a random variable;
    but by Hoeffding's inequality,
    with probability at least $1-\delta$,
    \[
       \Big| \|p_{\epsilon}^\mu\|_{L^1(\hmu)}
         - \|p_{\epsilon}^\mu\|_{L^1(\mu)} \Big|
         =
       \Big| \|p_{\epsilon}^\mu\|_{L^1(\hmu)}
         - 1 \Big|
        \leq
        \frac 1 \epsilon \sqrt{
            \frac{1}{2m} \ln \left (\frac 2 \delta\right)
        };
    \]
    henceforth discard this failure event.

    Next instantiate another dual optimum $p_\epsilon^\hmu$ via
    \Cref{fact:sep:dualopt_multivariate}, but now over the empirical
    measure $\hmu$; since $\lambda\in\Lambda$ is primal feasible in the definition
    of $\gamma_\epsilon(\hmu)$, and again using the form from
    \Cref{fact:gamma_epsilon:duality} with the supremum on the outside, it follows
    that
    \begin{equation}
        \gamma_{\epsilon}(\hmu) \geq \int (A\lambda) p_{\epsilon}^{\hmu} d\hmu.
        \label{eq:gamma_epsilon:dualopt:ugh2}
    \end{equation}
    Now recall the exact form of $p_{\epsilon}^\mu$ and $p_{\epsilon}^\hmu$ as
    provided by \Cref{fact:sep:dualopt_multivariate} (and more
    specifically \Cref{fact:sep:dualopt_univariate}), which are both exactly $1/\epsilon$
    up to some point, within $[0,1/\epsilon]$ at that point (potentially distinct
    for $p_\epsilon^\hmu$ and $p_\epsilon^\mu$), and zero thereafter;
    if $\|p_{\epsilon}^\mu\|_{L^1(\hmu)} \geq 1$, then
    \[
        \int |p_{\epsilon}^\mu - p_{\epsilon}^\hmu| d\hmu
        = \int p_{\epsilon}^\mu d\hmu - \int p_{\epsilon}^\hmu d\hmu
        = \int p_{\epsilon}^\mu d\hmu - 1,
    \]
    whereas $\|p_{\epsilon}^\mu\|_{L^1(\hmu)} \leq 1$ implies
    \[
        \int |p_{\epsilon}^\mu - p_{\epsilon}^\hmu| d\hmu
        = \int p_{\epsilon}^\hmu d\hmu - \int p_{\epsilon}^\mu d\hmu
        = 1 - \int p_{\epsilon}^\mu d\hmu.
    \]
    In either case, using as usual the fact
    $\sup_{x,y} |(A\lambda)_{x,y}| \leq \|\lambda\|_1 \leq 1$,
    and additionally the controls on $\|p_\epsilon^\hmu\|_{L^1(\hmu)}$ from above,
    \begin{align*}
        \left|
            \int (A\lambda) p_\epsilon^\hmu d\hmu
            -
            \int (A\lambda) p_\epsilon^\mu d\hmu
        \right|
        &\leq
        \left|
            \int p_\epsilon^\hmu d\hmu
            -
            \int p_\epsilon^\mu d\hmu
        \right|
        \\
        &\leq
        \int
        \left|
            p_\epsilon^\hmu
            -
            p_\epsilon^\mu
        \right| d\hmu
        \\
        &\leq
        \frac 1 \epsilon \sqrt{
            \frac{1}{2m} \ln \left (\frac 2 \delta\right)
        }.
    \end{align*}
    Combining this with \cref{eq:gamma_epsilon:dualopt:ugh1,eq:gamma_epsilon:dualopt:ugh2},
    \begin{align*}
        \gamma_{\epsilon}(\hmu)
        &\geq \int (A\lambda) p_{\epsilon}^{\hmu} d\hmu
        \\
        &\geq \int (A\lambda) p_{\epsilon}^{\mu} d\hmu
        -
        \frac 1 \epsilon \sqrt{
            \frac{1}{2m} \ln \left (\frac 2 \delta\right)
        }
        \\
        &\geq
        \gamma_{\epsilon}(\mu) -\sigma
        -
        \frac 1 \epsilon \sqrt{
            \frac{1}{2m} \ln \left (\frac 2 \delta\right)
        }.
    \end{align*}
    Since $\sigma > 0$ was arbitrary, the result follows.
\end{proofof}

\subsection{Other Results}
\label{sec:app:sep:other}

\begin{lemma}
    \label[lemma]{fact:gamma_epsilon:nondecreasing}
    Let $\nu$ be a probability measure on $\cX \times \{-1,+1\}$.
            If $0 \leq \epsilon_1 \leq \epsilon_2 \leq 1$,
            then
            $0 \leq \gamma_{\epsilon_1}(\nu) \leq \gamma_{\epsilon_2}(\nu) \leq 1$.
\end{lemma}
\begin{proof}
    Let $0 \leq \epsilon_1 \leq \epsilon_2 \leq 1$ be given; then
    $\cD_{\epsilon_1}(\nu) \supseteq \cD_{\epsilon_2}(\nu)$ by definition,
    and thus $\gamma_{\epsilon_1}(\nu) \leq \gamma_{\epsilon_2}(\nu)$.
    Next, $\gamma_{\epsilon_1}(\nu) \geq 0$ follows by \Cref{fact:gamma_epsilon:duality}
    since $\|A^\top p\|_\infty \geq 0$, or by considering the effect of the primal player
    choosing $\lambda = 0 \in \Lambda$.
    For the upper bound, since $\sup_{h,x}|h(x)| \leq 1$, then
    $
        \gamma_{\epsilon_2}(\nu)
        \leq
        \sup \left\{
            \|\lambda\|_1
            :
            \lambda \in \Lambda, \|\lambda\|_1 \leq 1
        \right\}
        \leq 1.
    $
%
%
\end{proof}

\begin{proofof}{\Cref{fact:gamma:basic:equiv}}
    Since every $p\in L^1(\nu)$ with $\|p\|_1=1$ and $p\geq 0$ $\nu$-a.e. defines
    a probability measure $pd\nu$ (ignoring a $\nu$-null set which does not affect
    that value of integration with respect to $\nu$),
    and since $\|\pm \bfe_h\|_1 = 1$ for every
    $h\in \cH$,
    \begin{align*}
        \gamma
        &\leq
        \inf \left\{
            \sup_{h\in \cH}
            \left|
            \int yh(x) d\xi(x,y)
            \right|
            :
            \textup{$\xi$ is a Borel probability measure over $\cX\times \{-1,+1\}$}
        \right\}
        \\
        &\leq
        \inf \left\{
            \sup_{\|\lambda\|_1 \leq 1}
            \int (A\lambda)_{x,y} p(x,y) d\nu(x,y)
            :
            p\in L^1(\nu), \|p\|_1 = 1,
            p \geq 0\ \nu\textup{-a.e.}
        \right\}
        \\
        &=
        \gamma_0(\mu).
    \end{align*}
    For $\gamma(\nu)\leq \gamma_0(\nu)$ with $\nu$ a discrete measure over a finite set,
    the proof is as above (indeed with a tiny refinement, since in this case both
    $\gamma(\nu)$ and $\gamma_0(\nu)$ consider the same set of weightings over $\nu$).
\end{proofof}

\begin{proofof}{\Cref{fact:gamma:basic:bad}}
    For convenience, define $I_i := (1/(i+1), i]$.
    This proof will proceed by establishing, for every $k$, a bounded
    weighting $p_k$, which will establish an upper bound on $\gamma_\epsilon(\mu)$
    for some $\epsilon>0$ which is a function of $k$.  The result will then follow
    for $\gamma_0(\mu)$ by the monotonicity of $\gamma_\epsilon(\mu)$ as a function
    of $\mu$ (cf. \Cref{fact:gamma_epsilon:nondecreasing}), and result for
    $\gamma_0(\hmu)$ will use deviation bounds on $\gamma_\epsilon(\mu)$ and again
    the monotonicity property.

    Define $p_k$ to be positive over intervals $I_i$ with $1\leq i \leq 2k$, and zero
    elsewhere as follows.  For any $x\in I_i$ , $p_k(x,y) = i(i+1)/ (2k)$.
    By this choice,
    \[
        \int_{I_i} p(x,y) d\mu(x,y) = \frac {i(i+1)}{2k}\left(\frac 1 i - \frac 1 {i+1}\right)
        = \frac 1 {2k};
    \]
    It follows that $\|p\|_1 = 1$, $\|p\|_\infty = 2k+1$, and $p_k$ makes
    $\mu$ look like the uniform distribution over $k$ consecutive intervals.

    Now consider any hypothesis $h\in \cH$, with some threshold $r$.
    If $r$ lies outside this set of intervals, then $h$ is equally correct and
    incorrect, thus
    $\int y h(x) p(x,y) d\mu(x,y) = 0$.
    Otherwise, suppose there are $a$ intervals before the threshold, and $b$ intervals
    after it;  $h$ must be incorrect on at least $a/2-1$ of the left intervals,
    and $b/2-1$ of the right intervals; since $2k-1 \leq a+b \leq 2k$ (this proof is
    charitable), thus at least $k-2$ intervals are predicted completely incorrectly.
    Consequently,
    \[
            \int yh(x) p(x,y) d\mu(x,y) \leq \frac {k+2}{2k} - \frac {k-2}{2k}
            = \frac 2 k.
    \]
    Thus the form of $\gamma_\epsilon(\nu)$ from
    \Cref{fact:gamma_epsilon:duality}
    provides
    $\gamma_{1/2k}(\mu) \leq 2/k$, and so $\gamma_0(\mu) = 0$
    by monotonicity (\Cref{fact:gamma_epsilon:nondecreasing}),
    and $\gamma \leq 0$ by \Cref{fact:gamma:basic:equiv}.

    The remainder of the proof considers finite sample effects.
    Let $\delta >0$, and a sample of size $m \geq 2$ be given.
    Choose integer $k := \lfloor m^{1/4} / (3\sqrt{2\ln(4/\delta)})\rfloor$
    (where the lower bound on $m$ provides $k\geq 1$),
    and consider the behavior of density $p_k$, defined as above.
    Note firstly that
    $\|p_k\|_\infty \leq 2k+1 \leq 3k \leq m^{1/4} / \sqrt{2\ln(4/\delta)}$.
    Next, with probability at least
    $1-\delta/2$, Hoeffding's bound grants
    \[
        \left|
        \int p d\hmu
        - \int p d\mu
        \right|
        \leq \|p\|_\infty \sqrt{\frac 1 {2m} \ln \left(\frac 4 \delta\right)}
        \leq \frac {1}{2m^{1/4}} \leq \frac 1 2.
    \]
    Now define $p_k' := p_k / (\int p_k d\hmu)$,
    which means $\int p_k' d\hmu =1$,
    and furthermore $\|p_k'\|_\infty \leq 2\|p_k\|_\infty$.
    Since $\cH$ has VC dimension $\cV(\cH) = 2$,
    \Cref{fact:pdc} grants, with probability at least $1-\delta/2$,
    \begin{align*}
        &\sup_{h\in \cH}
        \left|
        \int y h(x) p_k'(x,y) d\hmu(x,y)
        - \int y h(x) p_k'(x,y) d\mu(x,y)
        \right|
        \\
        &\qquad\leq
        \frac {2\|p_k'\|_\infty}{\sqrt m}
        \left(
            4 \sqrt{\ln(m+1)} + \sqrt {2\ln(2/\delta)}
        \right).
    \end{align*}
    Now set $\epsilon := 1/\|p_k'\|_\infty \geq \sqrt{2\ln(4/\delta)}/ (2m^{1/4})$.
    Then $p_k' \in \cD_\epsilon(\hmu)$,
    and the above computations provide
    \[
        \gamma_{\epsilon}(\hmu)
        \leq
        \frac {16(\sqrt{\ln(m+1)} + 1)}{m^{1/4}}
        + \frac 2 k
        \leq
        \frac {16(\sqrt{\ln(m+1)} + 1)}{m^{1/4}}
        + \frac 2 {m^{1/4}/(3\sqrt{2\ln(4/\delta)}) - 1},
    \]
    and lastly \Cref{fact:gamma_epsilon:nondecreasing} and
    \Cref{fact:gamma:basic:equiv} grant
    $\gamma(\hmu) \leq \gamma_0(\hmu) \leq \gamma_{\epsilon}(\hmu)$.
\end{proofof}

\begin{proofof}{\Cref{fact:gamma_epsilon:basic}}
    \begin{enumerate}
        \item This proof will proceed by establishing the contrapositive twice,
            and then using the fact that $\bar \cL \geq 0$ and
            $\gamma_\epsilon(\mu) \geq 0$.

        If $\bar \cL > 0$, then there must exist a nonzero dual feasible point to
        the dual of $\cL$ in \Cref{fact:duality}, since \Cref{fact:loss:basic}
        grants that $\ell^*(0) =0$.  This nonzero dual feasible point $p$ satisfies
        $p\in L^\infty(\mu)$ by the form of the duality problem, and thus
        $\tilde p := p / \|p\|_1$ also has $\tilde p \in L^\infty(\mu)$.
        The dual constraint provides $\|A^\top p\|_\infty = 0$,
        thus $\|A^\top \tilde p\|_\infty = 0$,
        and so \Cref{fact:AT} grants
        $\gamma_{\epsilon}(\mu) = 0$ with the choice
        $\epsilon = 1/\|\tilde p\|_\infty$ (and $\epsilon \leq 1$
        since $\|\tilde p\|_1 = 1$ and $\mu$ a probability measure means
        $\|\tilde p\|_\infty \geq 1$).

        On the other hand, if there exists $\epsilon$ so that
        $\gamma_{\epsilon}(\mu) = 0$, then attainment in the duality formula
        in \Cref{fact:gamma_epsilon:duality} provides the existence of $p\in L^\infty(\mu)$
        with $\|p\|_1 = 1$ and $\|A^\top p\|_\infty = 0$.
        By \Cref{fact:loss:basic}, there exists $c > 0$ so that $\ell^*$ is strictly
        negative along $(0,c)$.  Consequently, $\tilde p := cp / \|p\|_\infty$
        satisfies $\|\tilde p\|_1 > 0$, and  $\tilde p \in (0,c)$ $\mu$-a.e.,
        thus $\ell^*(\tilde p)< 0$ $\mu$-a.e.,
        and also $\|A^\top \tilde p\|_\infty = 0$;
        together, it follows that
        $\bar \cL \geq -\int \ell^*(\tilde p) > 0$ as desired.

    \item
        This result is the same as \Cref{fact:gamma_epsilon:deviations}.
\end{enumerate}
\end{proofof}

\begin{proofof}{\Cref{fact:gamma_epsilon:duality:body}}
    This result is the combination of \Cref{fact:gamma_epsilon:duality}
    and \Cref{fact:AT}.
\end{proofof}

\subsection{Optimization Guarantees}



Note that the following proof does not overtly use convexity;
convexity however is used both algorithmically by the line searches (otherwise they are not
efficient), and for their guarantees
(cf. \Cref{fact:singlestep:quadub,fact:singlestep:wolfe}).

\begin{proofof}{\Cref{fact:sep:iterhelper}}
    Consider any $0\leq t \leq T-1$.
    Since $\ell\in\Lbds$,
    \[
        \|\nhcL(A\lambda_t)\|_1
        \geq
        \frac 1 m \sum_{\substack{i\in [m]\\ (A\lambda_t)_{x_i,y_i} \geq 0}}
        \ell'((A\lambda_t)_{x_i,y_i})
        \geq \epsilon_t \beta_1.
    \]
    Combining this with the fact that $\ell' \in [0,\beta_2]$,
    the vector $p_t := \nhcL(A\lambda_t) / \|\nhcL(A\lambda_t)\|_1$
    satisfies $\|p_t\|_1 =1$ and
    \[
        \|p_t\|_\infty
        \leq \frac {\beta_2}{\|\nhcL(A\lambda_t)\|_1}
        \leq \frac {\beta_2 }{\epsilon_t \beta_1}
        = \frac {1}{\epsilon_t'},
    \]
    where $\epsilon_t'$ is as provided in the statement (and $\epsilon_t'\leq 1$
    since $\epsilon_t \leq 1$ and $\beta_1\leq \beta_2$).
    Recalling the dual form $\gamma_{\epsilon_t'}(\hmu) = \min\{ \|A^\top p\|_\infty:
    p \in \cD_{\epsilon_t'}(\hmu)\}$ from \Cref{fact:gamma_epsilon:duality},
    and noting that $p_t \in \cD_{\epsilon'_t}(\hmu)$,
    \[
        \|A^\top\nhcL(A\lambda_{t})\|_\infty
        = \|\nhcL(A\lambda_t)\|_1 \|A^\top p_t\|_\infty
        \geq \epsilon_t \beta_1 \gamma_{\epsilon'_t}(\hmu).
    \]
    Plugging this into
    the single-step guarantees from the three line search choices
    (cf. \Cref{fact:singlestep:quadub,fact:singlestep:wolfe}),
    \begin{align*}
        \hcL(A\lambda_t)
        &\leq \hcL(A\lambda_{t-1})
        - \frac {\rho^2\|A^\top\nhcL(A\lambda_{t-1})\|_\infty^2}{6B_2}
        \\
        &\leq \hcL(A\lambda_{t-1})
        - \frac {(\rho\beta_1\epsilon_{t-1}\gamma_{\epsilon_{t-1}'}(\hmu))^2}{6B_2}.
    \end{align*}
    The desired result comes by summing across all iterations and noting
    $\hcL(A\lambda_0) = \hcL(0) = \ell(0)$.
\end{proofof}

\subsection{Statistical Guarantees}


\begin{proofof}{\Cref{fact:sep:basic}}\label[proof]{proof:sep:basic}
    The first step of the proof is to show $\hcR(H\hat\lambda) \leq \epsilon$.
    Thus consider the case that every iteration has $\hcR(H\lambda_t) > \epsilon$;
    by the monotonicity of $\gamma_\epsilon(\mu)$ (cf. \Cref{fact:gamma_epsilon:nondecreasing}),
    positivity of $\gamma_\epsilon(\mu)$ (cf. \Cref{fact:gamma_epsilon:basic}),
    together with the bound on $m$ (and the deviations on
    $\gamma_\epsilon(\mu)$ in \Cref{fact:gamma_epsilon:basic}),
    with probability at least $1-\delta/2$,
    \[
        \gamma_{\epsilon_t'}(\hmu)
        \geq \gamma_{\epsilon'}(\hmu)
        \geq \gamma_{\epsilon'}(\mu) -\frac 1 {\epsilon} \sqrt{
    \frac{1}{2m} \ln \left (\frac 4 \delta\right)}
    \geq \frac{\gamma_{\epsilon'}(\mu)}{2}
        > 0,
    \]
    where the last equality is also by \Cref{fact:gamma_epsilon:basic}.
    Thus, by \Cref{fact:sep:iterhelper}, and the monotonicty of $\gamma_\epsilon(\mu)$ in
    $\epsilon$, and the second lower bound on $m$ (and thus on $m^a$),
    \begin{align*}
        \hcL(A\lambda_T)
        &\leq
        \ell(0)
        - \sum_{t=1}^{T} \frac {(\rho\beta_1\epsilon_{t-1}\gamma_{\epsilon_{t-1}'}(\hmu))^2}{6B_2}
        \\
        &\leq
        \ell(0)
        - \frac {m^a(\rho\beta_1\epsilon\gamma_{\epsilon'}(\mu))^2}{24B_2}
        \\
        &\leq 0,
    \end{align*}
    a contradiction since $\ell$ is nonnegative, and moreover positive on regions
    where it makes mistakes, therefore the above indicates
    $\hcR(A\lambda_T) = 0 < \epsilon$.
    As such, thanks to the final step of \Cref{alg:alg:alg} picking out
    the iterate with lowest classification error, $\hcR(H\hat\lambda) \leq \epsilon$.


    What remains is to establish a deviation inequality.
    Let $\cH_t$ denote the hypothesis class used by predictor $\lambda_t$ (i.e.,
    $\cH_t = \{ \sum_{i=1}^t c_i h_i : c_i \in \R, h_i \in \cH\}$),
    and let $\cS_{\cH_t}(m)$ denote the corresponding shatter coefficient when $\cH_t$ is
    applied to the sample of size $m$
    \citep[Section 3]{bbl_esaim}. 
    It follows \citep[Lemma 4.5]{schapire_freund_book_final} that
    \[
        \cS_{\cH_t}(m) \leq
        \left(\frac {em}{t}\right)^t
        \left(\frac {em}{\cV(\cH)}\right)^{t\cV(\cH)}.
    \]
    Plugging this and $t \leq m^a$
    into an appropriate VC theorem and simplifying \citep[Theorem 5.1 and subsequent
    discussion]{bbl_esaim}, with probability at least $1-\delta/2$,
    \begin{align*}
        \cR(A\hat\lambda)
        &\leq \epsilon
        + 2 \sqrt{
            \epsilon \frac {\ln(\cS_{\cH_t}(2m)) + \ln(8/\delta)}{m}
        }
        + 4 \frac {\ln(\cS_{\cH_t}(2m)) + \ln(8/\delta)}{m}
        \\
        &\leq \epsilon
        + 2 \sqrt{
            \epsilon \frac {(\cV(\cH)+1)\ln(2em) + \ln(8/\delta)}{m^{1-a}}
        }
        + 4 \frac {(\cV(\cH)+1)\ln(2em)  + \ln(8/\delta)}{m^{1-a}}.
    \end{align*}
\end{proofof}



\section{Deferred Material from \Cref{sec:nonseparable}}

\subsection{Proof of \Cref{fact:duality:body}}

\begin{lemma}
    \label[lemma]{fact:duality:dualopt:whatevers}
    Let loss $\ell \in \Lbds$
    and any $g\in \partial\ell(0)$ be given.
    \begin{enumerate}
        \item The restriction of $\ell^*$ to
            $[0,g]$, denoted $\ell^*_{[0,g]}$,
            is a (decreasing) bijection between
            $[0,g]$ and $[-\ell(0),0]$.
        \item
            Let $f(z) := (\ell^*_{[0,g]})^{-1}(z)$
            denote the inverse of $\ell^*_{[0,g]}$.
            If $\nu$ is a probability measure,
            and $p \in L^\infty(\nu)$,
            and $r := - \int \ell^*(p) \in [0, \ell(0)]$,
            then $\nu([p \geq f(-r)]) \geq f(-r)$,
            with $f(-r) > 0$ iff $r > 0$.
    \end{enumerate}
\end{lemma}
\begin{proof}
    Choose any $g\in \partial \ell(0)$; by \Cref{fact:loss:basic},
    $\ell^*$ is 0 at 0, negative along $(0,g)$, and attains its minimum at $g$.
    Since $\ell^*$ is finite for every $z \in (0,g)$, then $\partial \ell^*(z)$
    exists \citep[Theorem 23.4]{ROC}, and every $x \in \partial\ell^*(z)$ satisfies
    $z\in \nabla \ell(x)$ \citep[Theorem 23.5]{ROC}, and so $\ell^*$ is strictly
    convex along
    \citep[Theorem E.4.1.2]{HULL},
    meaning $\ell^*$ is injective along $[0,g]$.  Since $\ell^*(0) = 0$,
    and $-\ell^*(g) = \ell(0)$ (by the Fenchel-Young inequality),
    and since $\ell^*$ is lower semi-continuous \citep[Theorem 12.2]{ROC},
    then $\ell^*$ is also surjective from $[0,g]$ to $[-\ell(0),0]$.

    Now let $f$ denote the inverse map (from $[-\ell(0),0]$ to $[0,g]$),
    and let probability measure $\nu$, function $p\in L^\infty(\nu)$,
    and scalar $r:= -\int\ell^*(p) \in [0,\ell(0)]$ be given (where the containment provides
    $f(-r)$ is valid).
    By Jensen's inequality,
    \[
        -r = \int \ell^*(p) \geq \ell^*\left(\int p\right);
    \]
    since $f$ is a decreasing map, this implies $\int p \geq f(-r)$.
    Furthermore,
    \[
        f(-r) \leq \int p \leq \int_{[p < f(-r)/2]} \frac{f(-r)}{2} + \int_{[p\geq f(-r)/2]} 1
        \leq \frac{f(-r)}{2} + \nu([p \geq f(-r)/2]),
    \]
    meaning $\nu([p \geq f(-r)/2]) \geq f(-r)/2$ as desired.

    Lastly, the statement $f(-r) > 0$ iff $r>0$ follows from the bijectivity
    of $\ell^*_{[0,g]}$.
\end{proof}

\begin{proofof}{\Cref{fact:duality:body}}
    The basic duality relation is provided by \Cref{fact:duality}.
    Since the optimal value satisfies $\bar\cL_\nu \in [0,\ell(0)]$ (since $\nu$ is a
    probability measure, $0\in\Lambda$ is
    primal feasible, and $\ell\geq 0$),
    then \Cref{fact:duality:dualopt:whatevers}
    may be applied with parameter $p$ being the dual
    optimum and parameter $r$ being the corresponding objective value $\bar \cL_\nu$.
\end{proofof}

\subsection{Proof of \Cref{fact:nonsep:controls}}

The first step is to use $\bar p$ to show that if $\lambda$ has low error and norm,
then $H\lambda$ will have very small margins over some positive measure set.

\begin{lemma}
    \label[lemma]{fact:nonsep:ptil:smallness}
    Let convex differentiable $\ell:\R\to\R_+$ with $\ell'(0)> 0$,
    any class $\cH$,
    and any probability measure $\nu$ over $\cX\times \{-1,+1\}$ with
    empirical counterpart $\hmu$ be given.
    Suppose the following quantities and constants exist.
    \begin{enumerate}
        \item Suppose there exists $p\in L^\infty(\mu)$ with $p\in [0, \|p\|_\infty]$
            $\hmu$-a.e.,
            and that there exists $\tau > 0$ with $\hmu([p \geq \tau]) \geq \tau/2$.
        \item Set
            \[
                c:= \frac {16 \ell(0)}{\tau \ell'(0)}
                \max\left\{1,  \frac 1 \tau \right\}
                \max\left\{1,  \|p\|_\infty \right\},
            \]
            and suppose there exist $C>0$ and $D\leq c\|p\|_\infty \tau^2/8$ so
            that every $\lambda\in\Lambda$ with $\|\lambda\|_1 \leq C$
            satisfies $|\int (A\lambda) p d\hmu| \leq D$.
    \end{enumerate}
    If $\lambda\in \Lambda$ satisfies $\|\lambda\|_1 \leq C$
    and $\hmu([ |H\lambda| \geq c]) \geq 1 - \tau/8$,
    then $\hcL(A\lambda) \geq 2\hcL(A\lambda_0) = 2\ell(0)$.
\end{lemma}

\begin{proofof}{\Cref{fact:nonsep:ptil:smallness}}
    First consider the case that
    $\hmu([y(H\lambda)_x \leq -c]) = \hmu([A\lambda \geq c]) \geq \tau/8$.
    By subgradient rules for convex functions, since $\ell \geq 0$ and $\ell'(0)>0$,
    and using the definition of $c$,
    \begin{align*}
        \int \ell(A\lambda)d\hmu
        &\geq \int_{[y(H\lambda)_x \leq -c]} \ell(A\lambda)d\hmu
        \\
        &\geq \int_{[y(H\lambda)_x \leq -c]} \left(\ell(0) + \ell'(0)(A\lambda)\right)d\hmu
        \\
        &\geq \frac {\ell'(0)c\tau}{8}
        \\
        &\geq 2\ell(0),
    \end{align*}
    meaning $\hcL(A\lambda) \geq 2\ell(0) = 2\hcL(0) = 2\cL(A\lambda_0)$.

    Now consider the remaining possibility that $\hmu([y(H\lambda)_x \leq -c]) < \tau/8$.
    Since by assumption $\hmu([|y(H\lambda)_x| \geq c]) \geq 1- \tau/8$,
    it follows that
    $\hmu([y(H\lambda)_x \geq c]) \geq 1- \tau/4$.
    In turn, it also holds that
    \[
        \hmu([y(H\lambda)_x \geq c] \cap [p \geq \tau]) \geq \tau/4.
    \]

    Next, the definition of $D$ provides that
    \begin{align*}
        D
        &\geq \left| \int A\lambda p d\hmu\right|
        \\
        &\geq
        \int_{[y(H\lambda)_x \leq 0]} y(H\lambda)_x p(x,y)d\hmu(x,y)
        + \int_{[y(H\lambda)_x > 0]} y(H\lambda)_x p(x,y)d\hmu(x,y),
    \end{align*}
    meaning
    \begin{align*}
         \int_{[y(H\lambda)_x \leq 0]} -y(H\lambda)_x p(x,y)d\hmu(x,y)
         &\geq \int_{[y(H\lambda)_x > 0]} y(H\lambda)_x p(x,y)d\hmu(x,y)
        -D.
    \end{align*}
    As such, since $p \in [0,\|p\|_\infty]$ over the sample,
    \begin{align*}
        -\|p\|_\infty\int_{[y(H\lambda)_x \leq 0]} y(H\lambda)_x d\hmu(x,y)
        &\geq
        \int_{[y(H\lambda)_x \leq 0]} -y(H\lambda)_x p(x,y)d\hmu(x,y)
        \\
        &\geq
        \int_{[y(H\lambda)_x > 0]} y(H\lambda)_x p(x,y)d\hmu(x,y)
        -D
        \\
        &\geq
        \int_{[y(H\lambda)_x \geq c]\cap [p \geq \tau]} y(H\lambda)_x p(x,y)d\hmu(x,y)
        -D
        \\
        &\geq
        \frac {c\tau^2}{4}
        -D.
    \end{align*}
    Turning back to $\hcL$ and proceeding similarly to the earlier case,
    \begin{align*}
        \int \ell(A\lambda)d\hmu
        &\geq \int_{[y(H\lambda)_x \leq 0]} \ell(A\lambda)d\hmu
        \\
        &\geq \int_{[y(H\lambda)_x \leq 0]} \left(\ell(0) + \ell'(0)(A\lambda)\right)d\hmu
        \\
        &\geq \frac{\ell'(0)}{\|p\|_\infty}\left(\frac {c\tau^2}{4} - D\right)
        \\
        &\geq 2\ell(0),
    \end{align*}
    which again yields $\hcL(A\lambda) \geq 2\ell(0) = 2\hcL(0) = 2\hcL(A\lambda_0)$.
\end{proofof}

Next, these small margins in turn cause the line search to not look too far,
meaning the next iterate will also have some small margins.
Note that $\cH$
is assumed binary; this is in order to changes in $H\lambda$ to changes in $\lambda$.

\begin{lemma}
    \label[lemma]{fact:nonsep:control_linesearch}
    Let convex $\ell :\R\to\R_+$,
    binary $\cH$, and probability measure $\mu$ with empirical
    counterpart $\hmu$ be given.
    Let positive reals $C, c, \tau$ be given so that $\lambda \in \Lambda$
    with
    $\|\lambda_1\|_1 \leq C + 2c$ and $\hmu([|H\lambda| \geq c]) \geq 1-\tau/8$
    implies $\hcL(A\lambda) \geq 2\ell(0)$.
    Then for any $\lambda \in \Lambda$ with $\|\lambda\|_1 \leq C$
    and $\cL(A\lambda) \leq \ell(0)$,
    the set of line search candidates
    \[
        S_\lambda := \left\{
            \lambda' \in \Lambda : \hcL(A\lambda') < 2\ell(0),
            \exists \alpha \in \R, h\in \cH \centerdot \lambda' := \lambda + \alpha \bfe_h
        \right\}
    \]
    satisfies $\hmu([|H\lambda'| \geq c]) < 1-\tau/8$ for every $\lambda' \in S_\lambda$.
\end{lemma}
\begin{proof}
    Let $\lambda$ with $\|\lambda\|_1 \leq C$ be
    given, and consider any $\lambda'$ of the form
    $\lambda':= \lambda + \alpha' \bfe_h$
    for some $\alpha' \in \R$ and $h\in \cH$.
    The desired statement will be shown by contrapositive; namely,
    $\hmu([|H\lambda'|\geq c]) \geq 1-\tau/8$ implies $\lambda' \not \in S_\lambda$.

    For any example $(x,y)$, since $H$ has binary predictors, the map
    $\alpha \mapsto \1[|y(H(\lambda + \alpha \bfe_h))_x|\geq c]$ is constant for
    $\alpha > 2c$ and for $\alpha < -2c$;
    consequently, since $\hmu$ is a discrete measure over a finite set, the map
    $\alpha \mapsto \hmu([|y((H(\lambda + \alpha \bfe_h))_x|\geq c])$ is also
    constant for $\alpha > 2c$ and $\alpha < -2c$.
    As such, the existence of $\lambda'$ as above
    implies the existence of $\lambda'' := \lambda + \alpha'' \bfe_h$
    where $h\in \cH$ is as before, $\alpha''$ and $\alpha'$ have the same sign,
    and $|\alpha''| \leq \min\{2c, |\alpha'|\}$; in other words, $\lambda''$ is along
    the path from $\lambda$ to $\lambda'$, but moreover satisfies
    $\|\lambda - \lambda''\|_1 \leq 2c$.
    But this means
    $\|\lambda''\| \leq \|\lambda\|_1 + \|\lambda - \lambda''\|_1
    \leq C + 2c$, whereby the stated assumptions combined with
    $\hmu([|H\lambda'|\geq c]) \geq 1-\tau/8$
    provide $\hcL(A\lambda'') \geq 2\ell(0)$, thus $\lambda'' \not \in S_\lambda$.
    Furthermore, since $\lambda''$ is along the path from $\lambda$ to $\lambda''$,
    and $\hcL(A\lambda) \leq \ell(0)$ and
    $\hcL(A\lambda'') \geq 2\ell(0) \geq \hcL(A\lambda)$,
    it follows by convexity that $\hcL(A\lambda') \geq 2\ell(0)$,
    and thus $\lambda'' \not \in S_\lambda$ as well.
\end{proof}

These small margin controls directly give a bound on step sizes.

\begin{lemma}[{see also \citet[eq. (28)]{bartlett_traskin_adaboost}}]
    \label[lemma]{fact:stepsize:basic_ub}
    Let $\ell \in \Lbds \cap \Ltd$ be given with Lipschitz gradient parameter $B_2$,
    binary class $\cH$,
    time horizon $t$,
    and empirical probability measure $\hmu$ corresponding to a sample of size $m$
    be given.
    Let positive real $B_1>0$ be given so that
    for any $\lambda\in \Lambda$
    with
    \[
        \|\lambda\|_1
        \leq
        \sqrt{
        \frac {t\max\{5,2B_2/B_1\}(\hcL(A\lambda_{0}) - \hcL(A\lambda_t))}
        {\rho^2B_1}},
    \]
    and any line search candidate
    $\lambda' := \lambda + \alpha \bfe_h$ for $h\in\cH$, $\alpha\in\R$, and satisfying
    $\hcL(A\lambda') \leq \hcL(A\lambda)$,
    then
    $B_1 \leq \frac 1 m\sum_{i=1} \ell''((A\lambda')_i)$.
    The following properties hold.
    \begin{enumerate}
        \item For every integer $0\leq i < t$, an optimal step $\bar\alpha_i$ exists,
            and every step sizes choice satisfies
            \[
                \alpha_i^2
                \leq
                \min\left\{
                \frac {9\|A^\top \nhcL(A\lambda_{i-1})\|_\infty^2}{4\rho^2 B_1^2}
                    \ ,  \ 
                \frac {\max\{5,2B_2/B_1\}(\hcL(A\lambda_{i-1}) - \hcL(A\lambda_i))}
                {\rho^2 B_1}
            \right\}.
            \]
        \item For every integer $0\leq i < t$
            and any sequence of step size choices,
            \begin{align*}
                \|\lambda_i\|_1
                \leq \sqrt{i} \sqrt{\sum_{j=1}^i \alpha_j^2}
                &\leq
                \sqrt{i}
                \sqrt{
                \frac {\max\{5,2B_2/B_1\}(\hcL(A\lambda_{0}) - \hcL(A\lambda_i))}
                {\rho^2B_1}}
                \\
                &\leq
                \sqrt{i}
                \sqrt{
                    \frac {\ell(0)\max\{5,2B_2/B_1\}}
                {\rho^2B_1}}.
            \end{align*}
    \end{enumerate}
\end{lemma}
\begin{proof}
    This proof establishes both properties simultaneously by induction on $i$.
    In the base case $\lambda_i = \lambda_0 = 0$ and there is nothing to
    show, thus suppose $i \geq 1$.

    Define the interval
    $
        I_i := \{ \alpha \geq 0 : \hcL(A(\lambda_{i-1} + \alpha v_{i}))
            \leq \hcL(A\lambda_{i-1})
        \}
        $.
    Combining the inductive hypothesis (controlling $\|\lambda_{i-1}\|_1$)
    with the assumptions on line search candidates
    means the second-order lower bound $B_1$ is active along $I_i$.
    Now consider a Taylor expansion of $\hcL$,
    but in the direction reverse to
    \Cref{fact:singlestep:quadub}, and using the fact that $H$ is binary;
    then for any $\alpha \in I_i$
    and some $z\in[\lambda_{i-1},\lambda_{i-1} +  \alpha v_i]$,
    \begin{align*}
        \hcL(A(\lambda_{i-1} + \alpha v_i))
        &=
        \hcL(A\lambda_{i-1}) + \alpha \ip{A^\top\hcL(A\lambda_{i-1})}{v_i}
        + \frac {\alpha^2}{2m}\sum_i \ell''((Az)_i) (A v_i)^2
        \\
        &\geq
        \hcL(A\lambda_{i-1}) - \alpha \|A^\top\hcL(A\lambda_{i-1})\|_\infty
        + \frac {\alpha^2}{2} B_1.
    \end{align*}
    This last expression defines a univariate quadratic which lies
    below $\hcL(A(\lambda_{i-1} + \alpha v_i))$ along $I_i$ (outside of $I_i$, the
    constraints granting the lower bound $B_1$ may be violated).
    Consequently, $I_i$ is bounded,
    and the optimal step $\bar \alpha_i$ must exist, and moreover
    satisfies
    \begin{equation}
        \bar \alpha_i \leq \alpha_* :=
        \frac{\|A^\top\hcL(A\lambda_{i-1})\|_\infty}{B_1},
        \label{eq:helper:wats:z}
    \end{equation}
    where $\alpha_*$ is the minimizer to the above quadratic.
    Plugging $\alpha_*$ back into the quadratic,
    for any $\alpha\in I_i$,
    \begin{equation}
        \hcL(A(\lambda_{i-1}+\alpha v_i))
        \geq
        \hcL(A\lambda_{i-1})
        -\frac {\|A^\top\hcL(A\lambda_{i-1})\|_\infty^2}{2B_1}.
        \label{eq:helper:wats:z}
    \end{equation}

    Now consider the first two step size options in \Cref{alg:alg:alg};
    combining \cref{eq:helper:wats:z} with \Cref{fact:singlestep:quadub},
    \begin{align*}
        \alpha_i^2
        &\leq \bar\alpha_i^2
        \leq \frac {\|A^\top\hcL(A\lambda_{i-1})\|_\infty^2}{B_1^2}
        \leq \frac {2B_2(\hcL(A\lambda_{i-1}) - \hcL(A\lambda_i))}{\rho^2B_1^2}
    \end{align*}
    as desired.

    For option 3 (the Wolfe search), combining \cref{eq:wolfe:1}
    with \cref{eq:helper:wats:z} grants
    \begin{align*}
        \alpha_i^2
        &\leq
        \frac {9(\hcL(A\lambda_{i-1}) - \hcL(A\lambda_i))^2}
        {\rho^2\|A^\top\nhcL(A\lambda_{i-1})\|_\infty^2}
        \leq
        \frac {9(\hcL(A\lambda_{i-1}) - \hcL(A\lambda_i))}
        {2\rho^2B_1}
        \leq
        \frac {9\|A^\top\nhcL(A\lambda_{i-1})\|_\infty^2}
        {4\rho^2B_1^2},
    \end{align*}
    which establishes the first inductive property for all step sizes.

    The second statement is just Cauchy-Schwarz
    combined with the bound on $\alpha_i^2$:
    \begin{align*}
        \|\lambda_i\|_1
        &\leq \sum_{j=1}^i \alpha_j \|v_j\|_1
        \\
        &\leq \sqrt{i} \sqrt{\sum_{j=1}^i \alpha_j^2}
        \\
        &\leq
        \sqrt{i}
        \sqrt{
        \frac {\max\{5,2B_2/B_1\}\sum_{j=1}^i(\hcL(A\lambda_{j-1}) - \hcL(A\lambda_j))}
        {\rho^2B_1}}
        \\
        &=
        \sqrt{i}
        \sqrt{
        \frac {\max\{5,2B_2/B_1\}(\hcL(A\lambda_{0}) - \hcL(A\lambda_i))}
        {\rho^2B_1}},
    \end{align*}
    and nonnegativity of $\ell$ and the fact that all steps perform descent
    grants $\hcL(A\lambda_0) - \hcL(A\lambda_i) \leq \ell(0)$.
\end{proof}

After some algebra, the upper bound also grants a lower bound; due to this indirection,
it should be possible to improve this bound.  Note that the beginning of this derivation,
when initially lower bounding $\alpha_i$, uses derivations similar
to those used by \citet{zhang_yu_boosting} and \citet{bartlett_traskin_adaboost}.

\begin{lemma}
    \label[lemma]{fact:stepsize:basic_lb}
    Let $\ell \in \Lbds \cap \Ltd$ be given with Lipschitz gradient parameter $B_2$,
    binary class $\cH$,
    time horizon $t$,
    and empirical probability measure $\hmu$ corresponding to a sample of size $m$
    be given.
    Suppose there exists $c_2>0$ with $\|\lambda_i\|_1 \leq c_2 \sqrt{i}$ for
    all $0 \leq i \leq t$.
    Additionally, let $C_2 > 0$ and $\bar\lambda\in\Lambda$ be given with
    $\|\bar\lambda\|_1\leq C_2$
    and
    $\epsilon := \min_{i\in [t-1]} \hcL(A\lambda_i) - \hcL(A\bar\lambda) \geq 0$.
    Then
    \[
        \alpha_i
        \geq \frac {\rho (\hcL(A\lambda_{i-1}) - \hcL(A\bar\lambda))}
        {2B_2 (\|\bar\lambda\|_1 + c_2\sqrt{i-1})}
    \]
    and
    \[
        \sum_{i=1}^t \alpha_i
        \geq
        \frac {\rho \epsilon }{2B_2c_2}
        \left(
            2\sqrt{t-1}
            +
            \frac {c_2}{C_2}
            +
            \frac {2C_2}{c_2}
            \ln\left(
                \frac
                {C_2}{C_2 +c_2 \sqrt{t-1}}
            \right)
        \right),
    \]
    or more simply
    $
        \sum_{i=1}^t \alpha_i
        \geq
        (\rho \epsilon \sqrt{t-1}) / (4B_2c_2)$
        if $C_2 \leq c_2\sqrt{t-1}$.
\end{lemma}

\begin{proof}
    For any $1 \leq i \leq t$, note by \Cref{fact:AT} that
    \begin{align*}
        \|A^\top \nhcL(A\lambda_{i-1})\|_\infty
        &= \sup_{\|\lambda\|_1\leq 1} \ip{A^\top \nhcL(A\lambda_{i-1})}{\lambda}
        \\
        &\geq \frac {1}{\|\lambda_{i-1} - \bar \lambda\|_1}
        \ip{\nhcL(A\lambda_{i-1})}{A(\lambda_{i-1} - \bar \lambda)}
        \\
        &\geq \frac {1}{\|\lambda_{i-1} - \bar \lambda\|_1}
        (\hcL(A\lambda_{i-1}) - \hcL(A\bar\lambda)).
    \end{align*}
    Thus, by the lower bounds
    in \Cref{fact:singlestep:quadub,fact:singlestep:wolfe} for every step size,
    \[
        \alpha_i
        \geq \frac {\rho \|A^\top \nhcL(A\lambda_{i-1})\|_\infty}{2B_2}
        \geq \frac {\rho (\hcL(A\lambda_{i-1}) - \hcL(A\bar\lambda))}
        {2B_2 \|\lambda_{i-1} - \bar \lambda\|_1}
        \geq \frac {\rho \epsilon}
        {2B_2 \|\lambda_{i-1} - \bar \lambda\|_1}.
    \]
    Combining this with the provided upper bound on $\|\lambda_i\|_1$,
    \[
        \alpha_i
        \geq \frac {\rho \epsilon}
        {2B_2 (\|\bar\lambda\|_1 + \|\lambda_{i-1}\|_1)}
        \geq \frac {\rho \epsilon}
        {2B_2 (\|\bar\lambda\|_1 + c_2\sqrt{i-1})}
    \]
    As a consequence of this, and recalling the simplification $\|\bar\lambda\|_1\leq C_2$,
    \begin{align*}
        \sum_{i=1}^t \alpha_i
        &\geq \frac {\rho \epsilon }{2B_2c_2}
        \left(
            \frac {c_2}{C_2}
            +
            \sum_{i=1}^{t-1} \frac {c_2}{C_2 + c_2\sqrt{i}}
        \right)
        \\
        &\geq
        \frac {\rho \epsilon }{2B_2c_2}
        \left(
            \frac {c_2}{C_2}
            +
            \int_0^{t-1} \frac {c_2dx}{C_2 + c_2\sqrt{x}}
        \right)
        \\
        &=
        \frac {\rho \epsilon }{2B_2c_2}
        \left(
            \frac {c_2}{C_2}
            +
            \left.
            \left(
            2\sqrt{x} - \frac{2 C_2\ln(\sqrt{x} + C_2/c_2)
            }{c_2}
            \right)\right|_{0}^{t-1}
        \right)
        \\
        &=
        \frac {\rho \epsilon }{2B_2c_2}
        \left(
            \frac {c_2}{C_2}
            +
            2\sqrt{t-1}
            +
            \frac {2C_2}{c_2}
            \ln\left(
                \frac
                {C_2}{C_2 +c_2 \sqrt{t-1}}
            \right)
        \right).
    \end{align*}
    When $C_2\leq c_2\sqrt{t-1}$, it suffices to instantiate the above bound
    with $C_2' := c_2\sqrt{t-1}$ (whereby it still holds that
    $\|\bar\lambda\|_1\leq C_2\leq C_2'$),
    and then rearrange,
    noting
    $1-\ln(2) \geq 1/4$
    and deleting the nonnegative standalone term $c_2/C_2$.
\end{proof}

By combining the above chain of results, the proof of \Cref{fact:nonsep:controls} follows.

\begin{proofof}{\Cref{fact:nonsep:controls}}
    Recalling the structure from \Cref{fact:duality:body},
    let dual optimum $\bar p$ and real $\tau>0$ be
    given so that $\mu([\bar p \geq \tau]) \geq\tau$.
    By Hoeffding's inequality and the first lower bound on $m$,
    with probability at least $1-\delta/4$
    \[
        \hmu([\bar p \geq  \tau])
        \geq \mu([\bar p\geq\tau]) - \sqrt{\frac 1 {2m} \ln \left(\frac 4 \delta\right)}
        \geq \frac \tau 2.
    \]
    Henceforth disregard the corresponding failure event.

    Next define
    \[
        D:=
        \frac{2(R_t + 2c)\|\bar p\|_\infty}{m^{1/2}}\left(
            2\sqrt{2\cV(\cH)\ln(m+1)} +\sqrt{2\ln(4/\delta)}
        \right).
    \]
    The second condition on $m$ grants $D \leq c\tau^2/8$,
    whereby \Cref{fact:pdc:body} grants,
    with probability at least $1-\delta/4$,
    that
    every $\lambda\in\Lambda$ with $\|\lambda\|_1 \leq R_t+ 4c$
    satisfies
    \[
        \left|\int (A \lambda) \bar p d\hmu\right|
        \leq D \leq \frac {c\ttil^2}{8}.
    \]
    Discard the corresponding failure event as well; unioning this with the earlier
    failure event, the remaining steps hold with probability at least $1-\delta/2$.

    It follows from \Cref{fact:nonsep:ptil:smallness}
    that every such $\lambda\in\Lambda$ with $\|\lambda\|_1\leq R_t + 2c$ either
    satisfies
    $\hmu([ |H\lambda| \geq c]) < 1 - \ttil/8$,
    or else $\hcL(A\lambda) \geq 2\ell(0)$.
    This in turn means the preconditions to \Cref{fact:nonsep:control_linesearch}
    are met (with $C:= R_t$),
    and in particular, for any $\lambda \in \Lambda$ with $\|\lambda\|_1 \leq R_t$
    and $\hcL(A\lambda) \leq \hcL(A\lambda_0)$,
    the set of line search candidates
    \[
        S_\lambda := \left\{
            \lambda' \in \Lambda : \hcL(A\lambda') < 2\ell(0),
            \exists \alpha \in \R, h\in \cH \centerdot \lambda' := \lambda + \alpha \bfe_h
        \right\}
    \]
    satisfies $\hmu([|H\lambda'| \geq c]) < 1-\tau/8$ for every $\lambda' \in S_\lambda$.
    But then, for every $\lambda' \in S_\lambda$ (and note $\lambda \in S_\lambda$),
    \begin{align}
        \frac 1 m \sum_{i\in [m]} \ell''((A\lambda)_{x_i,y_i})
        \geq
        \frac 1 m
        \sum_{\substack{i\in [m] \\ |(A\lambda)_{x_i,y_i}| \leq c}}
        \ell''((A\lambda)_{x_i,y_i})
        \geq
        \hmu([|H\lambda| \leq c]) \inf_{z\in [-c,+c]} \ell''(z)
        \geq B_1.
        \notag
    \end{align}
    This establishes the first desired statement.

    For the second statement (upper bounds on $\alpha_i$ and $\|\lambda_i\|_1$),
    note that the above properties
    satisfy the preconditions to \Cref{fact:stepsize:basic_ub},
    whereby the desired upper bounds follow.

    Similarly, for the third statement (lower bounds on $\alpha_i$ and $\|\lambda_i\|_1$),
    the preconditions for \Cref{fact:stepsize:basic_lb} are now met.
\end{proofof}

\subsection{Optimization Guarantees}
As stated previously, the following proof is a reworking
of a proof due to \citet{zhang_yu_boosting}, albeit with the present decoupling of line
search and coordinate selection.

\begin{proofof}{\Cref{fact:nonsep:opt:zhangyu_style}}
    Let $1\leq i \leq t$ be arbitrary.
    The first step of this proof is to develop two lower bounds on
    $\|A^\top \nhcL(A\lambda_{i-1})\|_\infty$.
    First, just as in the proof of
    \Cref{fact:stepsize:basic_lb},
    \begin{align*}
        \|A^\top \nhcL(A\lambda_{i-1})\|_\infty
        &= \sup_{\|\lambda\|_1\leq 1} \ip{A^\top \nhcL(A\lambda_{i-1})}{\lambda}
        \\
        &\geq \frac {1}{\|\lambda_{i-1} - \bar \lambda\|_1}
        \ip{\nhcL(A\lambda_{i-1})}{A(\lambda_{i-1} - \bar \lambda)}
        \\
        &\geq \frac {1}{\|\lambda_{i-1} - \bar \lambda\|_1}
        (\hcL(A\lambda_{i-1}) - \hcL(A\bar\lambda)).
    \end{align*}
    The second lower bound is provided by assumption, and thus,
    by the guarantee on any line search as in
    \Cref{fact:singlestep:quadub,fact:singlestep:wolfe},
    \begin{align*}
        \hcL(A\lambda_{i}) - \hcL(A\bar\lambda)
        &\leq
        \hcL(A\lambda_{i-1}) - \hcL(A\bar\lambda)
        - \frac {\rho^2\|A^\top \nhcL(A\lambda_{i-1})\|_\infty^2}{6B_2}
        \\
        &\leq
        \hcL(A\lambda_{i-1}) - \hcL(A\bar\lambda)
        - \frac {c_3\rho^2\alpha_i(\hcL(A\lambda_{i-1}) - \hcL(A\bar\lambda))}
        {6B_2\|\bar \lambda - \lambda_{i-1}\|_1}
        \\
        &=
        \left(
            \hcL(A\lambda_{{i-1}}) - \hcL(A\bar\lambda)
        \right)
        \left(
            1
            - \frac {c_3\rho^2\alpha_i}
            {6B_2\|\bar \lambda - \lambda_{i-1}\|_1}
        \right)
        \\
        &\leq
        \left(
            \hcL(A\lambda_{i-1}) - \hcL(A\bar\lambda)
        \right)
        \exp
        \left(
            - \frac {c_4\alpha_i}
            {\|\bar \lambda - \lambda_{i-1}\|_1}
        \right).
    \end{align*}
    where the last step took $c_4 := c_3 \rho^2/(6B_2)$ for convenience.
    Iterating this bound,
    \begin{align*}
        \hcL(A\lambda_{t}) - \hcL(A\bar\lambda)
        &\leq
        \left(\hcL(A\lambda_{0}) - \hcL(A\bar\lambda)\right)
        \exp
        \left(
            - c_4\sum_{i=1}^{t}\frac {\alpha_i}
            {\|\bar \lambda - \lambda_{i-1}\|_1}
        \right).
    \end{align*}
    Focusing on the summation, define $S_i := \sum_{j\leq i} \alpha_i$ with
    $S_0 := 0$, whereby $\|\lambda_i\|_1 \leq S_i$.   Using this
    (see also the similar derivation by \citet[Proof of Lemma 4.2]{zhang_yu_boosting}),
    \begin{align*}
        \sum_{i=1}^{t}\frac
        {\alpha_i}{\|\bar \lambda\|_1 + \|\lambda_{i-1}\|_1}
        &\geq
        \sum_{i=1}^{t}\frac
        {\alpha_i}{\|\bar \lambda\|_1 + S_{i-1}}
        \\
        &=
        \sum_{i=1}^{t}\left(
            \frac
            {\|\bar \lambda\|_1 + S_i}
            {\|\bar \lambda\|_1 + S_{i-1}}
            -
            \frac
            {\|\bar \lambda\|_1 + S_{i-1}}
            {\|\bar \lambda\|_1 + S_{i-1}}
        \right)
        \\
        &\geq
        \sum_{i=1}^{t}\ln\left(
            \frac
            {\|\bar \lambda\|_1 + S_i}
            {\|\bar \lambda\|_1 + S_{i-1}}
        \right)
        =
        \ln\left(
            \frac
            {\|\bar \lambda\|_1 + S_i}
            {\|\bar \lambda\|_1}
        \right).
    \end{align*}
    Plugging this into the preceding display and collecting terms,
    the result follows.
\end{proofof}

Note that the substitution $\|\lambda_i\|_1\leq S_i$ at the end of the proof
of \Cref{fact:nonsep:opt:zhangyu_style}
works around the fact that $\sum_i \alpha_i$ could be much larger than
$\|\lambda_i\|_1$; this issue is frequently avoided in the literature by assuming
that $\cH$ is closed under negation, whereby $v_i=\bfe_{h_i}$ in each
round (i.e., rather than $v_i \in \{\pm \bfe_{h_i}\}$).

\subsection{Statistical Guarantees}

\begin{proofof}{\Cref{fact:nonsep:basic}}
    Let $\sigma>0$ be arbitrary, and
    choose $\bar\lambda \in \Lambda$ with $\|\bar\lambda\|_1\leq R_{t-1}$
    so that
    \[
        \cL(A\bar\lambda) \leq \sigma + \inf_{\|\lambda\|_1 \leq R_{t-1}} \cL(A\lambda).
    \]
    By McDiarmid's inequality and the fact that
    $\sup_{x,y} |(A\bar\lambda)_{x,y}| \leq R_{t-1}$, with probability at least
    $1-\delta/6$,
    \[
        \hcL(A\bar\lambda)
        \leq
        R_{t-1} \sqrt{\frac 2 m \ln\left( \frac 6 \delta\right)}
        + \cL(A\bar\lambda) \leq
        \sigma +
        R_{t-1} \sqrt{\frac 2 m \ln\left( \frac 6 \delta\right)}
        + \inf_{\|\lambda\|_1 \leq R_{t-1}} \cL(A\lambda).
    \]
    Now let $b\in (0,1/2)$ be arbitrary,
    and consider two cases for the difference
    $\epsilon := \min_{i\in [t-1]} \hcL(A\lambda_i) - \hcL(A\bar\lambda)$.
    \begin{itemize}
        \item If $\epsilon \leq (t-1)^{-b}$, then
            \[
                \hcL(A\lambda_t)
                \leq \hcL(A\lambda_{t-1}) \leq (t-1)^{-b} + \hcL(A\bar\lambda).
            \]
        \item Otherwise $\epsilon > (t-1)^{-b}$.  Thus
            \Cref{fact:nonsep:controls} grants, with probability at least $1-\delta/2$,
            \[
                \alpha_i^2
                \leq
                \frac {9\|A^\top \nhcL(A\lambda_{i-1})\|_\infty^2}{4\rho^2 B_1^2}
                \qquad\textup{and}\qquad
                \sum_{i=1}^t \alpha_i \geq
                \frac {\rho \epsilon \sqrt{t-1}}{4B_2R_1},
            \]
            which can then  be plugged into \Cref{fact:nonsep:opt:zhangyu_style}
            (with $c_3 := 2\rho B_1/3$),
            together with the lower bound on $\epsilon$, to yield
            \begin{align*}
                \hcL(A\lambda_{t}) - \hcL(A\bar\lambda)
                &\leq
                \left(\hcL(A\lambda_{0}) - \hcL(A\bar\lambda)\right)
                \left(
                    \frac
                    {\|\bar \lambda\|_1}
                    {\|\bar \lambda\|_1 + \sum_{i\leq t} \alpha_i}
                \right)^{6B_2/ (c_3\rho^2)}
                \\
                &\leq
                \left(\hcL(A\lambda_{0}) - \hcL(A\bar\lambda)\right)
                \left(
                    \frac
                    {\|\bar \lambda\|_1}
                    {\|\bar \lambda\|_1 + \rho (t-1)^{1/2-b} / (4B_2R_1)}
                \right)^{9B_2/ (B_1\rho^3)}.
            \end{align*}
    \end{itemize}
    Summing these two bounds gives a relation which holds in general;
    consequently, with probability at least $1-\delta/2$ (due to the
    invocation of \Cref{fact:nonsep:controls}),
    \begin{align*}
        \hcL(A\lambda_{t}) - \hcL(A\bar\lambda)
        \leq (t-1)^{-b}
        +
        \ell(0)
        \left(
            \frac
            {\|\bar \lambda\|_1}
            {\|\bar \lambda\|_1 + \rho (t-1)^{1/2-b} / (4B_2R_1)}
        \right)^{9B_2/ (B_1\rho^3)}.
    \end{align*}
    The first desired claim follows (with probability $1-5\delta/6$)
    by recalling the earlier application of McDiarmid's inequality to $\bar\lambda$,
    combined with $\sigma \downarrow 0$,
    the choice $b:= 1/4$ (which is not optimal, but neither is the exponent
    $9B_2/ (B_1\rho^3)$), and
    standard Rademacher bounds for voting classifiers applied to Lipschitz losses
    \citep[Theorem 4.1, eq. (8), and their proofs, which control for $\cL$]{bbl_esaim},
    which makes use of the bound $\|\lambda_t\|_1 \leq R_t$ (granted by the earlier
    instantiation of \Cref{fact:nonsep:controls}), and simplifying via $t-1 \leq t$,
    \begin{align*}
        \cL(A\lambda_{t})
        &\leq \inf_{\|\lambda\|_1 \leq R_{t-1}} \cL(A\bar\lambda)
        + t^{-1/4}
        + R_{t-1} \sqrt{\frac 2 m \ln\left( \frac 6 \delta\right)}
        \\
        &\qquad
        + \frac {2\beta_2R_t}{m^{1/2}}
        \left(2\sqrt{2\cV(\cH)\ln(m+1)} + \ell(R_t)\sqrt{2\ln(6/\delta)}\right)
        \\
        &\qquad
        + \ell(0)
        \left(
            \frac
            {\|\bar \lambda\|_1}
            {\|\bar \lambda\|_1 + \rho t^{1/4} / (4B_2R_1)}
        \right)^{9B_2/ (B_1\rho^3)}.
    \end{align*}
    Plugging in $t=m^a$ gives the first result.

    For the second guarantee,
    since $\cH\in \Fds(\mu^{\cX})$, then
    \Cref{fact:Fds_weakening} grants
    \[
        \inf \{ \cL(A\lambda) : \lambda\in\Lambda\}
        =
        \inf \left\{ \int \ell (-yf(x)) d\nu(x,y)
        : f \in \Fb\right\},
    \]
    where $\Fb$ is the family of Borel measurable functions from $\cX$ to $\R$.
    From here, since $\ell$ is \emph{classification calibrated}
    \citep[Theorem 2, noting that the present manuscript instead
    takes losses to be nondecreasing rather than nonincreasing]{bartlett_jordan_mcauliffe},
    there exists a function $\psi:\R \to \R$ satisfying
    \[
        \psi\left(
            \cR(A\lambda_t) - \bar\cR 
        \right)
        \leq
        \cL(A\lambda_t) - \inf_{f\in\Fb} \int \ell(-yf(x))d\mu(x,y)
         =
         \cL(A\lambda_t) - \inf_{\lambda\in\Lambda} \cL(A\lambda)
    \]
    (where the last step used the previous display),
    and moreover $\psi(z) \to 0$ as $z \to 0$
    \citep[Theorem 1]{bartlett_jordan_mcauliffe}.
    The specialization for $\llog$ is due to
    \citet[Subsection 3.5 and Corollary 3.1]{zhang_convex_consistency}.

    The final guarantee follows by applying a version of the VC theorem to
    the predictors, and follows the exact strategy as in \Cref{fact:sep:basic}
    (but using failure probability $\delta/6$),
    and making use of the equality
    \[
        \cR(H\hat\lambda) - \cR(H\lambda_t)
        =
        (\cR(H\hat\lambda) - \hcR(H\hat\lambda))
        +
        (\hcR(H\hat\lambda) - \hcR(H\lambda_t))
        +
        (\hcR(H\lambda_t) - \cR(H\lambda_t)),
    \]
    and the fact that $\hcR(H\hat\lambda) \leq \hcR(H\lambda_t)$.
\end{proofof}

\section{Proof of Consistency}

\begin{proofof}{\Cref{fact:consistency}}
    This proof is a standard application of the Borel-Cantelli Lemma; for
    an exposition on such applications, please see the
    proof of consistency of
    AdaBoost due to \citet[Proof of Corollary 12.3]{schapire_freund_book_final}.
    In particular, let any $\epsilon >0$ be given, and let $E_{m,\epsilon}$
    be the event that the output $\hat\lambda_m$, trained on $m$ examples,
    has classification risk $\cR(H\hat\lambda_m)$ exceeding the Bayes risk
    $\cR_\star$ by more than $\epsilon$;
    to prove consistency, it suffices (thanks to Borel-Cantelli)
    to show that $\sum_m \Pr(E_{m,\epsilon}) < \infty$.

    There are two cases to consider: either $\bar\cL =0$, or $\bar \cL > 0$.
    In the case that $\bar\cL = 0$, instantiate the finite sample guarantee
    in \Cref{fact:sep:basic} for each $m\geq 1$ with $\epsilon/2$ and $\delta:= 1/m^2$;
    there is a real $M<\infty$ where $m> M$ provides the preconditions on
    the bound are met and the bound is at most $\epsilon$, and thus
    \[
        \sum_{m=1}^\infty \Pr(E_{m,\epsilon})
        \leq \sum_{m=1}^M 1
        + \sum_{m=M+1}^\infty \frac 1 {m^2}
        \leq M + \frac {\pi^2}{6} < \infty.
    \]

     When $\bar \cL>0$, once again instantiate a relevant finite sample
     guarantee, this time from \Cref{fact:nonsep:basic}, with
     $\delta := 1/m^2$.  It will be necessary to use all three guarantees;
     first, let $M_1$ be sufficiently large so that the third guarantee
     provides $\cR(H\hat\lambda_m) \leq \cR(H\lambda_m') + \epsilon/2$
     with failure probability $m^{-2}$ for
     all $m> M_1$, where $\lambda'_m$ is the last iterate considered by the
     algorithm when run on $m$ examples (thus $\lambda'_m$ is basically $\lambda_{m^a}$,
     modulo rounding issues).  Next, by the second guarantee, there exists
     $\epsilon'$ small enough so that
     \[
         \epsilon' \geq \psi(\cR(A\lambda'_m) - \cR_\star)
         \qquad\Longrightarrow
         \qquad
         \epsilon \geq \cR(A\lambda'_m) - \cR_\star.
     \]
     As such, now consider the first guarantee, where the goal will be to
     establish that $\cL(A\lambda'_m) - \bar\cL \leq \epsilon'$
     for all large $m$.
     But note firstly that the quantity $R_{t-1} \to \infty$ as $m\to\infty$,
     which combined with
     \[
         \inf \left \{\cL(A\lambda) : \lambda \in \Lambda\right\}
         =
         \inf_{C > 0}
         \inf \left \{\cL(A\lambda) : \lambda \in \Lambda, \|\lambda\|_1 \leq C\right\},
     \]
     grants that, if $m> M_2$ (for some $M_2$), then
     $\inf_{\|\lambda\|\leq R_t} \cL(A\lambda) \leq \bar\cL + \epsilon'/2$.
     Finally, the rest of the terms in the first guarantee are at most $\epsilon'/2$ for
     $m> M_3$ (for some $M_3$) with the same $m^{-2}$ failure probability.  As such,
     similarly to before,
     $\sum_m \Pr(E_{m,\epsilon}) \leq \max\{M_1,M_2,M_3\}+\pi^2/6 < \infty$.
\end{proofof}

\end{document}